\documentclass{article} %
\usepackage{iclr2025_conference}

\newcommand{\reviewtext}[1]{\textcolor{black}{#1}}  %

\usepackage{times}
\usepackage{microtype}
\usepackage{graphicx}
\usepackage{subfigure}
\usepackage{booktabs} %
\usepackage{xspace}
\usepackage{titletoc}
\usepackage{enumitem}
\usepackage{xfrac}
\usepackage{wrapfig}
\usepackage{algorithm,algorithmic}

\usepackage[colorlinks,linkcolor=linkcolor,citecolor=citecolor]{hyperref}
\usepackage{url}

\usepackage{amsmath}
\usepackage{amssymb}
\usepackage{mathtools}
\usepackage{amsthm}
\usepackage{bbm}
\usepackage{nicefrac}

\usepackage{tabularx}
\usepackage{multirow}

\usepackage{lscape}

\usepackage{mdframed}

\usepackage[capitalize,noabbrev]{cleveref}

\newcounter{insight}
\DeclareRobustCommand{\insight}{%
  \refstepcounter{insight}%
  Insight~\theinsight:\xspace%
}

\theoremstyle{plain}
\newtheorem{theorem}{Theorem}[section]
\newtheorem{proposition}[theorem]{Proposition}
\newtheorem{lemma}[theorem]{Lemma}

\theoremstyle{definition}

\newtheorem{assumption}[theorem]{Assumption}

\theoremstyle{remark}

\theoremstyle{theorem}
\newtheorem{informalproposition}[theorem]{Informal Proposition}
\newtheorem{informaltheorem}[theorem]{Informal Theorem}

\Crefname{assumption}{Assumption}{Assumptions}
\crefname{assumption}{Assumption}{Assumptions}

\Crefname{ALC@unique}{Line}{Lines}
\crefname{ALC@unique}{line}{lines}

\newcommand{\sref}[1]{\S\ref{#1}}

\allowdisplaybreaks

\makeatletter
\renewcommand{\paragraph}{%
  \@startsection{paragraph}{4}%
  {\z@}{0ex \@plus 0ex \@minus 0ex}{-1em}%
  {\normalfont\normalsize\bfseries}%
}
\makeatother

\newlist{lemenum}{enumerate}{1} %
\setlist[lemenum]{label=(\roman*),ref=\thelemma\,(\roman*),topsep=0pt}
\Crefname{lemenumi}{Lemma}{Lemmas}

\newlist{corenum}{enumerate}{1} %
\setlist[corenum]{label=(\roman*),ref=\thecorollary\,(\roman*),topsep=0pt}
\Crefname{corenumi}{Corollary}{Corollaries}

\newlist{thmenum}{enumerate}{1} %
\setlist[thmenum]{label=(\roman*),ref=\thetheorem\,(\roman*),topsep=0pt}
\Crefname{thmenumi}{Theorem}{Theorems}

\newlist{propenum}{enumerate}{1} %
\setlist[propenum]{label=(\roman*),ref=\thedefinition\,(\roman*),topsep=0pt}
\Crefname{propenumi}{Property}{Properties}

\newlist{assenum}{enumerate}{1} %
\setlist[assenum]{label=(\roman*),ref=\theassumption\,(\roman*),topsep=0pt}
\Crefname{assenumi}{Assumption}{Assumptions}

\usepackage{svg}
\usepackage{graphicx}
\usepackage{rotating}
\usepackage{tikz}
\usepackage{pgfplots}
\pgfplotsset{compat=newest}
\usetikzlibrary{shapes.geometric}

\definecolor{chaptercolor}{HTML}{1A254B}
\definecolor{darkblue}{HTML}{1A254B}
\definecolor{linkcolor}{HTML}{2B50AA}
\definecolor{citecolor}{HTML}{2B50AA}
\definecolor{lightlinkcolor}{HTML}{9A8F97}
\definecolor{darklinkcolor}{HTML}{1A254B}
\definecolor{light}{HTML}{F8F8F8}
\definecolor{lightblue}{HTML}{A7BED3}
\definecolor{red}{HTML}{F2545B}
\definecolor{blue}{HTML}{2b50aa}

\usepackage{import}
\usepackage{xifthen}
\usepackage{pdfpages}
\usepackage{transparent}

\NewDocumentCommand{\incfig}{mo}{
  \begin{center}
    \IfValueT{#2}{\def\svgwidth{#2}}{\def\svgwidth{\columnwidth}}
    \import{./figures/}{#1.pdf_tex}
  \end{center}
}

\usepackage{pgf}
\usepackage{adjustbox}

\NewDocumentCommand{\incplt}{O{\columnwidth}m}{%
  \begin{center}
    \adjustbox{center}{\adjustbox{width=#1+10pt}{\includegraphics[width=#1]{./plots/output/#2.pdf}}}
  \end{center}
}



\newcommand{\figref}[2]{Figure~\hyperref[#1]{\ref{#1} (#2)}}

\newcolumntype{S}{>{\scriptsize\selectfont}l}
\usepackage{tcolorbox}
\tcbuselibrary{breakable}
\tcbset{
  downbox/.style={
    colback=green!10,  %
    colframe=green!10, %
    boxrule=0pt,      %
    boxsep=1pt,       %
    left=1pt,         %
    right=1pt,         %
    top=0pt,          %
    bottom=0pt,       %
    baseline=3pt
  }
}
\tcbset{
  upbox/.style={
    colback=red!10,  %
    colframe=red!10, %
    boxrule=0pt,      %
    boxsep=1pt,       %
    left=1pt,         %
    right=1pt,         %
    top=0pt,          %
    bottom=0pt,       %
    baseline=3pt
  }
}
\newcommand{\downres}[1]{\smash{\tcbox[downbox]{\ensuremath{\downarrow\!#1}}}}
\newcommand{\upres}[1]{\smash{\tcbox[upbox]{\ensuremath{\uparrow\!#1}}}}
\newcommand{\reso}[1]{\,\scriptsize{(#1)}}

\newtcolorbox{examplebox}[2]{%
    colback=gray!20,       %
    colframe=black,        %
    width=\textwidth,      %
    colbacktitle=black,    %
    coltitle=white,        %
    fonttitle=\bfseries,   %
    boxsep=0pt,            %
    left=5pt,              %
    right=5pt,             %
    top=5pt,               %
    bottom=5pt,            %
    title={#1},            %
    after skip=\bigskipamount, %
    before skip=\bigskipamount, %
    breakable=true,
    }

\usepackage{aligned-overset}

\DeclareFontFamily{U}{mathb}{\hyphenchar\font45}
\DeclareFontShape{U}{mathb}{m}{n}{
      <5> <6> <7> <8> <9> <10> gen * mathb
      <10.95> mathb10 <12> <14.4> <17.28> <20.74> <24.88> mathb12
      }{}
\DeclareSymbolFont{mathb}{U}{mathb}{m}{n}
\DeclareFontSubstitution{U}{mathb}{m}{n}
\DeclareMathSymbol{\Asterisk}      {2}{mathb}{"06}
\newcommand{\mainsym}{\textcolor{blue}{\parentheses*{\hspace{-0.1em}\Asterisk\hspace{-0.1em}}}}

\newcommand{\psym}{\ensuremath{\textcolor{blue}{(\dagger)}}\xspace}

\newcommand*{\abs}[1]{| #1 |}

\NewDocumentCommand{\norm}{sm}{\IfBooleanTF{#1}{\|#2\|}{\left\| #2 \right\|}}
\NewDocumentCommand{\normF}{sm}{\IfBooleanTF{#1}{\|#2\|_{\mathrm{F}}}{\left\| #2 \right\|_{\mathrm{F}}}}
\NewDocumentCommand{\dTV}{sm}{d_{\mathrm{TV}}\IfBooleanTF{#1}{(#2)}{\left( #2 \right)}}

\newcommand*{\const}{\mathrm{const}}

\newcommand*{\method}{\textsc{SIFT}\xspace}
\newcommand*{\adamethod}{\textsc{Adaptive SIFT}\xspace}
\newcommand*{\methodfast}{\textsc{SIFT-Fast}\xspace}
\newcommand*{\methodl}{\textsc{SIFT}($\lambda'$)\xspace}
\newcommand*{\methodfastl}{\textsc{SIFT-Fast}($\lambda'$)\xspace}
\newcommand*{\methodp}{\textsc{SIFT}($\rho^2$)\xspace}

\newcommand*{\prompt}{\ensuremath{\vx^\star}}

\newcommand*{\loss}{\ensuremath{\mathcal{L}}}

\DeclareMathOperator*{\defeq}{\,\dot{=}\,}

\DeclareMathOperator*{\argmax}{arg\,max}
\DeclareMathOperator*{\argmin}{arg\,min}

\DeclareMathOperator*{\spn}{span}

\DeclarePairedDelimiter\parentheses{(}{)}
\DeclarePairedDelimiter\brackets{[}{]}
\DeclarePairedDelimiter\braces{\{}{\}}

\newcommand{\R}{\mathbb{R}}

\newcommand{\Nat}{\mathbb{N}}

\renewcommand{\vec}[1]{\boldsymbol{#1}}
\newcommand{\mat}[1]{\boldsymbol{#1}}

\newcommand{\spa}[1]{\mathcal{#1}}
\newcommand{\opt}[1]{#1^\star}

\NewDocumentCommand{\irred}{som}{\ensuremath{\sigma_{\hspace{-1pt}\infty}\IfBooleanTF{#1}{^2}{}(#3\IfValueTF{#2}{;#2}{})}}

\NewDocumentCommand{\Ind}{m}{\mathbbm{1}\{{#1}\}}
\NewDocumentCommand{\fnPr}{}{\mathbb{P}}
\RenewDocumentCommand{\Pr}{om}{\fnPr\IfValueT{#1}{_{#1}}\parentheses*{#2}}
\RenewDocumentCommand{\H}{mo}{\mathrm{H}\IfValueTF{#2}{\!\left[#1\ \middle|\ #2\right]}{\brackets*{#1}}}
\NewDocumentCommand{\Hsm}{mo}{\mathrm{H}\IfValueTF{#2}{[#1 \mid #2]}{\brackets{#1}}}
\NewDocumentCommand{\I}{mmo}{\mathrm{I}\IfValueTF{#3}{\!\left(#1;#2\ \middle|\ #3\right)}{\parentheses*{#1; #2}}}
\NewDocumentCommand{\Ism}{mmo}{\mathrm{I}\IfValueTF{#3}{(#1;#2 \mid #3)}{\parentheses{#1; #2}}}

\NewDocumentCommand{\E}{somo}{\ensuremath{\mathbb{E}\IfValueT{#2}{_{#2}}{} \IfBooleanTF{#1}{#3}{\IfValueTF{#4}{\!\left[#3\ \middle|\ #4\right]}{\brackets*{#3}}}}}
\NewDocumentCommand{\Esm}{somo}{\ensuremath{\mathbb{E}\IfValueT{#2}{_{#2}}{} \IfBooleanTF{#1}{#3}{\IfValueTF{#4}{\!\left[#3\ \middle|\ #4\right]}{\brackets{#3}}}}}
\NewDocumentCommand{\Var}{somo}{\mathrm{Var}\IfValueT{#2}{_{#2}}{} \IfBooleanTF{#1}{#3}{\IfValueTF{#4}{\!\left(#3\ \middle|\ #4\right)}{\parentheses*{#3}}}}
\NewDocumentCommand{\Varsm}{somo}{\mathrm{Var}\IfValueT{#2}{_{#2}}{} \IfBooleanTF{#1}{#3}{\IfValueTF{#4}{\left(#3\ \middle|\ #4\right)}{\parentheses{#3}}}}
\NewDocumentCommand{\Cov}{som}{\mathrm{Cov}\IfValueT{#2}{_{#2}}{} \IfBooleanTF{#1}{#3}{\brackets*{#3}}}
\NewDocumentCommand{\Cor}{som}{\mathrm{Cor}\IfValueT{#2}{_{#2}}{} \IfBooleanTF{#1}{#3}{\brackets*{#3}}}

\NewDocumentCommand{\grad}{e_}{\boldsymbol{\nabla}\IfValueT{#1}{_{\!\!#1}\,}}

\NewDocumentCommand{\BigO}{m}{O\parentheses*{#1}}
\NewDocumentCommand{\BigOTil}{m}{\widetilde{O}\parentheses*{#1}}

\NewDocumentCommand{\transpose}{m}{#1^\top}
\NewDocumentCommand{\inv}{m}{#1^{-1}}

\NewDocumentCommand{\diag}{som}{\mathrm{diag}\IfValueT{#2}{_{#2}}{} \IfBooleanTF{#1}{\braces{#3}}{\braces*{#3}}}

\NewDocumentCommand{\N}{somm}{\mathcal{N}\IfBooleanTF{#1}{\left(}{(}\IfValueT{#2}{#2;}{} #3, #4\IfBooleanTF{#1}{\right)}{)}}
\NewDocumentCommand{\GP}{omm}{\mathcal{GP}(\IfValueT{#1}{#1;}{} #2, #3)}

\newcommand{\fpre}{f^{\mathrm{pre}}}
\newcommand{\vzero}{\vec{0}}

\newcommand{\ve}{\vec{e}}
\newcommand{\vf}{\vec{f}}

\newcommand{\vfsub}[1]{\vec{f}_{\!#1}}

\newcommand{\vfpre}{\vec{f}^{\mathrm{pre}}}

\newcommand{\vk}{\vec{k}}

\newcommand{\vv}{\vec{v}}

\newcommand{\vw}{\vec{w}}

\newcommand{\vwpre}{{\vw^{\mathrm{pre}}}}
\newcommand{\vx}{\vec{x}}
\newcommand{\vs}{\vec{s}}
\newcommand{\vsp}{\vec{s'}}

\newcommand{\vxp}{\vec{x'}}

\newcommand{\vy}{\vec{y}}
\newcommand{\vysub}[1]{\vec{y}_{\!#1}}

\newcommand{\vmusub}[1]{\boldsymbol{\mu}_{\!#1}}

\newcommand{\vphi}{\boldsymbol{\phi}}

\newcommand{\vtheta}{\boldsymbol{\theta}}

\newcommand{\mzero}{\mat{0}}
\newcommand{\mA}{\mat{A}}
\newcommand{\mB}{\mat{B}}

\newcommand{\mC}{\mat{C}}

\newcommand{\mG}{\mat{G}}

\newcommand{\mI}{\mat{I}}

\newcommand{\mK}{\mat{K}}
\newcommand{\mKsub}[1]{\mat{K}_{\!#1}}

\newcommand{\mPhi}{\mat{\Phi}}
\newcommand{\mPi}{\mat{\Pi}}

\newcommand{\mV}{\mat{V}}
\newcommand{\mVsub}[1]{\mV_{\!#1}}
\newcommand{\mW}{\mat{W}}
\newcommand{\mWsub}[1]{\mW_{\!#1}}
\newcommand{\mWhat}{\mat{\widehat{W}}}
\newcommand{\mWpre}{{\mat{W}^{\mathrm{pre}}}}

\newcommand{\mLambda}{\mat{\Lambda}}
\newcommand{\mSigma}{\mat{\Sigma}}

\newcommand{\spA}{\spa{A}}

\newcommand{\spD}{\spa{D}}

\newcommand{\spW}{\spa{W}}
\newcommand{\spX}{\spa{X}}
\newcommand{\spY}{\spa{Y}}

\title{Efficiently Learning at Test-Time: \\ Active Fine-Tuning of LLMs}

\author{Jonas Hübotter\thanks{Correspondence to \texttt{jonas.huebotter@inf.ethz.ch}}, \,Sascha Bongni, Ido Hakimi, Andreas Krause \\
ETH Z\"urich, Switzerland
}

\iclrfinalcopy %
\begin{document}

\maketitle

\vspace{-13pt}
\begin{abstract}
  Recent efforts in fine-tuning language models often rely on automatic data selection, commonly using Nearest Neighbors retrieval from large datasets.
  However, we theoretically show that this approach tends to select redundant data, limiting its effectiveness or even hurting performance.
  To address this, we introduce \method, a data selection algorithm designed to reduce uncertainty about the model's response given a prompt, which unifies ideas from retrieval and active learning.
  Whereas Nearest Neighbor retrieval typically fails in the presence of information duplication, \method accounts for information duplication and optimizes the overall information gain of the selected examples.
  We focus our evaluations on fine-tuning at test-time for prompt-specific language modeling on the Pile dataset, and show that \method consistently outperforms Nearest Neighbor retrieval, with minimal computational overhead. %
  Moreover, we show that our uncertainty estimates can predict the performance gain of test-time fine-tuning, and use this to develop an adaptive algorithm that invests test-time compute proportional to realized performance gains.
  We provide the \href{https://github.com/jonhue/activeft}{\texttt{activeft}} (Active Fine-Tuning) library which can be used as a drop-in replacement for Nearest Neighbor retrieval.\looseness=-1\looseness=-1
\end{abstract}
\vspace{-5pt}

\vspace{-2pt}
\section{Introduction}\label{sec:introduction}
\vspace{-2pt}

\begin{wrapfigure}{r}{0.4\textwidth}
  \vspace{-27pt}
  \incplt[0.4\textwidth]{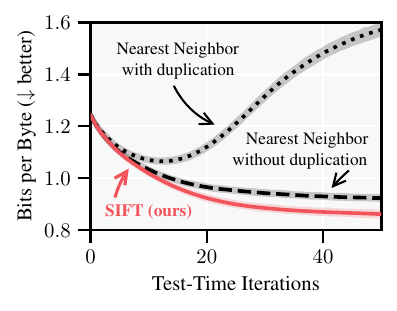}
  \vspace{-15pt}
  \caption{Selecting fine-tuning data using \method~(\textcolor{red}{red}) robustly outperforms Nearest Neighbor retrieval~(black) and avoids the failure-mode of Nearest Neighbor retrieval where the same data is selected repeatedly, which is a common result of information duplication.\looseness=-1}
  \label{fig:main}
  \vspace{-10pt}
\end{wrapfigure}

The standard paradigm of machine learning separates training and testing.
Training aims to learn a model by \emph{inductively} extracting general rules from data, and testing applies this model to new, unseen data.
We investigate an alternative \emph{transductive} paradigm where the model is fine-tuned at test-time specifically to the given task.
Variations of this paradigm have been studied since the inception of machine learning as a field.
Early examples are local learning~\citep{cleveland1979robust,cleveland1988locally,atkeson1997locally} and local fine-tuning~\citep{bottou1992local}.
More recently, with the advent of large pre-trained models which have good representations and are strong foundations for fine-tuning, the idea of \emph{test-time fine-tuning} has re-gained attention~\citep{krause2018dynamic,krause2019dynamic,sun2020test}.
\cite{hardt2023test} show that fine-tuning on data related to the prompt to a large language model (LLM) can significantly improve performance.
Also, test-time fine-tuning is the central component of state-of-the-art approaches to the ARC challenge~\citep{chollet2019measure,cole2023dataset,akyurek2024surprising}, a non-saturated benchmark which is intended to test reasoning capabilities based on ``core knowledge'' rather than mere memorization.\looseness=-1

\paragraph{\emph{Active Fine-Tuning:} Effective data selection for fine-tuning LLMs}

Test-time fine-tuning demands automatic data selection since manually selecting data for each test instance is infeasible.
Moreover, the sample efficiency of test-time fine-tuning is a central bottleneck as the number of gradient steps is directly proportional to inference time.
Previous works on data selection for fine-tuning LLMs have fundamentally relied on Nearest Neighbor retrieval within some embedding space~\citep{hardt2023test,xia2024less}.
We show theoretically and empirically that Nearest Neighbor retrieval is insufficient for fine-tuning LLMs since it can lead to the selection of redundant data.
Notably, recent works using influence functions for data selection such as \cite{xia2024less} have pointed out this limitation.
In contrast, a large body of work on (inductive) active learning has studied non-redundant data selection~\citep[e.g.,][]{sener2017active,ash2019deep,yehuda2022active,kirsch2019batchbald} that covers the data manifold well~(cf.~\cref{fig:schema}).
Retrieval and active learning can be seen as two extreme ends of a spectrum: retrieval selects relevant but potentially redundant data, while active learning selects diverse but potentially irrelevant data.\looseness=-1

\begin{wrapfigure}{r}{0.4\textwidth}
  \vspace{-0.5cm}
  \centering
  \includesvg[width=0.4\columnwidth]{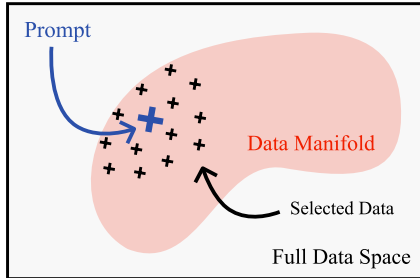}
  \vspace{-0.5cm}
  \caption{We consider a scenario where we have a pre-trained language model capturing a latent manifold~(\textcolor{red}{red}) in the large sequence space~(\textcolor{gray}{white}). We aim to improve the models performance on a given prompt~(\textcolor{blue}{blue}) by \emph{efficiently} fine-tuning the model on \emph{few} relevant and diverse data points~(\textcolor{black}{black}) at test-time.\looseness=-1}
  \label{fig:schema}
  \vspace{-1cm}
\end{wrapfigure}

We bridge this gap by unifying ideas from retrieval and active learning in \method, an algorithm based on emerging literature on transductive active learning~\citep{hubotter2024transductive} that \emph{\textbf{S}elects \textbf{I}nformative data for \textbf{F}ine-\textbf{T}uning} as illustrated in \cref{fig:schema}.
Our results show that \method leads to substantial improvements in performance and efficiency.
Concretely, we show the following:\looseness=-1

\begin{enumerate}[align=left, leftmargin=12pt, labelindent=\parindent, listparindent=\parindent, labelwidth=0pt, itemindent=!]
  \item \textbf{Nearest Neighbor retrieval is insufficient~(\sref{sec:problem_setting})}:
  We prove that selecting the top-$N$ highest scoring points from a large dataset according to a fixed scoring function leads to the selection of redundant data.\looseness=-1
  \item \textbf{\method estimates uncertainty about responses~(\sref{sec:preliminaries})}:
  We develop the notion of \emph{uncertainty about the response to the prompt}, and derive an anytime high probability bound to the total variation distance between the model's distribution over responses and the ground truth which is governed by this uncertainty.\looseness=-1
  \looseness=-1
\end{enumerate}
\leavevmode
\begin{enumerate}[resume, align=left, leftmargin=12pt, labelindent=\parindent, listparindent=\parindent, labelwidth=0pt, itemindent=!]
  \vspace{-1.3\baselineskip}
  \item \textbf{\method provably reduces uncertainty~(\sref{sec:method})}: We propose \method, an algorithm that selects data which reduces uncertainty about the response to the prompt. We prove statistical rates for the uncertainty reduction (\sref{sec:convergence}) and show that \method is compute-efficient, with minimal overhead compared to Nearest Neighbor retrieval (\sref{sec:compute_efficiency}).\looseness=-1
  \item \textbf{\method performs better and is more robust than Nearest Neighbor retrieval~(\sref{sec:results})}: We find that fine-tuning an LLM on data selected by \method consistently and robustly improves performance, which is not the case with Nearest Neighbor retrieval.
  Moreover, our results suggest that at test-time, an LLM might be able to learn more effectively through fine-tuning than from its context.\looseness=-1
  \item \textbf{\method can invest test-time compute proportionally to performance gains~(\sref{sec:compute_proportional})}: We observe that our uncertainty estimates can accurately predict the performance gain of test-time fine-tuning.
  Motivated by this, we dynamically adapt compute to the expected performance gain.\looseness=-1
\end{enumerate}

\section{Test-Time Fine-Tuning}\label{sec:problem_setting} %

We define test-time fine-tuning of LLMs~\citep{hardt2023test} as follows.
We consider a domain~$\spX$ of token sequences and assume that we have access to a large dataset of examples $\spD \subseteq \spX$ which we call the \emph{data space}.
We further assume that we have access to a pre-trained autoregressive language model that maps token sequences $\spX$ to probability distributions over the next token from a vocabulary of size~$V$.
Our work addresses the central question: \begin{center}
  \emph{Given a prompt $\prompt \in \spX$, how can we effectively select fine-tuning data \\ from the large dataset $\spD$ such that the fine-tuned model performs well on the prompt?}
\end{center}
We then fine-tune the model for a single gradient step on each selected sequence.\looseness=-1

Locally adjusting a model at test-time has gained popularity with few-shot in-context learning~\citep{brown2020language,wei2022chain,bubeck2023sparks,openai2024learning} and retrieval augmented generation \citep[RAG,][]{lewis2020retrieval,guu2020retrieval,borgeaud2022improving}.
In contrast to this approach, test-time fine-tuning works by fine-tuning the parameters of a pre-trained model at test-time specifically to each prompt.
Notably, test-time fine-tuning takes time linear in the number of tokens whereas in-context learning with a transformer has quadratic complexity~\citep{vaswani2017attention}.
Next to this, \cite{hardt2023test} and other works have found (test-time) fine-tuning to perform substantially better than in-context learning~\citep{hu2021lora,mosbach2023few}.
This work further improves the performance of test-time fine-tuning.
Prior work has also studied how one can explicitly meta-learn the ability to perform test-time fine-tuning~\citep{finn2017model,sun2024learning}, though we find this capability to emerge even from models that are not explicitly trained in this way.\looseness=-1

The central question studied in this work also arises when fine-tuning LLMs during post-training.
For example, in targeted instruction tuning, the goal is to fine-tune a model to obtain desired capabilities, which are commonly embodied by a set of examples~$\prompt$~\citep{xia2024less}.
The extension of our work to such a ``batched'' setting is straightforward.\looseness=-1

\subsection{Nearest Neighbor Retrieval is Insufficient}\label{sec:problem_setting:nn_insufficient}

\begin{wrapfigure}{r}{0.5\textwidth}
  \vspace{-15pt}
  \begin{mdframed}[
    backgroundcolor=light,
    linecolor=black,
    linewidth=1.2pt,
    innerleftmargin=5pt,
    innerrightmargin=5pt,
    innertopmargin=5pt,
    innerbottommargin=8pt
  ]
    \textbf{Prompt:} What is the age of Michael Jordan and \textcolor{red}{how many kids does he have?} \\[5pt]
    \textbf{Nearest Neighbor:}
    {\small\begin{enumerate}[wide, labelindent=0pt]
      \item The age of Michael Jordan is 61 years.
      \item Michael Jordan was born on February 17, 1963.
    \end{enumerate}}\vspace{3pt}
    \textbf{\method (ours):}
    {\small\begin{enumerate}[wide, labelindent=0pt]
      \item The age of Michael Jordan is 61 years.
      \item \textcolor{red}{Michael Jordan has five children.}
    \end{enumerate}}
  \end{mdframed}
  \vspace{-7pt}
  \caption{We retrieve two data points to answer the prompt. Nearest Neighbor selects redundant data, while \method yields maximal information~(cf.~\sref{sec:examples}).\looseness=-1}
  \label{fig:qualitative_example}
  \vspace{-0.4cm}
\end{wrapfigure}

Prior work on data selection for fine-tuning has relied on Nearest Neighbor retrieval.
The idea of making predictions on $\prompt$ depending on its nearest neighbors has been around as long as machine learning itself~\citep{fix1951discriminatory,cover1967nearest}.
\cite{bottou1992local} were the first to apply this idea to the fine-tuning of convolutional neural networks by selecting the nearest neighbors of a test image in pixel-space.
More recently, due to advances in representation learning~\citep{devlin2018bert,reimers2019sentence} and efficiency~\citep[e.g.,][]{johnson2019billion}, Nearest Neighbor retrieval has regained attention and been applied to test-time fine-tuning~\citep{hardt2023test}.\looseness=-1

\cite{xia2024less} use influence functions~\citep{cook1977detection,koh2017understanding,pruthi2020estimating} to select data for fine-tuning LLMs.
This line of work aims to select data that reduces a first-order Taylor approximation to the test loss after fine-tuning, an approach that corresponds to Nearest Neighbor retrieval in a certain embedding space.
They highlight two main limitations of the use of influence functions and Nearest Neighbor retrieval for data selection:\looseness=-1
\vspace{-5pt}\begin{itemize}[align=left, leftmargin=9pt, labelindent=\parindent, listparindent=\parindent, labelwidth=0pt, itemindent=!]
    \item Nearest Neighbor retrieval leads to the selection of redundant data.
    \Cref{fig:qualitative_example} illustrates this limitation with a qualitative example.
    We formalize this limitation in \cref{prop:insufficiency_nn}, which we summarize here informally:
    \begin{informalproposition}\label{prop:insufficiency_nn_informal}
      Selecting the top-$N$ nearest neighbors from the data space (according to cosine similarity or Euclidean distance) may not reduce the uncertainty about the response to the prompt beyond fine-tuning on the closest neighbor.
      Every additional passage may be redundant.\looseness=-1
    \end{informalproposition}

    \item Nearest Neighbor retrieval selects data with high positive cosine similarity to the prompt.
    Yet, data with high \emph{negative} cosine similarity can be equally informative as data with high positive cosine similarity~\citep[Appendix K.2]{xia2024less}, but is ignored by standard Nearest Neighbor retrieval.
\end{itemize}
\vspace{-5pt}
In this work, we propose \method and show that it naturally addresses both limitations.
\reviewtext{\method unifies work on retrieval, which finds relevant but redundant data, and active learning~(AL), which finds non-redundant but irrelevant data.
In \sref{sec:extended_related_work}, we discuss how \method relates to prior work in retrieval and AL.}\looseness=-1

\section{Preliminaries: Uncertainty Estimation for Fine-Tuning}\label{sec:preliminaries}

We suppose the assigned probability that $y \in [V]$ is the class label of an input $\vx \in \spX$ is given by $s_y(\vf^\star(\vx))$, where $s_y$ is the softmax $\smash{s_y(\vf) \defeq \exp(f_y) / (\sum_{i=1}^V \exp(f_i))}$.
That is, $\vf^\star(\vx)$ denotes the ``ground truth'' logits for a given input $\vx$.
In the context of language modeling, $V$ is the number of tokens in the vocabulary, and $y$ denotes the index of the next token.
We defer all proofs to \cref{sec:proofs}.\looseness=-1

We use a surrogate model to quantify the informativeness of data, which we define next.\looseness=-1

\begin{assumption}[\emph{Surrogate model:} Linear model class within a known latent space]\label{assumption:linear}
  We assume ${\vf^\star(\vx) = \mW^\star \vphi(\vx)}$ with unknown ${\mW^\star \in \R^{V \times d}}$ and where ${\vphi(\cdot) \in \R^d}$ denotes known embeddings.
\end{assumption}\vspace{-6pt}

The surrogate model uses the latent space induced by the pre-trained model to describe the data manifold. %
We emphasize that while \method relies on this surrogate model for data selection, it still fine-tunes the full pre-trained model, including latent features.
Surrogate dense embedding models of this kind have been used extensively for data selection via Nearest Neighbor retrieval~\citep[e.g.,][]{lewis2020retrieval,karpukhin2020dense,borgeaud2022improving,xia2024less}, and to understand the training dynamics and generalization of large neural networks \citep[e.g.,][]{jacot2018neural,lee2019wide,malladi2023kernel,templeton2024scaling,park2024linear}.
Furthermore, a surrogate model that assumes linearity in some fixed latent space may be a reasonable approximation for test-time fine-tuning since the latent space of the unfrozen model is not expected to change substantially by a few gradient steps.\looseness=-1

In this work, we explore a scenario where we have a pre-trained model $\smash{\vfpre(\vx) = \mWpre \vphi(\vx)}$.
We let ${\vf(\vx; \mW) \defeq \mW \vphi(\vx)}$ and denote by $\loss(\mW; D)$ the negative log-likelihood loss of $\vf(\cdot; \mW)$ on a dataset $D$ of inputs $\vx$ with corresponding class labels $y$: $\smash{\loss(\mW; D) \defeq -\sum_{(\vx, y) \in D} \log s_y(\vf(\vx; \mW))}$.\looseness=-1

\paragraph{Uncertainty Estimation}

Our first intermediate goal is to estimate the uncertainty about the response to a given prompt $\prompt$ after having fine-tuned on selected data $D_n$ of size $n$.
To this end, we generalize prior work on confidence sets under categorical feedback~\citep[i.e., class feedback,][]{amani2021ucb,zhang2024online} to our fine-tuning setting.
We consider the function class ${\spW_B \defeq \{\mW \in \R^{V \times d} \mid \normF{\mW - \mWpre} \leq B\}}$ where $\normF{\cdot}$ denotes the Frobenius norm %
and with $B$ a constant such that $\mW^\star \in \spW_B$.
Then given data~$D_n$, we can refine the prior estimate~$\mWpre$ of~$\mW^\star$ by minimizing the regularized negative log-likelihood loss \vspace{-2pt}\begin{align}
  \loss^\lambda(\mW; D_n) \defeq \loss(\mW; D_n) + \frac{\lambda}{2} \normF{\mW - \mWpre}^2 \label{eq:reg_nll_loss}
  \vspace{-2pt}
\end{align} with regularization coefficient $\lambda > 0$.
We write its minimizer as $\smash{\mWsub{n} \defeq \argmin_{\mW \in \spW_B} \loss^\lambda(\mW; D_n)}$.
We will further denote the ground truth probability distribution over the response to $\vx$ by $\vs^\star(\vx) \defeq \vs(\vf^\star(\vx))$ and our approximation after selection of $n$ samples by $\vs_n(\vx) \defeq \vs(\vf(\vx; \mWsub{n}))$.\looseness=-1

We construct confidence sets of the form $[\vs_n(\vx) \pm \beta_n(\delta) \sigma_n(\vx)]$ centered around this prediction, and show their uniform anytime validity.
The width of these sets is characterized by our central quantity~$\sigma_n(\vx)$ which we define next.
We consider the inner-product kernel $k(\vx,\vxp) \defeq \transpose{\vphi(\vx)} \vphi(\vxp)$ and define for a set of inputs $X = \{\vx_1, \dots, \vx_n\} \subseteq \spD$: \begin{align}
  \sigma_{X}^2(\vx) \defeq k(\vx, \vx) - \transpose{\vk_{X}}(\vx) \inv{(\mKsub{X} + \lambda\kappa\mI_{n})} \vk_{X}(\vx) \label{eq:variance}
\end{align} where $\vk_{X}(\vx) = (k(\vx_1,\vx), \dots, k(\vx_n,\vx)) \in \R^n$, $\mKsub{X} \in \R^{n \times n}$ is the kernel matrix satisfying $(\mKsub{X})_{i,j} = k(\vx_i,\vx_j)$, and $\smash{\kappa \defeq \sup_{\vx \in \spX, \mW \in \spW_B} 1/\lambda_{\min}(\mA(\vx; \mW))}$.
Here, $\mA(\vx; \mW) \in \R^{V \times V}$ is the matrix satisfying ${(\mA(\vx; \mW))_{i,j} \defeq s_i(\vx; \mW)(\Ind{i = j} - s_j(\vx; \mW))}$ which is the proper generalization of the derivative of the sigmoid function, standard in the analysis of binary feedback~\citep{faury2020improved,pasztor2024bandits}. %
We write $\smash{\sigma_n^2(\vx) \defeq \sigma_{X_n}^2(\vx)}$ where $X_n \subseteq \spD \subseteq \spX$ are the inputs in $D_n$.
With this we are ready to state our first result, namely that for careful choice of $\beta_n(\delta)$, the confidence sets contain $\vs^\star(\vx)$ simultaneously for all $\vx \in \spX$ and $n \geq 1$ with probability at least $1 - \delta$.\looseness=-1

\begin{theorem}[Confidence Sets]\label{thm:confidence_sets}
  Let \cref{assumption:linear} hold and $\mW^\star \in \spW_B$.
  Let $\delta \in (0,1)$ and set \begin{align}
    \beta_n(\delta) \defeq 2 \sqrt{V (1 + 2 B)} \brackets*{B + \frac{L V^{3/2} d}{\lambda} \log\parentheses*{\frac{2}{\delta}\sqrt{1 + \frac{n}{d \lambda}}}} \in \BigO{\log(n / \delta)} \label{eq:beta}
  \end{align} where ${L \defeq \sup_{\vx \in \spX, \mW \in \spW_B} \lambda_{\max}(\mA(\vx; \mW))}$.
  Then \begin{align*}
    \Pr{\forall n \geq 1, \vx \in \spX : \dTV*{\vs_n(\vx), \vs^\star(\vx)} \leq \beta_n(\delta) \sigma_n(\vx)} \geq 1 - \delta
  \end{align*} where $\dTV{\vs,\vsp} \defeq \frac{1}{2} \sum_i \abs{s_i - s'_i}$ is the total variation distance.
\end{theorem}\vspace{-5pt}

We use $\sigma_n(\vx)$ as a proxy to the \emph{uncertainty about the response to $\vx$} after having fine-tuned on the selected data~$D_n$, since it directly governs the size of the confidence sets around our current estimate of response probabilities.
This uncertainty is a key quantity not just in classification:
In \cref{sec:proofs_regression}, we state analogous confidence sets for regression with the standard squared error loss, building on results by \cite{abbasi2013online} and \cite{chowdhury2017kernelized}.\looseness=-1

\paragraph{The Close Relationship of Regularized Loss Minimization and Test-Time Fine-Tuning}

Recall that test-time fine-tuning does not solve the regularized objective of~\cref{eq:reg_nll_loss}, but instead takes a single gradient step.
So why do we expect the surrogate model $\vf(\cdot; \mWsub{n})$ be closely related to the fine-tuned~$\vfpre$?
To answer this question, we contrast two alternative models: \begin{itemize}
  \item $\mW_{\!\lambda} \defeq \argmin_{\mW} \loss^\lambda(\mW)$, \hfill \emph{(minimizer of regularized loss)}
  \item $\mWhat_{\!\eta} \defeq \mWpre - \eta \grad \loss(\mWpre)$ with any step size $\eta > 0$, \hfill \emph{(single gradient-step fine-tuning)}
\end{itemize} where we keep the dataset $D$ fixed and omit the dependency on $D$.
Our following proposition shows that both models are close if the loss landscape is relatively smooth and for careful choice of $\lambda \approx \frac{1}{\eta}$.\looseness=-1

\begin{proposition}\label{prop:reg_loss_min_vs_ttft}
  It holds that $\normF*{\mWsub{1/\eta} - \mWhat_{\!\eta}\,} \leq \eta \, \normF*{\grad \loss(\mWsub{1/\eta}) - \grad \loss(\mWpre)}$.
\end{proposition}\vspace{-5pt}

Recent works have also observed $\mWsub{1/\eta} \approx \mWhat_{\!\eta}$ empirically~\citep{ali2019continuous,ali2020implicit}.
Intuitively, with a larger step size, $\smash{\mWhat_{\!\eta}}$ is farther away from $\mWpre$, and hence corresponds to the regularized estimate with less regularization.
This connection between regularized loss minimization and test-time fine-tuning is closely linked to the tight connection between regularization and early stopping~\citep{morgan1989generalization,yao2007early,li2020gradient}.
We will use this connection in the following to derive \method in the context of fine-tuning.\looseness=-1

\section{\method: Efficiently Reducing Uncertainty about the Response}\label{sec:method}

We introduce \method, an algorithm for selecting data for fine-tuning that effectively reduces the uncertainty about the response to the prompt~$\prompt \in \spX$.
Note that we can compute the uncertainty $\sigma_X(\prompt)$ about the response to the prompt $\prompt$ for any selected data ${X \subseteq \spD}$ in closed-form, since its definition (cf.~\cref{eq:variance}) depends only on the selected inputs~$X$. %
\method minimizes this uncertainty about $\prompt$:\looseness=-1 \begin{align*}
  \vx_{n+1} &\defeq \argmin_{\vx \in \spD} \sigma_{X_n \cup \{\vx\}}^2(\prompt) = \argmax_{\vx \in \spD} \transpose{\vk_{X_n \cup \{\vx\}}}(\prompt) \inv{(\mKsub{X_n \cup \{\vx\}} + \lambda' \mI_{n+1})} \vk_{X_n \cup \{\vx\}}(\prompt). \tag{\methodl}
\end{align*}
\method selects data that minimizes a bound on the approximation error of the surrogate model, and then fine-tunes the full LLM using this data.
We discuss the design choices, including the choice of embeddings, that make \method efficient in \sref{sec:compute_efficiency}.
In \sref{sec:method_details:balancing_relevance_diversity}, we illustrate with an example of how \method balances relevance and diversity, where we also see that the free parameter $\lambda' = \lambda \kappa$ controls this trade-off.
Larger $\lambda'$ emphasize relevance of selected data, while smaller $\lambda'$ emphasize diversity.
Probabilistically, \method can be interpreted as maximizing the information gain of the selected data~$X_n$ on the response to the prompt $\prompt$ in a tractable model.
We formally introduce this interpretation of \method in \sref{sec:method_maximizes_info_gain}.\looseness=-1

\subsection{Uncertainty Provably Vanishes}\label{sec:convergence}

We prove that unlike with Nearest Neighbor retrieval, the uncertainty about the response to the prompt vanishes if \method is used to select data for fine-tuning.
We give an informal overview here, and defer the formal treatment to \sref{sec:method_details:uncertainty_vanishes}.
Our theoretical analysis shows that test-time fine-tuning can fully reduce uncertainty only if the data space contains sufficient information to determine the correct response.
If the data space does not contain all relevant information, the remaining uncertainty is quantified by the limiting uncertainty after seeing ``all data in the data space infinitely often'', which we call the \emph{irreducible uncertainty} and denote by $\irred{\prompt}$.
We provide the formal definition in \sref{sec:method_details:uncertainty_vanishes}, but intuitively, the irreducible uncertainty is the largest quantity satisfying $\sigma_X(\prompt) \geq \irred{\prompt}$ for all $X \subseteq \spD$.
We then specialize the result of \cite{hubotter2024transductive} to show that the uncertainty about the response to the prompt shrinks at the rate~$\smash{\BigOTil{1/\sqrt{n}}}$ until it reaches the irreducible uncertainty:
\begin{informaltheorem}[Convergence Guarantee]\label{informal_thm:convergence}
  Fix any $\lambda' > 0$ and let \methodl select $X_n$ from the data space $\spD$.
  Then for all $n \geq 1$ and $\prompt \in \spX$, \begin{align*}
    \sigma_n^2(\prompt) - \irred*{\prompt} \leq \frac{\BigO{\lambda' \log n}}{\sqrt{n}}.
  \end{align*}
\end{informaltheorem}
Naturally, convergence is slower with a larger regularization parameter / smaller step size.
Notably, the irreducible uncertainty depends on the data space.
With a large and diverse data space, the irreducible uncertainty is typically negligible.
This statistical guarantee is a key property of \method.
As we show in \cref{prop:insufficiency_nn}, Nearest Neighbor retrieval fails to satisfy a guarantee of this kind.\looseness=-1

\subsection{Compute-Efficient Data Selection}\label{sec:compute_efficiency}

We have established how to select informative data for fine-tuning.
Next to good statistical efficiency, good computational efficiency is key for selecting data at test-time.
In the following, we describe design choices such that \method has negligible overhead compared to Nearest Neighbor retrieval.\looseness=-1

\paragraph{Sequence-Level Selection}
In the self-supervised paradigm, each sequence of tokens $\vx \in \spD$ corresponds to a dataset of next-token predictions $x_{1:k} \mapsto x_{k+1}$.
Rather than selecting individual next-token predictions from the data space of all sub-sequences $x_{1:k}$, we select full sequences $\vx$ from the significantly smaller data space $\spD$, then fine-tune for a single gradient step on each sub-sequence within $\vx$.
This is a common practice in prior works that use Nearest Neighbor retrieval for data selection~\citep[e.g.,][]{xia2024less,hardt2023test}.\looseness=-1

\paragraph{Surrogate Sequence Embedders}
We use a surrogate sequence embedding model to generate embeddings of the data space and prompts.
We use the same embedding model as \cite{hardt2023test} which is a large Roberta model~\citep{liu2019roberta} with 355M parameters that was fine-tuned for one pass on the Pile training set.
The embedding dimension is $1024$.
Unlike \cite{hardt2023test}, we additionally normalize the embeddings to unit length, the reasons for which we discuss in \sref{sec:analysis_active_fine_tuning}.\looseness=-1

We obtain decent performance with this surrogate model.
Nevertheless, our theoretical results indicate that using embeddings extracted from the LLM to be fine-tuned could further improve the performance of \method.
Empirical neural tangent embeddings~\citep{wei2022more,holzmuller2023framework} and influence function embeddings~\citep{xia2024less} can be implemented efficiently and offer alternative latent spaces capturing the pre-trained model.
We hypothesize that the decent performance of the surrogate model is explained by the similarity of emergent latent spaces of language models that were trained on similar data.\looseness=-1

\begin{wrapfigure}{r}{0.4\textwidth}
  \vspace{-25pt}
  \incplt[0.4\textwidth]{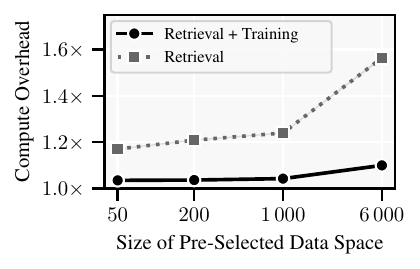}
  \vspace{-15pt}
  \caption{The (multiplicative) computational overhead of \method compared to Nearest Neighbor retrieval is minimal. The compute overhead with a 1k data space is less than $1.05\times$.}
  \label{fig:k_ablation}
  \vspace{-10pt}
\end{wrapfigure}
\paragraph{Efficient Implementation of \method}
In our experiments, we pre-select~$200$~candidates via Nearest Neighbor retrieval with Faiss~\citep{johnson2019billion} and then apply \method to select $50$~sequences from this smaller data space.
On the Pile dataset, we find that performance can be increased further by pre-selecting more candidates~(cf.~\cref{fig:k_ablation_perf} in \sref{sec:efficient_computation_via_conditional_embeddings}) but the marginal gains diminish.
The precise performance benefit of pre-selecting more candidates may differ on other datasets.
We describe in \sref{sec:efficient_computation_via_conditional_embeddings} how \method can be solved iteratively without computing the inverse in every iteration.
When a matrix of the size of the pre-selected data space fits in GPU memory, we find that \method has a negligible computational overhead compared to Nearest Neighbor retrieval.
We report results with an NVIDIA RTX 4090 GPU in \cref{fig:k_ablation}.\footnote{We use the client-server architecture described by \cite{hardt2023test} with CPU-only servers.}
While our main implementation of \method is fast if the data space is small, it does not scale linearly with the size of the data space~$K$.
In \sref{sec:efficient_computation_via_conditional_embeddings}, we show that a priority queue can be used to achieve an almost-linear runtime in $K$.\looseness=-1

\vspace{-3pt}
\section{Results}\label{sec:results}
\vspace{-3pt}

We focus on language modeling with causal language models.
Following \cite{hardt2023test}, we fine-tune a pre-trained LLM for a single gradient step each on $N = 50$ selected data points in the order that they are selected, most to least relevant.
We use the Pile dataset~\citep{gao2020pile} for evaluation, restricting our use to data which is obtained and used in compliance with the terms of service of the data host.
This version of the Pile contains a diverse set of 17 high-quality sub-datasets, ranging from Q\&A to code, scientific publications, math, and more.
Concretely, we use the Pile training set containing 210M sequences of total size~1.3TB as data space for data selection, and we evaluate on the Pile test~set.\footnote{We evaluate on 1\% of the test set (0.1\% with Phi-3), corresponding to $1`812$ sequences.}
We report the \emph{bits per byte} metric as recommended by~\cite{gao2020pile}, which is proportional to the negative log-likelihood loss normalized by a dataset-specific constant.
Error bars correspond to 90\% confidence intervals computed via bootstrapping with $1`000$ samples.\looseness=-1

\paragraph{Base Models and Baselines}
We evaluate the GPT-2 model~\citep{radford2019language} with 124M parameters also evaluated by \cite{hardt2023test}, with the default learning rate of the \texttt{transformers} library~\citep{wolf2020huggingface}.
We obtain analogous results with GPT-2-large~(774M~parameters) and the state-of-the-art Phi-3~\citep[3.8B,][]{abdin2024phi}.\footnote{We detail hyperparameter choices for larger models in \sref{sec:experiment_details}.}
With Phi-3, we use low-rank adaptation~\citep[LoRA,][]{hu2021lora}, fine-tuning slightly less than 1\% of the model's total parameters. %
We compare \method with $\lambda' = 0.01$ to Nearest Neighbor retrieval~(NN) and the failure mode of Nearest Neighbor retrieval that repeatedly selects the closest neighbor.
The failure mode of Nearest Neighbor retrieval~(NN-F) corresponds to an extreme case of redundancy in the data space which we suspect to be a realistic scenario in larger or less curated datasets.
Finally, we compare to Uncertainty Sampling~(US), which is a widely used active learning strategy~\citep{lewis1995sequential,settles2009active} that selects the data with the highest uncertainty in the model's response by selecting according to $\smash{\vx_{n+1} = \argmax_{\vx\in\spD} \sigma_n^2(\vx)}$.
We compare to the heuristic that uses US to choose from the $200$ nearest neighbors, in which case US can be understood as finding a diverse cover of this pre-selected data space~\citep[see, e.g.,][]{holzmuller2023framework,kirsch2019batchbald}.
In contrast, \method \emph{minimizes} the uncertainty in the model's response to the prompt~$\prompt$, leading to a ``denser'' cover close to~$\prompt$ and a ``coarser'' cover further away from~$\prompt$.\looseness=-1

\begin{figure}[]
  \vspace{-10pt}
  \begin{minipage}[b]{0.65\textwidth}
    \centering
    \incplt[\textwidth]{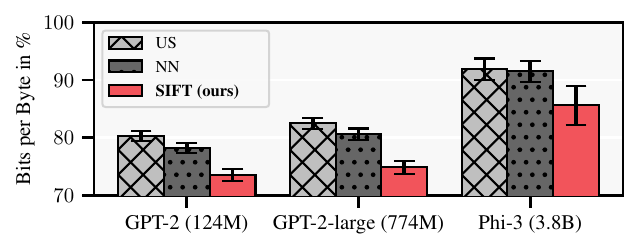}
    \vspace{5pt}
  \end{minipage}
  \hfill
  \begin{minipage}[b]{0.333\textwidth}
    \centering
    \incplt[\textwidth]{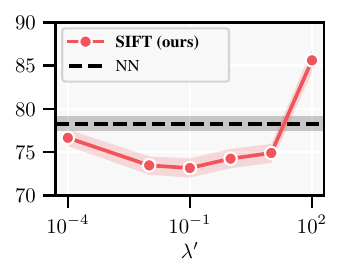}
  \end{minipage}
  \vspace{-15pt}
  \caption{Bits per byte (in \% relative to the base model, $\downarrow$ better) after $50$ test-time iterations.
  \textbf{Left:}~Performance gains of \method are consistent across models. The failure-mode of Nearest Neighbor consistently performs worse than the base model.
  \Cref{table:gpt2large_aft,table:phi3_aft} in \sref{sec:full_results} detail our results with GPT-2-large and Phi-3 analogously to \cref{table:main_results_per_dataset}.
  \textbf{Right:}~Most choices of $\lambda'$ lead to comparable performance. %
  With $\lambda' \to \infty$, \methodl repeatedly selects the nearest neighbor.\looseness=-1
  }
  \label{fig:main_results}
  \vspace{-15pt}
\end{figure}

\begin{wraptable}{r}{0.6\textwidth}
  \vspace{-0.35cm}
  \centering
  \small
  \begin{tabularx}{0.6\textwidth}{@{\hspace{0mm}}S@{\hspace{2mm}}c@{\hspace{2mm}}c@{\hspace{2mm}}c@{\hspace{2mm}}c@{\hspace{2mm}}r}
    \toprule
    & \textbf{US} & \textbf{NN} & \textbf{\textsc{NN-F}} & \textbf{\method} & $\Delta$ \\[1pt]
    \hline \\[-6pt]
    NIH Grants & 93.1\reso{1.1} & 84.9\reso{2.1} & 91.6\reso{16.7} & \textbf{53.8}\reso{8.9} & \downres{31.1} \\[2pt]
    US Patents & 85.6\reso{1.5} & 80.3\reso{1.9} & 108.8\reso{6.6} & \textbf{62.9}\reso{3.5} & \downres{17.4} \\[2pt]
    GitHub & 45.6\reso{2.2} & 42.1\reso{2.0} & 53.2\reso{4.0} & \textbf{30.0}\reso{2.2} & \downres{12.1} \\[2pt]
    Enron Emails & \textbf{68.6}\reso{9.8} & \textbf{64.4}\reso{10.1} & 91.6\reso{20.6} & \textbf{53.1}\reso{11.4} & \downres{11.3} \\[2pt]
    Wikipedia & 67.5\reso{1.9} & \textbf{66.3}\reso{2.0} & 121.2\reso{3.5} & \textbf{62.7}\reso{2.1} & \downres{3.6} \\[2pt]
    Common Crawl & 92.6\reso{0.4} & 90.4\reso{0.5} & 148.8\reso{1.5} & \textbf{87.5}\reso{0.7} & \downres{2.9} \\[2pt]
    PubMed Abstr. & 88.9\reso{0.3} & 87.2\reso{0.4} & 162.6\reso{1.3} & \textbf{84.4}\reso{0.6} & \downres{2.8} \\[2pt]
    ArXiv & 85.4\reso{1.2} & \textbf{85.0}\reso{1.6} & 166.8\reso{6.4} & \textbf{82.5}\reso{1.4} & \downres{2.5} \\[2pt]
    PubMed Central & \textbf{81.7}\reso{2.6} & \textbf{81.7}\reso{2.6} & 155.6\reso{5.1} & \textbf{79.5}\reso{2.6} & \downres{2.2} \\[2pt]
    Stack Exchange & 78.6\reso{0.7} & 78.2\reso{0.7} & 141.9\reso{1.5} & \textbf{76.7}\reso{0.7} & \downres{1.5} \\[2pt]
    Hacker News & \textbf{80.4}\reso{2.5} & \textbf{79.2}\reso{2.8} & 133.1\reso{6.3} & \textbf{78.4}\reso{2.8} & \downres{0.8} \\[2pt]
    FreeLaw & \textbf{63.9}\reso{4.1} & \textbf{64.1}\reso{4.0} & 122.4\reso{7.1} & \textbf{64.0}\reso{4.1} & \upres{0.1} \\[2pt]
    DeepMind Math & \textbf{69.4}\reso{2.1} & \textbf{69.6}\reso{2.1} & 121.8\reso{3.1} & \textbf{69.7}\reso{2.1} & \upres{0.3} \\[2pt]
    \hline \\[-8pt]
    \emph{All} & 80.2\reso{0.5} & 78.3\reso{0.5} & 133.3\reso{1.2} & \textbf{73.5}\reso{0.6} & \downres{4.8} \\
    \bottomrule
  \end{tabularx}
  \caption{Bits per byte (in \% relative to the base model, $\downarrow$) after $50$ test-time iterations on individual datasets of the Pile. We only include datasets with at least $10$ examples in our test set. \textbf{Bold} numbers denote the best performing selected subset. Numbers in parentheses are standard errors. $\Delta$ denotes the performance gain of \method over the strongest baseline.}
  \label{table:main_results_per_dataset}
  \vspace{-10pt}
\end{wraptable}

\paragraph{\insight \method consistently selects better data for fine-tuning than Nearest Neighbor retrieval.\hfill}
We show in \cref{fig:main} that \method outperforms NN and avoids its failure mode where the same data point is selected repeatedly.
In \figref{fig:main_results}{left}, we show that the performance gains of \method are consistent across models.
\Cref{table:main_results_per_dataset} compares the performance of \method against NN across all datasets of the Pile, using GPT-2 as base model.
Overall, we find that \method improves performance both on datasets where NN already performs well, such as GitHub, and on datasets where NN performs poorly, such as NIH Grants.
On all datasets of the Pile, \method performs at least as well as the strongest baseline (within margin of error), suggesting that it is a robust method for data selection.
We observe the trend that relative performance gains of \method over Nearest Neighbor retrieval \emph{increase with model capability}.
That is, with stronger base models, informativeness of selected data appears to become more important.\looseness=-1

\paragraph{\insight \method is robust to the choice of $\lambda'$.}
We evaluate \method with varying choices of $\lambda'$, and summarize the results in \figref{fig:main_results}{right}.
We include extended results in \cref{table:lambda_results} of \sref{sec:additional_results}, showing that for all evaluated $\lambda'$ between $1\mathrm{e-}8$ and $10$, \method performs at least on-par with Nearest Neighbor retrieval on \emph{all} datasets of the Pile, often outperforming it.
This suggests that \method is robust to the choice of $\lambda'$.
Nevertheless, there may be an advantage to adaptively tuning $\lambda'$ (e.g., via cross-validation).
In particular, choosing the best $\lambda'$ for each dataset, \method outperforms all baselines on every dataset.\looseness=-1

\paragraph{\insight \method selects data the ``right'' number of times.}
Nearest Neighbor retrieval implicitly relies on non-redundancy within the data space to not select duplicate information, as illustrated in the example of \cref{fig:qualitative_example}.
This is almost never the case in practice, and in the extreme case of duplicate data, Nearest Neighbor selects the same data point repeatedly.
\method does not rely on excluding previously selected data points.
Instead, \method may select the same data point any number of times, adaptively taking more than one gradient step on it, if beneficial.
To ensure that the selected data is maximally informative, \method takes into account the redundancy of data points explicitly.
This makes \method robust to information duplication by design.%

\begin{wrapfigure}{r}{0.4\textwidth}
  \vspace{-35pt}
  \incplt[0.4\textwidth]{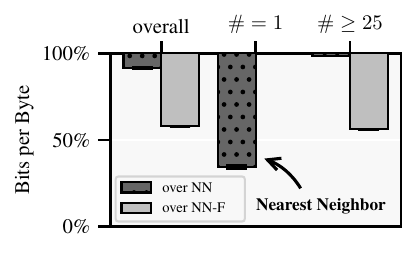}
  \vspace{-15pt}
  \caption{Bits per byte (in \% relative to NN / NN-F, $\downarrow$ better) after $50$ test-time iterations. Error bars correspond to standard errors. The left bars measure the performance gain over all of the Pile. The middle and right bars measure the performance gain for all prompts where \method selects $\#$ unique points.}
  \label{fig:advantage_over_nn}
  \vspace{-15pt}
\end{wrapfigure}
We illustrate this in \cref{fig:advantage_over_nn} where we evaluate the performance gain of \method over Nearest Neighbor and its failure mode.
As expected, we find that on all test prompts where \method selects many unique points, \method outperforms repeatedly selecting the closest neighbor by a large margin.
Interestingly, we also find that on all test prompts where \method selects only a single point, \method outperforms Nearest Neighbor by a large margin.
This suggests that in some cases repeatedly taking gradient steps on the closest neighbor is beneficial, and \method identifies these cases.\looseness=-1

\paragraph{\insight Test-time fine-tuning can significantly improve language modeling ability.}
Our results from \cref{fig:absolute_perf_larger_models} indicate that test-time fine-tuning improves the performance of the base LLM substantially, surprisingly, even with a state-of-the-art model such as Phi-3.
Our Phi-3 with test-time fine-tuning and \method achieves $0.595$~bits per byte, outperforming the previous leader in the \href{https://paperswithcode.com/sota/language-modelling-on-the-pile}{Pile language modeling benchmark}, a $30\times$~larger model.\footnote{We compare to prior work in the Pile language modeling benchmark in \cref{table:pile_benchmark} of \sref{sec:pile_benchmark}.}
We also evaluate the recent Llama-3.2 family of models~\citep{dubey2024llama}, and with Llama-3.2~(3B) as base model we achieve~$0.557$ bits per byte, a significant improvement upon the previous state-of-the-art.
We compare test-time fine-tuning to the common in-context learning, where we include as much of the data as possible into the context window of the test instance, in addition to its original context, by concatenating text in order of selection.
While in-context learning tends to improve the performance of the base model, we find that fine-tuning at test-time tends to outperform or perform on-par with in-context learning. Furthermore, the compute cost of in-context learning grows quadratically with the context window size, meaning that including long texts within large context windows is expensive.
Remarkably, test-time fine-tuning consistently outperforms  in-context learning by more than $25\%$ on math and coding, tasks that require more complex reasoning~(\sref{sec:full_results}).\looseness=-1

\begin{figure}[]
  \vspace{-10pt}
  \centering
  \incplt[\textwidth]{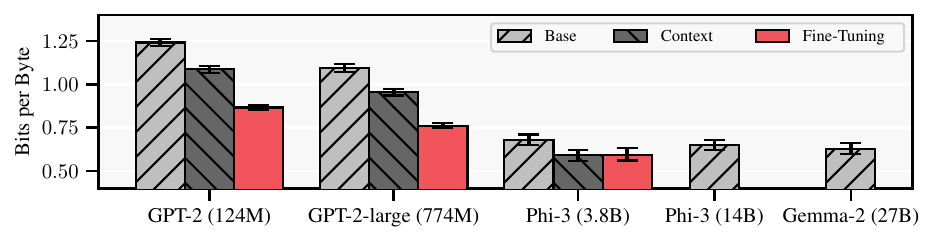}
  \vspace{-15pt}
  \caption{%
  Bits per byte ($\downarrow$ better), comparing fine-tuning and in-context learning with $50$~test-time examples selected by \method.
  We find that fine-tuning systematically outperforms or performs on-par with in-context learning, even when fine-tuning only a LoRA adapter as with Phi-3.
  Test-time fine-tuning with Phi-3 (3.8B) surpasses the performance of the more than $3\times$~larger Phi-3~(14B) and the $7\times$~larger Gemma-2~(27B).
  \looseness=-1}
  \label{fig:absolute_perf_larger_models}
  \vspace{-10pt}
\end{figure}

\paragraph{Further Insights}
In \sref{sec:analysis_active_fine_tuning}, we discuss additional findings on active fine-tuning such as that the performance gains of \method over Nearest Neighbor retrieval \emph{grow with dataset size}, and that normalizing embeddings is important for the effectiveness of data selection.
In \sref{sec:analysis_test_time_fine_tuning}, we discuss additional findings on test-time fine-tuning, for example, the trend that \emph{larger models learn faster at test-time}.\looseness=-1

\section{Compute-Proportional Test-Time Fine-Tuning}\label{sec:compute_proportional}

We have shown that test-time fine-tuning can improve language modeling ability and that \method is a robust method for data selection, outperforming Nearest Neighbor retrieval.
However, a key shortcoming of previous approaches to test-time fine-tuning is that they spend a fixed amount of test-time compute, regardless of the nature of the prompt, the available data, or the model.
This is not computationally scalable in many practical applications, since a fixed test-time compute budget leads to non-proportionate performance gains.
For example, for the prompt ``Hello'' to a chatbot we would not like to spend any test-time compute, while for a more complex prompt we would like to spend more compute.
In this section, we evaluate whether uncertainty estimates can be used to adaptively stop test-time fine-tuning such that the realized performance gain is proportional to the compute used.\looseness=-1

\begin{figure}
  \vspace{-5pt}
  \incplt[\textwidth]{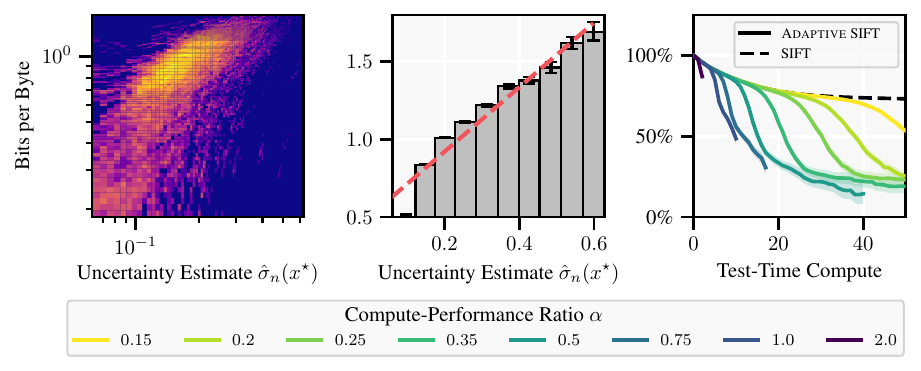}
  \vspace{-15pt}
  \caption{\textbf{Left:}~We visualize the empirical density of the uncertainty estimates $\hat{\sigma}_n$ wrt.\ the bits per byte $\mathrm{bpb}_n$.
  Brighter colors indicate higher density on a logarithmic scale.
  We observe a strong linear relationship between uncertainty estimates and bits per byte.
  \textbf{Middle:}~We construct a ``reliability diagram'' of  uncertainty estimates.
  Notably, since we evaluate with respect to bits per byte rather than an accuracy, canonical calibration plots are not applicable.
  In particular, it is well known that bits per byte do not go to zero for perfect models due to irreducible \emph{aleatoric} uncertainty, which is not captured by our \emph{epistemic} uncertainty estimates.
  Nevertheless, we observe that our epistemic uncertainty estimates are predictive of the model's performance.
  The red line indicates a linear fit.
  \textbf{Right:}~We visualize the bits per byte (in \% relative to the base model, $\downarrow$ better) of all prompts whose model is fine-tuned at a given iteration.
  We find that by adaptively stopping with respect to the known uncertainties $\sigma_n$, we can spend test-time compute proportional to realized performance gains~(see also~\cref{fig:early_stopping_details} in~\sref{sec:additional_results}).
  \emph{Remarks:}~Results are with GPT-2.
  In the left and middle plots, we remove the lowest and highest 0.25\% of uncertainty estimates (i.e., the outliers) for better visualization.
  In the left plot, we additionally remove the lowest and highest 0.25\% of bits per byte.
  }
  \label{fig:uncertainty_estimation}
\end{figure}

\paragraph{\insight The response uncertainty can \emph{predict} performance gain.}
We find that $\sigma_n(\prompt)$ is monotonically and linearly correlated at coefficient~$\approx 0.4$ with the model error after $n$ test-time iterations, i.e., the bits per byte~$\mathrm{bpb}_n(\prompt)$.
This is remarkable because $\sigma_n$ contains information only from the surrogate embedding model, and is normalized such that $\sigma_0(\prompt) = 1$.
To determine the importance of the base model, we also evaluate the denormalized uncertainty estimate $\hat{\sigma}_n(\prompt) \defeq \sigma_n(\prompt) \cdot \mathrm{bpb}_0(\prompt)$, which unlike $\sigma_n$ cannot be evaluated at test-time.
We multiply $\sigma_n$ by $\mathrm{bpb}_0$ to ensure that the uncertainty measure is in the same units as the performance metric, correcting for the use of normalized surrogate embeddings.
We find that $\hat{\sigma}_n(\prompt)$ is strongly correlated at coefficient~$\gtrapprox 0.5$ with the bits per byte. %
We summarize correlations in \cref{table:uncertainty_estimation} of \sref{sec:additional_results} and visualize the predictive capability of~$\hat{\sigma}_n$ in \figref{fig:uncertainty_estimation}{left} and \figref{fig:uncertainty_estimation}{middle}.
Our findings indicate that approximations of the base model's uncertainty, before test-time fine-tuning, can be beneficial.
In future work, we intend to determine whether generating embeddings from the base model can provide such scale-correction.\looseness=-1

Recall that \method minimizes the response uncertainty $\sigma_n$ to the given prompt.
The predictive ability of uncertainty estimates provides an intuitive explanation for the effectiveness of \method.\looseness=-1

\paragraph{\emph{Compute-Proportional Performance Gains:} Early stopping at the ``right'' time.}
Motivated by the predictive power of uncertainty estimates, we evaluate whether they can be used to \emph{adaptively stop} test-time fine-tuning such that the realized performance gain is proportional to the compute used.
In the following, we propose a such a stopping criterion for \method.
Using the approximation of the error via uncertainty estimates discussed above and that $\sigma_0(\prompt) = 1$: \begin{align}
  \text{performance gain} = \frac{\mathrm{bpb}_0(\prompt)}{\mathrm{bpb}_n(\prompt)} \approx \frac{\sigma_0(\prompt)}{\sigma_n(\prompt)} = \frac{1}{\sigma_n(\prompt)}.
\end{align}
We would like to stop fine-tuning when further test-time compute does not yield proportional performance gain, i.e., when ``${\text{performance gain} < \alpha \cdot n}$'' with $n$ approximating the compute of $n$ iterations and $\alpha$ a constant comparing the units of compute and performance.
Plugging in our above approximation of the performance gain, we propose to stop test-time fine-tuning \emph{before} iteration $n$ if \begin{align}
  \sigma_n(\prompt) > (\alpha n)^{-1}. \tag{\adamethod}
\end{align}
Intuitively, this stops fine-tuning the LLM when its progress in crafting a better response stalls.
For complex prompts that benefit from fine-tuning, \adamethod spends more test-time compute, whereas for prompts where the model is already strong or where the data space is not informative, \adamethod spends less test-time compute.
\figref{fig:uncertainty_estimation}{right} shows that the performance gains of this approach are proportional to the compute used.

\paragraph{Towards Scaling Laws of Test-Time Fine-Tuning}
Interestingly, our results bear resemblance to scaling laws of LLM pre-training~\citep{kaplan2020scaling,henighan2020scaling,hoffmann2022training}.
These scaling laws express the performance of a model as a function of the compute used for pre-training (e.g., the number of parameters or training tokens).
Such scaling laws are crucial for determining how to optimally spend a fixed amount of compute.
Recently, scaling laws for ``test-time inference'' have gained attention, where test-time compute is usually spent on search~(e.g., beam search) with a variable number of forward passes of a few-shot prompted base LLM~\citep{brown2024large,snell2024scaling}.
Our results suggest that similar scaling laws exist for test-time fine-tuning, expressing the performance of a model as a function of the compute used for fine-tuning at test-time.
Such scaling laws can be an important tool to determine how to spend test-time compute.
There are many open questions in this direction, which we do not address in this work.
For example, how does model size affect the scaling laws of test-time fine-tuning? Or, can a model be fine-tuned at test-time to build reasoning chains?
Based on our results and previous evaluations of fine-tuning and in-context learning~\citep[e.g.,][]{hu2021lora,mosbach2023few,hardt2023test}, we conjecture that test-time fine-tuning may lead to a more efficient use of compute than repeatedly prompting a base LLM.
We believe that these open questions are exciting directions for future work.\looseness=-1

\vspace{-3pt}
\section{Discussion and Future Work}\label{sec:conclusion}
\vspace{-3pt}

We propose a data selection algorithm, \method, unifying ideas from retrieval and active learning.
\method estimates the uncertainty about the response to a given prompt after having been fine-tuned on some data~(\sref{sec:preliminaries}), and then selects the data that minimizes this uncertainty~(\sref{sec:method}).
This addresses the limitations of Nearest Neighbor retrieval~(\sref{sec:problem_setting}).
\method can be seen as a generalization of Nearest Neighbor retrieval from a search method to a learning method, which ensures explicitly that the retrieved data is maximally informative.
We show on the Pile dataset that \method consistently outperforms Nearest Neighbor retrieval in prompt-specific fine-tuning at test-time and that this kind of local learning can be more effective than locally learning from examples in-context~(\sref{sec:results}).
Finally, we observe that our uncertainty estimates can predict the performance gain of test-time fine-tuning, and use this to develop an adaptive algorithm which achieves compute-proportional performance gains~(\sref{sec:compute_proportional}).\looseness=-1

Test-time fine-tuning addresses a fundamental limitation of in-context learning, namely that in-context learning is typically limited to a fixed and finite context window.
In contrast, test-time fine-tuning allows the LLM to dynamically and effectively access a potentially unbounded non-parametric memory.
By improving the effectiveness of test-time fine-tuning, this work opens up several exciting directions for future research.
Test-time fine-tuning may be used to ground the model on a trusted dataset, mitigate biases against under-represented groups in the training data, or to dynamically include private data depending on user privileges.
Particularly interesting would be a broad evaluation on non-perplexity tasks such as code generation or in the life sciences with large-scale medical or protein data.
Unlike few-shot in-context learning which is limited in scope to autoregressive models, test-time fine-tuning and \method may be extended to other model classes such as diffusion models.
Furthermore, \method may be used effectively in other settings that require automatic data selection, such as targeted instruction tuning during post-training of LLMs.
Finally, our results suggest scaling laws for test-time fine-tuning and we outline several exciting open questions~(\sref{sec:compute_proportional}).\looseness=-1

\clearpage

\section*{Contributions}

JH conceived and led the project, being involved in all its components and leading the theory, implementation of the \method algorithm, design of experiments, and writing.
SB set up and ran the first experiments validating the approach, and contributed to running the final ablation studies.
IH ran additional experiments, especially those with larger models, and optimized the code.
AK advised.\looseness=-1

\section*{Acknowledgements}

We would like to thank Armin Lederer, Vignesh Ram Somnath, Bhavya Sukhija, Scott Sussex, and Lenart Treven for feedback on early versions of the paper.
This project was supported in part by the European Research Council~(ERC) under the European Union's Horizon 2020 research and Innovation Program Grant agreement no.~815943, and the Swiss National Science Foundation under NCCR Automation, grant agreement~51NF40~180545.
Ido Hakimi was supported by an ETH AI Center Postdoctoral fellowship.\looseness=-1

\bibliography{sources}

\begin{thebibliography}{123}
\providecommand{\natexlab}[1]{#1}
\providecommand{\url}[1]{\texttt{#1}}
\expandafter\ifx\csname urlstyle\endcsname\relax
  \providecommand{\doi}[1]{doi: #1}\else
  \providecommand{\doi}{doi: \begingroup \urlstyle{rm}\Url}\fi

\bibitem[Abbasi-Yadkori(2013)]{abbasi2013online}
Yasin Abbasi-Yadkori.
\newblock \emph{Online learning for linearly parametrized control problems}.
\newblock PhD thesis, University of Alberta, 2013.

\bibitem[Abdin et~al.(2024)Abdin, Jacobs, Awan, Aneja, Awadallah, Awadalla, Bach, Bahree, Bakhtiari, Behl, et~al.]{abdin2024phi}
Marah Abdin, Sam~Ade Jacobs, Ammar~Ahmad Awan, Jyoti Aneja, Ahmed Awadallah, Hany Awadalla, Nguyen Bach, Amit Bahree, Arash Bakhtiari, Harkirat Behl, et~al.
\newblock Phi-3 technical report: A highly capable language model locally on your phone.
\newblock \emph{arXiv preprint arXiv:2404.14219}, 2024.

\bibitem[Akyürek et~al.(2024)Akyürek, Damani, Qiu, Guo, Kim, and Andreas]{akyurek2024surprising}
Ekin Akyürek, Mehul Damani, Linlu Qiu, Han Guo, Yoon Kim, and Jacob Andreas.
\newblock The surprising effectiveness of test-time training for abstract reasoning.
\newblock \emph{arXiv preprint arXiv:2411.07279}, 2024.

\bibitem[Ali et~al.(2019)Ali, Kolter, and Tibshirani]{ali2019continuous}
Alnur Ali, J~Zico Kolter, and Ryan~J Tibshirani.
\newblock A continuous-time view of early stopping for least squares regression.
\newblock In \emph{AISTATS}, 2019.

\bibitem[Ali et~al.(2020)Ali, Dobriban, and Tibshirani]{ali2020implicit}
Alnur Ali, Edgar Dobriban, and Ryan Tibshirani.
\newblock The implicit regularization of stochastic gradient flow for least squares.
\newblock In \emph{ICML}, 2020.

\bibitem[Amani \& Thrampoulidis(2020)Amani and Thrampoulidis]{amani2021ucb}
Sanae Amani and Christos Thrampoulidis.
\newblock Ucb-based algorithms for multinomial logistic regression bandits.
\newblock \emph{NeurIPS}, 2020.

\bibitem[Arora et~al.(2016)Arora, Li, Liang, Ma, and Risteski]{arora2016latent}
Sanjeev Arora, Yuanzhi Li, Yingyu Liang, Tengyu Ma, and Andrej Risteski.
\newblock A latent variable model approach to pmi-based word embeddings.
\newblock \emph{Transactions of the Association for Computational Linguistics}, 4, 2016.

\bibitem[Ash et~al.(2020)Ash, Zhang, Krishnamurthy, Langford, and Agarwal]{ash2019deep}
Jordan~T Ash, Chicheng Zhang, Akshay Krishnamurthy, John Langford, and Alekh Agarwal.
\newblock Deep batch active learning by diverse, uncertain gradient lower bounds.
\newblock In \emph{ICLR}, 2020.

\bibitem[Atkeson et~al.(1997)Atkeson, Moore, and Schaal]{atkeson1997locally}
Christopher~G Atkeson, Andrew~W Moore, and Stefan Schaal.
\newblock Locally weighted learning.
\newblock \emph{Lazy learning}, 1997.

\bibitem[Aum{\"u}ller et~al.(2020)Aum{\"u}ller, Bernhardsson, and Faithfull]{aumuller2020ann}
Martin Aum{\"u}ller, Erik Bernhardsson, and Alexander Faithfull.
\newblock Ann-benchmarks: A benchmarking tool for approximate nearest neighbor algorithms.
\newblock \emph{Information Systems}, 87, 2020.

\bibitem[Bagatella et~al.(2024)Bagatella, Hübotter, Martius, and Krause]{bagatella2024active}
Marco Bagatella, Jonas Hübotter, Georg Martius, and Andreas Krause.
\newblock Active fine-tuning of generalist policies.
\newblock \emph{arXiv preprint arXiv:2410.05026}, 2024.

\bibitem[Basu et~al.(2023)Basu, Rawat, and Zaheer]{basu2023statistical}
Soumya Basu, Ankit~Singh Rawat, and Manzil Zaheer.
\newblock A statistical perspective on retrieval-based models.
\newblock In \emph{ICML}, 2023.

\bibitem[Bhargava et~al.(2023)Bhargava, Witkowski, Shah, and Thomson]{bhargava2023s}
Aman Bhargava, Cameron Witkowski, Manav Shah, and Matt Thomson.
\newblock What's the magic word? a control theory of llm prompting.
\newblock \emph{arXiv preprint arXiv:2310.04444}, 2023.

\bibitem[Bhattamishra et~al.(2024)Bhattamishra, Patel, Blunsom, and Kanade]{bhattamishra2023understanding}
Satwik Bhattamishra, Arkil Patel, Phil Blunsom, and Varun Kanade.
\newblock Understanding in-context learning in transformers and llms by learning to learn discrete functions.
\newblock In \emph{ICLR}, 2024.

\bibitem[{Bickford Smith} et~al.(2023){Bickford Smith}, Kirsch, Farquhar, Gal, Foster, and Rainforth]{smith2023prediction}
Freddie {Bickford Smith}, Andreas Kirsch, Sebastian Farquhar, Yarin Gal, Adam Foster, and Tom Rainforth.
\newblock Prediction-oriented bayesian active learning.
\newblock In \emph{AISTATS}, 2023.

\bibitem[Bogunovic et~al.(2015)Bogunovic, Scarlett, Krause, and Cevher]{bogunovic2016truncated}
Ilija Bogunovic, Jonathan Scarlett, Andreas Krause, and Volkan Cevher.
\newblock Truncated variance reduction: A unified approach to bayesian optimization and level-set estimation.
\newblock In \emph{NeurIPS}, 2015.

\bibitem[Borgeaud et~al.(2022)Borgeaud, Mensch, Hoffmann, Cai, Rutherford, Millican, Van Den~Driessche, Lespiau, Damoc, Clark, et~al.]{borgeaud2022improving}
Sebastian Borgeaud, Arthur Mensch, Jordan Hoffmann, Trevor Cai, Eliza Rutherford, Katie Millican, George~Bm Van Den~Driessche, Jean-Baptiste Lespiau, Bogdan Damoc, Aidan Clark, et~al.
\newblock Improving language models by retrieving from trillions of tokens.
\newblock In \emph{ICML}, 2022.

\bibitem[Bottou \& Vapnik(1992)Bottou and Vapnik]{bottou1992local}
L{\'e}on Bottou and Vladimir Vapnik.
\newblock Local learning algorithms.
\newblock \emph{Neural computation}, 4\penalty0 (6), 1992.

\bibitem[Brown et~al.(2024)Brown, Juravsky, Ehrlich, Clark, Le, R{\'e}, and Mirhoseini]{brown2024large}
Bradley Brown, Jordan Juravsky, Ryan Ehrlich, Ronald Clark, Quoc~V Le, Christopher R{\'e}, and Azalia Mirhoseini.
\newblock Large language monkeys: Scaling inference compute with repeated sampling.
\newblock \emph{arXiv preprint arXiv:2407.21787}, 2024.

\bibitem[Brown et~al.(2020)Brown, Mann, Ryder, Subbiah, Kaplan, Dhariwal, Neelakantan, Shyam, Sastry, Askell, et~al.]{brown2020language}
Tom~B. Brown, Benjamin Mann, Nick Ryder, Melanie Subbiah, Jared Kaplan, Prafulla Dhariwal, Arvind Neelakantan, Pranav Shyam, Girish Sastry, Amanda Askell, et~al.
\newblock Language models are few-shot learners.
\newblock \emph{arXiv preprint ArXiv:2005.14165}, 2020.

\bibitem[Bubeck et~al.(2023)Bubeck, Chandrasekaran, Eldan, Gehrke, Horvitz, Kamar, Lee, Lee, Li, Lundberg, et~al.]{bubeck2023sparks}
S{\'e}bastien Bubeck, Varun Chandrasekaran, Ronen Eldan, Johannes Gehrke, Eric Horvitz, Ece Kamar, Peter Lee, Yin~Tat Lee, Yuanzhi Li, Scott Lundberg, et~al.
\newblock Sparks of artificial general intelligence: Early experiments with gpt-4.
\newblock \emph{arXiv preprint arXiv:2303.12712}, 2023.

\bibitem[Chaloner \& Verdinelli(1995)Chaloner and Verdinelli]{chaloner1995bayesian}
Kathryn Chaloner and Isabella Verdinelli.
\newblock Bayesian experimental design: A review.
\newblock \emph{Statistical science}, 1995.

\bibitem[Chen et~al.(2016)Chen, Xu, Zhang, and Guestrin]{chen2016training}
Tianqi Chen, Bing Xu, Chiyuan Zhang, and Carlos Guestrin.
\newblock Training deep nets with sublinear memory cost.
\newblock \emph{arXiv preprint arXiv:1604.06174}, 2016.

\bibitem[Chollet(2019)]{chollet2019measure}
Fran{\c{c}}ois Chollet.
\newblock On the measure of intelligence.
\newblock \emph{arXiv preprint arXiv:1911.01547}, 2019.

\bibitem[Chowdhury \& Gopalan(2017)Chowdhury and Gopalan]{chowdhury2017kernelized}
Sayak~Ray Chowdhury and Aditya Gopalan.
\newblock On kernelized multi-armed bandits.
\newblock In \emph{ICML}, 2017.

\bibitem[Cleveland(1979)]{cleveland1979robust}
William~S Cleveland.
\newblock Robust locally weighted regression and smoothing scatterplots.
\newblock \emph{Journal of the American statistical association}, 74\penalty0 (368), 1979.

\bibitem[Cleveland \& Devlin(1988)Cleveland and Devlin]{cleveland1988locally}
William~S Cleveland and Susan~J Devlin.
\newblock Locally weighted regression: an approach to regression analysis by local fitting.
\newblock \emph{Journal of the American statistical association}, 83\penalty0 (403), 1988.

\bibitem[Cole \& Osman(2023)Cole and Osman]{cole2023dataset}
Jack Cole and Mohamed Osman.
\newblock Dataset-induced meta-learning (and other tricks): Improving model efficiency on arc.
\newblock \url{https://lab42.global/community-model-efficiency/}, 2023.
\newblock [Accessed 22-08-2024].

\bibitem[Cook(1977)]{cook1977detection}
R~Dennis Cook.
\newblock Detection of influential observation in linear regression.
\newblock \emph{Technometrics}, 19\penalty0 (1), 1977.

\bibitem[Cover \& Hart(1967)Cover and Hart]{cover1967nearest}
Thomas Cover and Peter Hart.
\newblock Nearest neighbor pattern classification.
\newblock \emph{IEEE transactions on information theory}, 13\penalty0 (1), 1967.

\bibitem[Cover(1999)]{cover1999elements}
Thomas~M Cover.
\newblock \emph{Elements of information theory}.
\newblock John Wiley \& Sons, 1999.

\bibitem[Devlin et~al.(2018)Devlin, Chang, Lee, and Toutanova]{devlin2018bert}
Jacob Devlin, Ming-Wei Chang, Kenton Lee, and Kristina Toutanova.
\newblock Bert: Pre-training of deep bidirectional transformers for language understanding.
\newblock In \emph{NAACL}, 2018.

\bibitem[Douze et~al.(2024)Douze, Guzhva, Deng, Johnson, Szilvasy, Mazaré, Lomeli, Hosseini, and Jégou]{douze2024faiss}
Matthijs Douze, Alexandr Guzhva, Chengqi Deng, Jeff Johnson, Gergely Szilvasy, Pierre-Emmanuel Mazaré, Maria Lomeli, Lucas Hosseini, and Hervé Jégou.
\newblock The faiss library.
\newblock \emph{arXiv preprint arXiv:2401.08281}, 2024.

\bibitem[Dubey et~al.(2024)Dubey, Jauhri, Pandey, Kadian, Al-Dahle, Letman, Mathur, Schelten, Yang, Fan, et~al.]{dubey2024llama}
Abhimanyu Dubey, Abhinav Jauhri, Abhinav Pandey, Abhishek Kadian, Ahmad Al-Dahle, Aiesha Letman, Akhil Mathur, Alan Schelten, Amy Yang, Angela Fan, et~al.
\newblock The llama 3 herd of models.
\newblock \emph{arXiv preprint arXiv:2407.21783}, 2024.

\bibitem[Elhage et~al.(2022)Elhage, Hume, Olsson, Schiefer, Henighan, Kravec, Hatfield-Dodds, Lasenby, Drain, Chen, et~al.]{elhage2022toy}
Nelson Elhage, Tristan Hume, Catherine Olsson, Nicholas Schiefer, Tom Henighan, Shauna Kravec, Zac Hatfield-Dodds, Robert Lasenby, Dawn Drain, Carol Chen, et~al.
\newblock Toy models of superposition.
\newblock \emph{arXiv preprint arXiv:2209.10652}, 2022.

\bibitem[Faury et~al.(2020)Faury, Abeille, Calauz{\`e}nes, and Fercoq]{faury2020improved}
Louis Faury, Marc Abeille, Cl{\'e}ment Calauz{\`e}nes, and Olivier Fercoq.
\newblock Improved optimistic algorithms for logistic bandits.
\newblock In \emph{ICML}, 2020.

\bibitem[Finn et~al.(2017)Finn, Abbeel, and Levine]{finn2017model}
Chelsea Finn, Pieter Abbeel, and Sergey Levine.
\newblock Model-agnostic meta-learning for fast adaptation of deep networks.
\newblock In \emph{ICML}, 2017.

\bibitem[Fix \& {Hodges Jr.}(1951)Fix and {Hodges Jr.}]{fix1951discriminatory}
Evelyn Fix and Joseph~Lawson {Hodges Jr.}
\newblock \emph{Discriminatory analysis: nonparametric discrimination, consistency properties}, volume~1.
\newblock USAF school of Aviation Medicine, 1951.

\bibitem[Gandelsman et~al.(2021)Gandelsman, Sun, Chen, and Efros]{gandelsman2022test}
Yossi Gandelsman, Yu~Sun, Xinlei Chen, and Alexei Efros.
\newblock Test-time training with masked autoencoders.
\newblock In \emph{NeurIPS}, 2021.

\bibitem[Gao et~al.(2020)Gao, Biderman, Black, Golding, Hoppe, Foster, Phang, He, Thite, Nabeshima, et~al.]{gao2020pile}
Leo Gao, Stella Biderman, Sid Black, Laurence Golding, Travis Hoppe, Charles Foster, Jason Phang, Horace He, Anish Thite, Noa Nabeshima, et~al.
\newblock The pile: An 800gb dataset of diverse text for language modeling.
\newblock \emph{arXiv preprint arXiv:2101.00027}, 2020.

\bibitem[Gao et~al.(2024)Gao, Tow, Abbasi, Biderman, Black, DiPofi, Foster, Golding, Hsu, Le~Noac'h, et~al.]{gao2024framework}
Leo Gao, Jonathan Tow, Baber Abbasi, Stella Biderman, Sid Black, Anthony DiPofi, Charles Foster, Laurence Golding, Jeffrey Hsu, Alain Le~Noac'h, et~al.
\newblock A framework for few-shot language model evaluation, 2024.

\bibitem[Geirhos et~al.(2024)Geirhos, Jaini, Stone, Medapati, Yi, Toderici, Ogale, and Shlens]{geirhos2024towards}
Robert Geirhos, Priyank Jaini, Austin Stone, Sourabh Medapati, Xi~Yi, George Toderici, Abhijit Ogale, and Jonathon Shlens.
\newblock Towards flexible perception with visual memory.
\newblock \emph{arXiv preprint arXiv:2408.08172}, 2024.

\bibitem[Guu et~al.(2020)Guu, Lee, Tung, Pasupat, and Chang]{guu2020retrieval}
Kelvin Guu, Kenton Lee, Zora Tung, Panupong Pasupat, and Mingwei Chang.
\newblock Retrieval augmented language model pre-training.
\newblock In \emph{ICML}, 2020.

\bibitem[Guu et~al.(2023)Guu, Webson, Pavlick, Dixon, Tenney, and Bolukbasi]{guu2023simfluence}
Kelvin Guu, Albert Webson, Ellie Pavlick, Lucas Dixon, Ian Tenney, and Tolga Bolukbasi.
\newblock Simfluence: Modeling the influence of individual training examples by simulating training runs.
\newblock \emph{arXiv preprint arXiv:2303.08114}, 2023.

\bibitem[Hardt \& Sun(2024)Hardt and Sun]{hardt2023test}
Moritz Hardt and Yu~Sun.
\newblock Test-time training on nearest neighbors for large language models.
\newblock In \emph{ICLR}, 2024.

\bibitem[Henighan et~al.(2020)Henighan, Kaplan, Katz, Chen, Hesse, Jackson, Jun, Brown, Dhariwal, Gray, et~al.]{henighan2020scaling}
Tom Henighan, Jared Kaplan, Mor Katz, Mark Chen, Christopher Hesse, Jacob Jackson, Heewoo Jun, Tom~B Brown, Prafulla Dhariwal, Scott Gray, et~al.
\newblock Scaling laws for autoregressive generative modeling.
\newblock \emph{arXiv preprint arXiv:2010.14701}, 2020.

\bibitem[Hoffmann et~al.(2022)Hoffmann, Borgeaud, Mensch, Buchatskaya, Cai, Rutherford, Casas, Hendricks, Welbl, Clark, et~al.]{hoffmann2022training}
Jordan Hoffmann, Sebastian Borgeaud, Arthur Mensch, Elena Buchatskaya, Trevor Cai, Eliza Rutherford, Diego de~Las Casas, Lisa~Anne Hendricks, Johannes Welbl, Aidan Clark, et~al.
\newblock Training compute-optimal large language models.
\newblock \emph{arXiv preprint arXiv:2203.15556}, 2022.

\bibitem[Holzm{\"u}ller et~al.(2023)Holzm{\"u}ller, Zaverkin, K{\"a}stner, and Steinwart]{holzmuller2023framework}
David Holzm{\"u}ller, Viktor Zaverkin, Johannes K{\"a}stner, and Ingo Steinwart.
\newblock A framework and benchmark for deep batch active learning for regression.
\newblock \emph{JMLR}, 2023.

\bibitem[Hu et~al.(2022)Hu, Shen, Wallis, Allen-Zhu, Li, Wang, Wang, and Chen]{hu2021lora}
Edward~J Hu, Yelong Shen, Phillip Wallis, Zeyuan Allen-Zhu, Yuanzhi Li, Shean Wang, Lu~Wang, and Weizhu Chen.
\newblock Lora: Low-rank adaptation of large language models.
\newblock In \emph{ICLR}, 2022.

\bibitem[H{\"u}botter et~al.(2024)H{\"u}botter, Sukhija, Treven, As, and Krause]{hubotter2024transductive}
Jonas H{\"u}botter, Bhavya Sukhija, Lenart Treven, Yarden As, and Andreas Krause.
\newblock Transductive active learning: Theory and applications.
\newblock In \emph{NeurIPS}, 2024.

\bibitem[Ilyas et~al.(2022)Ilyas, Park, Engstrom, Leclerc, and Madry]{ilyas2022datamodels}
Andrew Ilyas, Sung~Min Park, Logan Engstrom, Guillaume Leclerc, and Aleksander Madry.
\newblock Datamodels: Predicting predictions from training data.
\newblock In \emph{ICML}, 2022.

\bibitem[Jacot et~al.(2017)Jacot, Gabriel, and Hongler]{jacot2018neural}
Arthur Jacot, Franck Gabriel, and Cl{\'e}ment Hongler.
\newblock Neural tangent kernel: Convergence and generalization in neural networks.
\newblock \emph{NeurIPS}, 2017.

\bibitem[Jain \& Learned-Miller(2011)Jain and Learned-Miller]{jain2011online}
Vidit Jain and Erik Learned-Miller.
\newblock Online domain adaptation of a pre-trained cascade of classifiers.
\newblock In \emph{CVPR}, 2011.

\bibitem[Johnson et~al.(2019)Johnson, Douze, and J{\'e}gou]{johnson2019billion}
Jeff Johnson, Matthijs Douze, and Herv{\'e} J{\'e}gou.
\newblock Billion-scale similarity search with gpus.
\newblock \emph{IEEE Transactions on Big Data}, 7\penalty0 (3), 2019.

\bibitem[Kaplan et~al.(2020)Kaplan, McCandlish, Henighan, Brown, Chess, Child, Gray, Radford, Wu, and Amodei]{kaplan2020scaling}
Jared Kaplan, Sam McCandlish, Tom Henighan, Tom~B Brown, Benjamin Chess, Rewon Child, Scott Gray, Alec Radford, Jeffrey Wu, and Dario Amodei.
\newblock Scaling laws for neural language models.
\newblock \emph{arXiv preprint arXiv:2001.08361}, 2020.

\bibitem[Karpukhin et~al.(2020)Karpukhin, O{\u{g}}uz, Min, Lewis, Wu, Edunov, Chen, and Yih]{karpukhin2020dense}
Vladimir Karpukhin, Barlas O{\u{g}}uz, Sewon Min, Patrick Lewis, Ledell Wu, Sergey Edunov, Danqi Chen, and Wen-tau Yih.
\newblock Dense passage retrieval for open-domain question answering.
\newblock In \emph{EMNLP}, 2020.

\bibitem[Khandelwal et~al.(2020)Khandelwal, Levy, Jurafsky, Zettlemoyer, and Lewis]{khandelwal2019generalization}
Urvashi Khandelwal, Omer Levy, Dan Jurafsky, Luke Zettlemoyer, and Mike Lewis.
\newblock Generalization through memorization: Nearest neighbor language models.
\newblock In \emph{ICLR}, 2020.

\bibitem[Kingma \& Ba(2014)Kingma and Ba]{kingma2014adam}
Diederik~P Kingma and Jimmy~L Ba.
\newblock Adam: A method for stochastic optimization.
\newblock In \emph{ICLR}, 2014.

\bibitem[Kirsch et~al.(2018)Kirsch, Van~Amersfoort, and Gal]{kirsch2019batchbald}
Andreas Kirsch, Joost Van~Amersfoort, and Yarin Gal.
\newblock Batchbald: Efficient and diverse batch acquisition for deep bayesian active learning.
\newblock In \emph{NeurIPS}, 2018.

\bibitem[Koh \& Liang(2017)Koh and Liang]{koh2017understanding}
Pang~Wei Koh and Percy Liang.
\newblock Understanding black-box predictions via influence functions.
\newblock In \emph{ICML}, 2017.

\bibitem[Kolossov et~al.(2024)Kolossov, Montanari, and Tandon]{kolossov2023towards}
Germain Kolossov, Andrea Montanari, and Pulkit Tandon.
\newblock Towards a statistical theory of data selection under weak supervision.
\newblock In \emph{ICLR}, 2024.

\bibitem[Kossen et~al.(2024)Kossen, Gal, and Rainforth]{kossen2024context}
Jannik Kossen, Yarin Gal, and Tom Rainforth.
\newblock In-context learning learns label relationships but is not conventional learning.
\newblock In \emph{ICLR}, 2024.

\bibitem[Kothawade et~al.(2020)Kothawade, Beck, Killamsetty, and Iyer]{kothawade2021similar}
Suraj Kothawade, Nathan Beck, Krishnateja Killamsetty, and Rishabh Iyer.
\newblock Similar: Submodular information measures based active learning in realistic scenarios.
\newblock In \emph{NeurIPS}, 2020.

\bibitem[Kothawade et~al.(2022)Kothawade, Kaushal, Ramakrishnan, Bilmes, and Iyer]{kothawade2022prism}
Suraj Kothawade, Vishal Kaushal, Ganesh Ramakrishnan, Jeff Bilmes, and Rishabh Iyer.
\newblock Prism: A rich class of parameterized submodular information measures for guided data subset selection.
\newblock In \emph{AAAI}, 2022.

\bibitem[Krause et~al.(2018)Krause, Kahembwe, Murray, and Renals]{krause2018dynamic}
Ben Krause, Emmanuel Kahembwe, Iain Murray, and Steve Renals.
\newblock Dynamic evaluation of neural sequence models.
\newblock In \emph{ICML}, 2018.

\bibitem[Krause et~al.(2019)Krause, Kahembwe, Murray, and Renals]{krause2019dynamic}
Ben Krause, Emmanuel Kahembwe, Iain Murray, and Steve Renals.
\newblock Dynamic evaluation of transformer language models.
\newblock \emph{arXiv preprint arXiv:1904.08378}, 2019.

\bibitem[Lee et~al.(2018)Lee, Xiao, Schoenholz, Bahri, Novak, Sohl-Dickstein, and Pennington]{lee2019wide}
Jaehoon Lee, Lechao Xiao, Samuel Schoenholz, Yasaman Bahri, Roman Novak, Jascha Sohl-Dickstein, and Jeffrey Pennington.
\newblock Wide neural networks of any depth evolve as linear models under gradient descent.
\newblock \emph{NeurIPS}, 2018.

\bibitem[Lewis(1995)]{lewis1995sequential}
David~D Lewis.
\newblock A sequential algorithm for training text classifiers: Corrigendum and additional data.
\newblock In \emph{ACM Sigir Forum}, volume~29, 1995.

\bibitem[Lewis et~al.(2019)Lewis, Perez, Piktus, Petroni, Karpukhin, Goyal, K{\"u}ttler, Lewis, Yih, Rockt{\"a}schel, et~al.]{lewis2020retrieval}
Patrick Lewis, Ethan Perez, Aleksandra Piktus, Fabio Petroni, Vladimir Karpukhin, Naman Goyal, Heinrich K{\"u}ttler, Mike Lewis, Wen-tau Yih, Tim Rockt{\"a}schel, et~al.
\newblock Retrieval-augmented generation for knowledge-intensive nlp tasks.
\newblock In \emph{NeurIPS}, 2019.

\bibitem[Li et~al.(2020)Li, Soltanolkotabi, and Oymak]{li2020gradient}
Mingchen Li, Mahdi Soltanolkotabi, and Samet Oymak.
\newblock Gradient descent with early stopping is provably robust to label noise for overparameterized neural networks.
\newblock In \emph{AISTATS}, 2020.

\bibitem[Li et~al.(2018)Li, Zhang, and Zong]{li2016one}
Xiaoqing Li, Jiajun Zhang, and Chengqing Zong.
\newblock One sentence one model for neural machine translation.
\newblock In \emph{LREC}, 2018.

\bibitem[Lieber et~al.(2021)Lieber, Sharir, Lenz, and Shoham]{lieber2021jurassic}
Opher Lieber, Or~Sharir, Barak Lenz, and Yoav Shoham.
\newblock Jurassic-1: Technical details and evaluation.
\newblock Technical report, AI21 Labs, 2021.

\bibitem[Liu(2019)]{liu2019roberta}
Yinhan Liu.
\newblock Roberta: A robustly optimized bert pretraining approach.
\newblock \emph{arXiv preprint arXiv:1907.11692}, 2019.

\bibitem[Luo et~al.(2020)Luo, Huang, Szeliski, Matzen, and Kopf]{luo2020consistent}
Xuan Luo, Jia-Bin Huang, Richard Szeliski, Kevin Matzen, and Johannes Kopf.
\newblock Consistent video depth estimation.
\newblock \emph{ACM Transactions on Graphics (ToG)}, 2020.

\bibitem[MacKay(1992)]{mackay1992information}
David~JC MacKay.
\newblock Information-based objective functions for active data selection.
\newblock \emph{Neural computation}, 4\penalty0 (4), 1992.

\bibitem[Malladi et~al.(2023)Malladi, Wettig, Yu, Chen, and Arora]{malladi2023kernel}
Sadhika Malladi, Alexander Wettig, Dingli Yu, Danqi Chen, and Sanjeev Arora.
\newblock A kernel-based view of language model fine-tuning.
\newblock In \emph{ICML}, 2023.

\bibitem[Mikolov et~al.(2013)Mikolov, Yih, and Zweig]{mikolov2013linguistic}
Tom{\'a}{\v{s}} Mikolov, Wen-tau Yih, and Geoffrey Zweig.
\newblock Linguistic regularities in continuous space word representations.
\newblock In \emph{NAACL}, 2013.

\bibitem[Minoux(1978)]{minoux1978accelerated}
Michel Minoux.
\newblock Accelerated greedy algorithms for maximizing submodular set functions.
\newblock \emph{Optimization Techniques}, 7, 1978.

\bibitem[Morgan \& Bourlard(1989)Morgan and Bourlard]{morgan1989generalization}
Nelson Morgan and Herv{\'e} Bourlard.
\newblock Generalization and parameter estimation in feedforward nets: Some experiments.
\newblock \emph{NeurIPS}, 1989.

\bibitem[Mosbach et~al.(2023)Mosbach, Pimentel, Ravfogel, Klakow, and Elazar]{mosbach2023few}
Marius Mosbach, Tiago Pimentel, Shauli Ravfogel, Dietrich Klakow, and Yanai Elazar.
\newblock Few-shot fine-tuning vs. in-context learning: A fair comparison and evaluation.
\newblock In \emph{ACL}, 2023.

\bibitem[Murphy(2023)]{murphy2023probabilistic}
Kevin~P Murphy.
\newblock \emph{Probabilistic machine learning: Advanced topics}.
\newblock MIT press, 2023.

\bibitem[Nadaraya(1964)]{nadaraya1964estimating}
Elizbar~A Nadaraya.
\newblock On estimating regression.
\newblock \emph{Theory of Probability \& Its Applications}, 9\penalty0 (1), 1964.

\bibitem[Nemhauser et~al.(1978)Nemhauser, Wolsey, and Fisher]{nemhauser1978analysis}
George~L Nemhauser, Laurence~A Wolsey, and Marshall~L Fisher.
\newblock An analysis of approximations for maximizing submodular set functions—i.
\newblock \emph{Mathematical programming}, 14, 1978.

\bibitem[OpenAI(2024)]{openai2024learning}
OpenAI.
\newblock Learning to reason with llms.
\newblock \emph{OpenAI blog}, 2024.

\bibitem[Park et~al.(2024)Park, Choe, and Veitch]{park2024linear}
Kiho Park, Yo~Joong Choe, and Victor Veitch.
\newblock The linear representation hypothesis and the geometry of large language models.
\newblock In \emph{ICML}, 2024.

\bibitem[P{\'a}sztor et~al.(2024)P{\'a}sztor, Kassraie, and Krause]{pasztor2024bandits}
Barna P{\'a}sztor, Parnian Kassraie, and Andreas Krause.
\newblock Bandits with preference feedback: A stackelberg game perspective.
\newblock In \emph{NeurIPS}, 2024.

\bibitem[Ponte \& Croft(1998)Ponte and Croft]{ponte1998language}
Jay~M. Ponte and W.~Bruce Croft.
\newblock A language modeling approach to information retrieval.
\newblock In \emph{SIGIR}. Association for Computing Machinery, 1998.

\bibitem[Pruthi et~al.(2019)Pruthi, Liu, Kale, and Sundararajan]{pruthi2020estimating}
Garima Pruthi, Frederick Liu, Satyen Kale, and Mukund Sundararajan.
\newblock Estimating training data influence by tracing gradient descent.
\newblock In \emph{NeurIPS}, 2019.

\bibitem[Radford et~al.(2019)Radford, Wu, Child, Luan, Amodei, Sutskever, et~al.]{radford2019language}
Alec Radford, Jeffrey Wu, Rewon Child, David Luan, Dario Amodei, Ilya Sutskever, et~al.
\newblock Language models are unsupervised multitask learners.
\newblock \emph{OpenAI blog}, 1\penalty0 (8):\penalty0 9, 2019.

\bibitem[Reimers \& Gurevych(2019)Reimers and Gurevych]{reimers2019sentence}
Nils Reimers and Iryna Gurevych.
\newblock Sentence-bert: Sentence embeddings using siamese bert-networks.
\newblock In \emph{IJCNLP}, 2019.

\bibitem[Robertson et~al.(2009)Robertson, Zaragoza, et~al.]{robertson2009probabilistic}
Stephen Robertson, Hugo Zaragoza, et~al.
\newblock The probabilistic relevance framework: Bm25 and beyond.
\newblock \emph{Foundations and Trends{\textregistered} in Information Retrieval}, 3\penalty0 (4), 2009.

\bibitem[Sener \& Savarese(2017)Sener and Savarese]{sener2017active}
Ozan Sener and Silvio Savarese.
\newblock Active learning for convolutional neural networks: A core-set approach.
\newblock In \emph{ICLR}, 2017.

\bibitem[Seo et~al.(2000)Seo, Wallat, Graepel, and Obermayer]{seo2000gaussian}
Sambu Seo, Marko Wallat, Thore Graepel, and Klaus Obermayer.
\newblock Gaussian process regression: Active data selection and test point rejection.
\newblock In \emph{Mustererkennung}. Springer, 2000.

\bibitem[Settles(2009)]{settles2009active}
Burr Settles.
\newblock Active learning literature survey.
\newblock Technical report, University of Wisconsin-Madison Department of Computer Sciences, 2009.

\bibitem[Sherman \& Morrison(1950)Sherman and Morrison]{sherman1950adjustment}
Jack Sherman and Winifred~J Morrison.
\newblock Adjustment of an inverse matrix corresponding to a change in one element of a given matrix.
\newblock \emph{The Annals of Mathematical Statistics}, 21\penalty0 (1), 1950.

\bibitem[Shocher et~al.(2018)Shocher, Cohen, and Irani]{shocher2018zero}
Assaf Shocher, Nadav Cohen, and Michal Irani.
\newblock “zero-shot” super-resolution using deep internal learning.
\newblock In \emph{CVPR}, 2018.

\bibitem[Snell et~al.(2025)Snell, Lee, Xu, and Kumar]{snell2024scaling}
Charlie Snell, Jaehoon Lee, Kelvin Xu, and Aviral Kumar.
\newblock Scaling llm test-time compute optimally can be more effective than scaling model parameters.
\newblock In \emph{ICLR}, 2025.

\bibitem[Sparck~Jones(1972)]{sparck1972statistical}
Karen Sparck~Jones.
\newblock A statistical interpretation of term specificity and its application in retrieval.
\newblock \emph{Journal of documentation}, 28\penalty0 (1), 1972.

\bibitem[Srinivas et~al.(2009)Srinivas, Krause, Kakade, and Seeger]{srinivas2009gaussian}
Niranjan Srinivas, Andreas Krause, Sham~M Kakade, and Matthias Seeger.
\newblock Gaussian process optimization in the bandit setting: No regret and experimental design.
\newblock In \emph{ICML}, 2009.

\bibitem[Sun et~al.(2020)Sun, Wang, Liu, Miller, Efros, and Hardt]{sun2020test}
Yu~Sun, Xiaolong Wang, Zhuang Liu, John Miller, Alexei Efros, and Moritz Hardt.
\newblock Test-time training with self-supervision for generalization under distribution shifts.
\newblock In \emph{ICML}, 2020.

\bibitem[Sun et~al.(2024)Sun, Li, Dalal, Xu, Vikram, Zhang, Dubois, Chen, Wang, Koyejo, et~al.]{sun2024learning}
Yu~Sun, Xinhao Li, Karan Dalal, Jiarui Xu, Arjun Vikram, Genghan Zhang, Yann Dubois, Xinlei Chen, Xiaolong Wang, Sanmi Koyejo, et~al.
\newblock Learning to (learn at test time): Rnns with expressive hidden states.
\newblock \emph{arXiv preprint arXiv:2407.04620}, 2024.

\bibitem[Team et~al.(2024)Team, Riviere, Pathak, Sessa, Hardin, Bhupatiraju, Hussenot, Mesnard, Shahriari, Ram{\'e}, et~al.]{team2024gemma}
Gemma Team, Morgane Riviere, Shreya Pathak, Pier~Giuseppe Sessa, Cassidy Hardin, Surya Bhupatiraju, L{\'e}onard Hussenot, Thomas Mesnard, Bobak Shahriari, Alexandre Ram{\'e}, et~al.
\newblock Gemma 2: Improving open language models at a practical size.
\newblock \emph{arXiv preprint arXiv:2408.00118}, 2024.

\bibitem[Templeton et~al.(2024)Templeton, Conerly, Marcus, Lindsey, Bricken, Chen, Pearce, Citro, Ameisen, Jones, et~al.]{templeton2024scaling}
Adly Templeton, Tom Conerly, Jonathan Marcus, Jack Lindsey, Trenton Bricken, Brian Chen, Adam Pearce, Craig Citro, Emmanuel Ameisen, Andy Jones, et~al.
\newblock Scaling monosemanticity: Extracting interpretable features from claude 3 sonnet.
\newblock \emph{Transformer Circuits Thread, Anthropic}, 2024.

\bibitem[Vapnik(2013)]{vapnik2013nature}
Vladimir Vapnik.
\newblock \emph{The nature of statistical learning theory}.
\newblock Springer science \& business media, 2013.

\bibitem[Vaswani et~al.(2017)Vaswani, Shazeer, Parmar, Uszkoreit, Jones, Gomez, Kaiser, and Polosukhin]{vaswani2017attention}
Ashish Vaswani, Noam Shazeer, Niki Parmar, Jakob Uszkoreit, Llion Jones, Aidan~N. Gomez, Lukasz Kaiser, and Illia Polosukhin.
\newblock Attention is all you need.
\newblock In \emph{NeurIPS}, 2017.

\bibitem[Von~Oswald et~al.(2023)Von~Oswald, Niklasson, Randazzo, Sacramento, Mordvintsev, Zhmoginov, and Vladymyrov]{von2023transformers}
Johannes Von~Oswald, Eyvind Niklasson, Ettore Randazzo, Jo{\~a}o Sacramento, Alexander Mordvintsev, Andrey Zhmoginov, and Max Vladymyrov.
\newblock Transformers learn in-context by gradient descent.
\newblock In \emph{ICML}, 2023.

\bibitem[Wang et~al.(2021{\natexlab{a}})Wang, Sun, and Grosse]{wang2021beyond}
Chaoqi Wang, Shengyang Sun, and Roger Grosse.
\newblock Beyond marginal uncertainty: How accurately can bayesian regression models estimate posterior predictive correlations?
\newblock In \emph{AISTATS}, 2021{\natexlab{a}}.

\bibitem[Wang et~al.(2021{\natexlab{b}})Wang, Shelhamer, Liu, Olshausen, and Darrell]{wang2020tent}
Dequan Wang, Evan Shelhamer, Shaoteng Liu, Bruno Olshausen, and Trevor Darrell.
\newblock Tent: Fully test-time adaptation by entropy minimization.
\newblock In \emph{ICLR}, 2021{\natexlab{b}}.

\bibitem[Wang et~al.(2023)Wang, Zhu, Saxon, Steyvers, and Wang]{wang2024large}
Xinyi Wang, Wanrong Zhu, Michael Saxon, Mark Steyvers, and William~Yang Wang.
\newblock Large language models are latent variable models: Explaining and finding good demonstrations for in-context learning.
\newblock In \emph{NeurIPS}, 2023.

\bibitem[Watson(1964)]{watson1964smooth}
Geoffrey~S Watson.
\newblock Smooth regression analysis.
\newblock \emph{Sankhy{\=a}: The Indian Journal of Statistics, Series A}, 1964.

\bibitem[Wei et~al.(2022{\natexlab{a}})Wei, Hu, and Steinhardt]{wei2022more}
Alexander Wei, Wei Hu, and Jacob Steinhardt.
\newblock More than a toy: Random matrix models predict how real-world neural representations generalize.
\newblock In \emph{ICML}, 2022{\natexlab{a}}.

\bibitem[Wei et~al.(2022{\natexlab{b}})Wei, Wang, Schuurmans, Bosma, Xia, Chi, Le, Zhou, et~al.]{wei2022chain}
Jason Wei, Xuezhi Wang, Dale Schuurmans, Maarten Bosma, Fei Xia, Ed~Chi, Quoc~V Le, Denny Zhou, et~al.
\newblock Chain-of-thought prompting elicits reasoning in large language models.
\newblock In \emph{NeurIPS}, 2022{\natexlab{b}}.

\bibitem[Williams \& Rasmussen(2006)Williams and Rasmussen]{williams2006gaussian}
Christopher~KI Williams and Carl~Edward Rasmussen.
\newblock \emph{Gaussian processes for machine learning}, volume~2.
\newblock MIT press, 2006.

\bibitem[Wolf et~al.(2020)Wolf, Debut, Sanh, Chaumond, Delangue, Moi, Cistac, Rault, Louf, Funtowicz, et~al.]{wolf2020huggingface}
Thomas Wolf, Lysandre Debut, Victor Sanh, Julien Chaumond, Clement Delangue, Anthony Moi, Pierric Cistac, Tim Rault, Rémi Louf, Morgan Funtowicz, et~al.
\newblock Huggingface's transformers: State-of-the-art natural language processing.
\newblock \emph{arXiv preprint arXiv:1910.03771}, 2020.

\bibitem[Wynn(1970)]{wynn1970sequential}
Henry~P Wynn.
\newblock The sequential generation of $ d $-optimum experimental designs.
\newblock \emph{The Annals of Mathematical Statistics}, 1970.

\bibitem[Xia et~al.(2024)Xia, Malladi, Gururangan, Arora, and Chen]{xia2024less}
Mengzhou Xia, Sadhika Malladi, Suchin Gururangan, Sanjeev Arora, and Danqi Chen.
\newblock Less: Selecting influential data for targeted instruction tuning.
\newblock In \emph{ICML}, 2024.

\bibitem[Xu \& Kazantsev(2019)Xu and Kazantsev]{xu2019understanding}
Minjie Xu and Gary Kazantsev.
\newblock Understanding goal-oriented active learning via influence functions.
\newblock In \emph{NeurIPS Workshop on Machine Learning with Guarantees}, 2019.

\bibitem[Yao et~al.(2007)Yao, Rosasco, and Caponnetto]{yao2007early}
Yuan Yao, Lorenzo Rosasco, and Andrea Caponnetto.
\newblock On early stopping in gradient descent learning.
\newblock \emph{Constructive Approximation}, 26\penalty0 (2), 2007.

\bibitem[Ye et~al.(2023)Ye, Wu, Feng, Yu, and Kong]{ye2023compositional}
Jiacheng Ye, Zhiyong Wu, Jiangtao Feng, Tao Yu, and Lingpeng Kong.
\newblock Compositional exemplars for in-context learning.
\newblock In \emph{ICML}, 2023.

\bibitem[Yehuda et~al.(2021)Yehuda, Dekel, Hacohen, and Weinshall]{yehuda2022active}
Ofer Yehuda, Avihu Dekel, Guy Hacohen, and Daphna Weinshall.
\newblock Active learning through a covering lens.
\newblock In \emph{NeurIPS}, 2021.

\bibitem[Yu et~al.(2006)Yu, Bi, and Tresp]{yu2006active}
Kai Yu, Jinbo Bi, and Volker Tresp.
\newblock Active learning via transductive experimental design.
\newblock In \emph{ICML}, 2006.

\bibitem[Zeng et~al.(2022)Zeng, Liu, Du, Wang, Lai, Ding, Yang, Xu, Zheng, Xia, et~al.]{zeng2022glm}
Aohan Zeng, Xiao Liu, Zhengxiao Du, Zihan Wang, Hanyu Lai, Ming Ding, Zhuoyi Yang, Yifan Xu, Wendi Zheng, Xiao Xia, et~al.
\newblock Glm-130b: An open bilingual pre-trained model.
\newblock \emph{arXiv preprint arXiv:2210.02414}, 2022.

\bibitem[Zhang \& Sugiyama(2023)Zhang and Sugiyama]{zhang2024online}
Yu-Jie Zhang and Masashi Sugiyama.
\newblock Online (multinomial) logistic bandit: Improved regret and constant computation cost.
\newblock \emph{NeurIPS}, 2023.

\end{thebibliography}
\bibliographystyle{sources}

\clearpage\appendix
\section*{\LARGE Appendices}

\section*{Contents}
\startcontents
\printcontents{}{0}[2]{}

\section{Comparison to the State-of-the-Art on the Pile Language Modeling Benchmark}\label{sec:pile_benchmark}

\Cref{table:pile_benchmark} summarizes the state-of-the-art in the \href{https://paperswithcode.com/sota/language-modelling-on-the-pile}{Pile language modeling benchmark}.

\begin{table}[H]
  \centering
  \begin{tabular}{lcc}
    \toprule
    \textbf{Model} & \textbf{Bits per Byte} & Bits per Byte (without Wikipedia) \\
    \midrule
    Jurassic-1 \scriptsize\citep[178B,][]{lieber2021jurassic} & n/a & 0.601* \\
    GLM \scriptsize\citep[130B,][]{zeng2022glm} & n/a & 0.622* \\
    GPT-2 \scriptsize\citep[124M,][]{radford2019language} & 1.241 &  \\
    GPT-2 \scriptsize\citep[774M,][]{radford2019language} & 1.093 &  \\
    Llama-3.2-Instruct \scriptsize\citep[1B,][]{dubey2024llama} & 0.807 &  \\
    Llama-3.2-Instruct \scriptsize\citep[3B,][]{dubey2024llama} & 0.737 &  \\
    Gemma-2 \scriptsize\citep[2B,][]{team2024gemma} & 0.721 &  \\
    Llama-3.2 \scriptsize\citep[1B,][]{dubey2024llama} & 0.697 & 0.684 \\
    Phi-3.5 \scriptsize\citep[3.8B,][]{abdin2024phi} & 0.690 &  \\
    Phi-3 \scriptsize\citep[3.8B,][]{abdin2024phi} & 0.679 & 0.678 \\
    Phi-3 \scriptsize\citep[7B,][]{abdin2024phi} & 0.678 &  \\
    Gemma-2 \scriptsize\citep[9B,][]{team2024gemma} & 0.670 &  \\
    GPT-3 \scriptsize\citep[175B,][]{brown2020language} & 0.666* &  \\ %
    Phi-3.5-MoE \scriptsize\citep[16$\times$3.8B,][]{abdin2024phi} & 0.656 &  \\
    Phi-3 \scriptsize\citep[14B,][]{abdin2024phi} & 0.651 &  \\
    Llama-3.2 \scriptsize\citep[3B,][]{dubey2024llama} & 0.640 & 0.627 \\
    Gemma-2 \scriptsize\citep[27B,][]{team2024gemma} & 0.629 &  \\
    \midrule
    \emph{Test-Time FT with} \method + GPT-2\reso{124M} & 0.862 &  \\
    \emph{Test-Time FT with} \method + GPT-2\reso{774M} & 0.762 &  \\
    \emph{Test-Time FT with} \method + Llama-3.2\reso{1B} & \underline{0.606} & 0.607 \\
    \emph{Test-Time FT with} \method + Phi-3\reso{3.8B} & \underline{0.595} & \underline{0.599} \\
    \emph{Test-Time FT with} \method + Llama-3.2\reso{3B} & \underline{\textbf{0.557}} & \underline{\textbf{0.559}} \\
    \bottomrule
  \end{tabular}
  \caption{Evaluation of state-of-the-art models on the Pile language modeling benchmark, without copyrighted datasets. (*): Results with GPT-3 are from \cite{gao2020pile}; results with Jurassic-1 and GLM are from \cite{zeng2022glm} and do not report on the Wikipedia dataset. For a complete comparison, we also evaluate our Phi-3 and Llama-3.2 with test-time fine-tuning when excluding the Wikipedia dataset.
  \textbf{Bold} numbers denote the best performing model.
  \underline{Underlined} numbers denote a model that is better than the previous state-of-the-art.\looseness=-1}
  \label{table:pile_benchmark}
\end{table}

Due to our dataset being restricted to the non-copyrighted part of the Pile, the data distribution changes slightly.
To account for this, we take the reported results of prior work and exclude the datasets that have copyright restrictions from the evaluation.
Notably, some prior reported results of state-of-the-art models miss evaluation of the Wikipedia dataset, which we therefore also exclude for a direct comparison.
To the best of our knowledge, our results with test-time fine-tuning and \method achieve a new state-of-the-art on the Pile benchmark.\looseness=-1

\section{Extended Related Work}\label{sec:extended_related_work}

\subsection{Learning at Test-Time}

The subject of learning at test-time has a rich history in statistics and machine learning.
By ``learning at test-time'' we refer to models that are constructed specifically for a given test instance, differing from the model used for other test instances.
The following discussion provides a brief overview with emphasis on the most recent developments.\looseness=-1

\paragraph{$k$-Nearest Neighbors and Kernel Regression (since 1950s)} One of the most basic forms of learning at test-time was developed by \cite{fix1951discriminatory} and \cite{cover1967nearest}.
Given the supervised data $\spD \subseteq \spX \times \spY$ with input domain $\spX \subseteq \R^d$ and labels $\spY = \{0, \dots, K\}$, the $k$-NN algorithm predicts the label of a test instance $\prompt \in \spX$ by taking the majority vote of the $k$ nearest neighbors of~$\prompt$ in~$\spD$ according to some distance metric on~$\spX$ such as Euclidean distance.
In the case of regression, $\spY = \R$ and the prediction is the average of the labels of the $k$ nearest neighbors.
Kernel regression extended upon this idea in the 1960s by weighting neighbors according to their distance to the test instance~\citep{nadaraya1964estimating,watson1964smooth}.
This is a simple and often effective method if the inputs are well-structured and low-dimensional, e.g., if $\spX$ is a learned low-dimensional manifold~\citep{geirhos2024towards}.
When $K$ is large, as for example when $\spY$ is the set of all tokens in a language modeling task, naive application of $k$-NNs is difficult, nevertheless they have been shown to be effective when mixed with parametric language models~\citep{khandelwal2019generalization}.\looseness=-1

\paragraph{Local Learning (since 1970s)}
Local learning is the idea of using data ``relevant'' to the test instance~$\prompt$ to train a parametric model.
Formally, given a test instance $\prompt$, conventually a model~$f$ is used to predict $f(\prompt)$ where $f$ is trained to minimize the average loss over the training data.
Instead, local learning trains a model $f_{\prompt}$ specifically for $\prompt$ and predicts $f_{\prompt}(\prompt)$.
Original works train a linear model by weighting data according to their proximity to $\prompt$~\citep{cleveland1979robust,cleveland1988locally,atkeson1997locally}.
Here, each test instance trains a model from scratch since the optimal solution of linear regression is independent of initialization.
This perspective has regained interest recently in the context of neural networks, with \cite{sun2020test} naming it \emph{``test-time training''}.\looseness=-1

\paragraph{Transductive Learning (since 1980s)}

Vladimir Vapnik developed the general principle of \emph{transduction} which he states in \cite{vapnik2013nature} as follows:
\begin{quote}{Vladimir Vapnik:}
  \emph{``When solving a problem of interest, do not solve a more general problem as an intermediate step. Try to get the answer that you really need but not a more general one.''}
\end{quote}
This is perhaps the most general principle behind learning at test-time, and directly opposed to the principle of \emph{induction} --- extracting the most general rules from data --- which has arguably dominated machine learning research over the last decades.
In a way, local learning is pushing the principle of transduction to the opposite extreme: Each test instance defines its own learning problem, with the test instance alone being the target of prediction.\looseness=-1

\paragraph{Local Fine-Tuning (since 1990s)}

\cite{bottou1992local} were the first to use local learning in conjunction with a \emph{pre-trained} parametric model.
They train (i.e., ``fine-tune'') the last layer of a convolutional neural network for handwritten digit classification based on the nearest neighbors to the test instance in pixel space.
Very recently, \cite{hardt2023test} applied the same idea to language models, showing that local fine-tuning can significantly improve the performance of large language models on standard benchmarks.
Previously, this idea has also been evaluated by \cite{li2016one} and \cite{basu2023statistical}.
\emph{``Test-time fine-tuning''} (as well as ``active inference'') has frequently been used to refer to this approach of locally fine-tuning a pre-trained model.
Within the last few years, test-time fine-tuning has regained substantial interest in the context of self-supervised learning, where the pre-trained model is fine-tuned on the \emph{test instance itself}.
Notable applications of this approach are in vision~\citep{jain2011online,shocher2018zero,luo2020consistent,sun2020test,wang2020tent} and in language modeling~\citep{krause2018dynamic,krause2019dynamic}, where it is called \emph{dynamic evaluation}.
As one would also naively expect, test-time fine-tuning yields the largest improvements when the prompt is not (well-) represented in the pre-training data, e.g., due to a distribution shift~\citep{gandelsman2022test,hardt2023test}.
Notably, test-time fine-tuning is the central component of the state-of-the-art approaches to the ARC challenge~\citep{chollet2019measure,cole2023dataset}, a non-saturated benchmark which is intended to test reasoning capabilities based on ``core knowledge'' rather than mere memorization.\looseness=-1

\paragraph{(Few-Shot) In-Context Learning (since 2020s)}

Very recently, with the advent of large language models (LLMs), learning at test-time has regained interest.
\cite{brown2020language} showed that GPT-3 can \emph{learn in-context} from input-label pairs that are appended to the prompt, an emergent phenomenon of LLMs that has been widely studied since~\citep{von2023transformers,kossen2024context,bhattamishra2023understanding}.
In contrast to standard in-weights learning, in-context learning requires no parameter updates.
Interestingly, in-context learning adopts the same paradigm as local learning wherein a model is adapted specifically for the test instance~$\prompt$, here by skewing the autoregressive distribution towards the data included in the prompt.
This is often combined with the automatic sourcing of nearest neighbors to $\prompt$ in an external dataset, which is known as \emph{``retrieval augmented generation''}~\citep[RAG,][]{lewis2020retrieval,borgeaud2022improving}, and is akin to the other methods of test-time learning discussed above.
A crucial difference between test-time fine-tuning and in-context learning appears to be that learning from context works by \emph{changing the test instance}~\citep{bhargava2023s} whereas in-weights learning works by \emph{changing the model}.
With small datasets, in-context learning is therefore often more computationally efficient than test-time fine-tuning, however this ceases to be the case when the dataset grows since the complexity of transformers grows quadratically in the number of context tokens whereas the complexity of test-time fine-tuning grows linearly.\looseness=-1

\subsection{Data Selection}

Clearly, the choice of data to learn from at test-time is crucial for predictive performance.
Selecting uninformative data can increase inference time or even degrade performance~\citep[see, e.g.,][]{kolossov2023towards}.
Today, datasets for fine-tuning are often hand-designed, however, this is not possible in a test-time setting.
Automatic data selection has a rich history in machine learning, studied extensively in \emph{search}, \emph{experimental design}~\citep{chaloner1995bayesian}, and \emph{active learning}~\citep{settles2009active}.
The following attempts to give a brief overview of the most recent developments.\looseness=-1

\paragraph{(Document) Retrieval (since 1970s)}

Retrieval methods aim to search a dataset $\spD$ for the most relevant data to a given query/prompt.
The most classical methods such as TF-IDF~\citep{sparck1972statistical} and BM25~\citep{robertson2009probabilistic} are based on keyword matching, and were developed alongside the first search engines.
Due to their reliance on ``bags of words'', i.e., sets of one-hot-encoded word vectors, they are known as \emph{sparse retrievers}.
An alternative idea is to select the data $\vx$ that maximizes the likelihood of the query $\prompt$ given the data, i.e., $\argmax_{\vx \in \spD} p(\prompt \mid \vx)$, known as \emph{query likelihood retrievers}~\citep{ponte1998language,wang2024large}.
Here, the conditional probability can be a non-parametric term frequency or a parametric language model.
More recently, due to significant advances in representation learning~\citep{devlin2018bert,reimers2019sentence}, dense retrievers have become popular~\citep[e.g.,][]{lewis2020retrieval,karpukhin2020dense,borgeaud2022improving}.
A \emph{dense retriever} embeds dataset and query into a metric vector space, and retrieves the nearest neighbors to the query.
Standard vector-based search methods use cosine similarity or (equivalently\footnote{Here we assume that vectors are normalized to unit length, cf. \cref{sec:proofs:insufficiency_nn}.}) Euclidean distance.
Recent advances in algorithms and implementation mean that (approximate) nearest neighbor retrieval can be performed efficiently with databases of billions or even trillions of tokens~\citep[e.g.,][]{johnson2019billion,aumuller2020ann}.
The most common metric is cosine distance, which coincides with Euclidean distance when vectors are normalized to unit length.
Nearest neighbor retrieval has been the de-facto standard for data selection in RAG and local learning.\footnote{There is substantial literature that investigates selection of ``informative'' data for RAG~\citep[e.g.,][]{ye2023compositional}.}\looseness=-1

\paragraph{Influence Functions (since 1970s)}

Influence functions measure the change in a model's prediction when a single data point is removed from the training data.
First proposed by \cite{cook1977detection} for linear regression, they have since been used extensively to \emph{interpret} predictions~\citep{koh2017understanding,pruthi2020estimating}.
Very recently, \cite{xia2024less} applied influence functions to select data that leads to the largest (approximate) reduction in test-loss.
Concretely, using a first-order Taylor approximation of the loss $\ell$ and if the model at time $t$ is updated via stochastic gradient descent with step size $\eta_t$ on data $\vx$, the loss reduction can be approximated as \begin{align*}
  \ell(\prompt; \vtheta_{t+1}) - \ell(\prompt; \vtheta_t) \approx - \eta_t \langle \grad_{\vtheta} \ell(\vx; \vtheta_t), \grad_{\vtheta} \ell(\prompt; \vtheta_t) \rangle.
\end{align*}
That is, the data $\vx$ whose loss gradient is most aligned with the loss gradient of the test instance~$\prompt$, can be expected to lead to the largest loss reduction.\footnote{\cite{xia2024less} normalize embeddings before computing the inner product (thus, maximizing cosine similarity) to account for varying gradient norms depending on sequence lengths.}
Note that this simply leads to nearest neighbor retrieval in an embedding space informed by the model at time $t$.
A major limitation of using influence functions for data selection is that they implicitly assume that the influence of selected data adds linearly \citep[i.e., two equally scored data points are expected to doubly improve the model performance,][Section 3.2]{xu2019understanding}.
This assumption does quite obviously not hold in practice as seen, e.g., by simply duplicating data.
The same limitation applies to the related approach of \emph{datamodels}~\citep{ilyas2022datamodels}.
A recent line of work aims to address this limitation by designing simulators that can be probed with datasets to estimate their effect on a prediction requiring less compute than training the full model~\citep{guu2023simfluence}, yet, this does not address the data selection problem as the space of possible datasets is exponentially large.\looseness=-1

\paragraph{Coverage \& \emph{Inductive} Active Learning}

Next we discuss an orthogonal line of work, which takes into account the interaction between selected data, but not the interaction of that data with respect to a test instance.
Roughly speaking classical active learning studies how to most effectively select data from a domain $\spX$ for learning a model over this domain $\spX$.
Intuitively, this task can be thought of as selecting a subset $X \subseteq \spX$ of fixed size that captures the most ``information'' about the target function~$f$.
As such, this task is of an \emph{inductive} nature: we aim to extract general rules from the data that can be applied to unseen data later, without concrete specification of the unseen data.
Approaches to (inductive) active learning are broadly aiming to select \emph{diverse} data that covers the data manifold in~$\spX$ well.
Methods include those that maximize the mutual distances between selected data~(e.g., \textsc{CoreSet}~\citep{sener2017active}, \textsc{BADGE}~\citep{ash2019deep}, and \textsc{ProbCover}~\citep{yehuda2022active}) with respect to a latent distance metric and those ``uncertainty sampling'' methods that select data that the model is most uncertain about~(e.g., \emph{D-optimal design}~\citep{wynn1970sequential} and \textsc{BatchBALD}~\citep{kirsch2019batchbald}).\footnote{Section 5.2 of \cite{holzmuller2023framework} provides a comprehensive overview.}
Both families of methods can be seen as determining some decent covering of the data manifold in $\spX$.
In a probabilistic sense, uncertainty sampling can be seen to minimize the ``posterior predictive entropy'' in expectation over the observed data.
\reviewtext{Approaches to inductive active learning have frequently been applied to pre-training models, with image classification as the canonical application~\citep[e.g,][]{holzmuller2023framework}.}\looseness=-1

\subsection{\method Unifies Work on Retrieval and Work on Coverage}\label{sec:method_synthesize_nn_and_al}

\reviewtext{Retrieval and inductive active learning fall on to two extreme ends of a spectrum:
Retrieval methods search for relevant data without ensuring that data is non-redundant.
As such, naive application of search methods is insufficient for a learning task since those generally do not take ``distinctiveness'' into account (cf. \cref{sec:problem_setting:nn_insufficient}).
In contrast, active learning methods select non-redundant data without ensuring that data is relevant.
Like \method, many active learning methods are based on some measure of ``uncertainty'', however how this measure is utilized for data selection differs fundamentally in \method:}\looseness=-1

\paragraph{\emph{Transductive} Active Learning: Unifying retrieval \& coverage}

\reviewtext{Transductive active learning is motivated from the central observation that learning and prediction requires synthesizing information that is both relevant and non-redundant.
Transductive active learning~\citep{hubotter2024transductive} bridges this gap by selecting data that is both relevant and non-redundant.}
In this work, we propose \method, an approach to test-time transductive active learning (i.e., transductive active learning with a single prediction target), which extends previously proposed algorithms~\citep{mackay1992information,seo2000gaussian,yu2006active,hubotter2024transductive}.
Similar algorithmic ideas have recently been evaluated empirically in a variety of other settings~\citep{kothawade2021similar,wang2021beyond,kothawade2022prism,smith2023prediction} such as Bayesian optimization~\citep{hubotter2024transductive}, multi-task reinforcement learning~\citep{bagatella2024active}, and the amortized fine-tuning of neural networks~\citep{hubotter2024transductive}.
\method aims to select data that is both relevant and non-redundant with respect to the already seen data, whereby the hyperparameter $\lambda'$ controls the trade-off between relevance and redundancy.
\cite{hubotter2024transductive} introduce extensions of \method to more than one prediction target, i.e., amortizing learning across multiple prompts.
They show that if the prediction targets include \emph{all of~$\spX$}, then the method reduces to a form of \emph{inductive active learning}.\looseness=-1

\section{Further Details on \method}\label{sec:method_details}

\subsection{How \method Balances Relevance and Diversity}\label{sec:method_details:balancing_relevance_diversity}

Let us look more closely at the points selected by \method.
We will assume here for ease of notation that embeddings have unit length.\footnote{See \cref{sec:proofs_balance} for the expressions with non-normalized embeddings.}
The first point selected by \method has the largest (absolute) cosine similarity to the prompt within the latent space: \begin{align*}
  \vx_1 = \argmin_{\vx \in \spD} \sigma_{\{\vx\}}^2(\prompt) = \argmax_{\vx \in \spD} \frac{(\transpose{\vphi(\prompt)} \vphi(\vx))^2}{1 + \lambda'} = \argmax_{\vx \in \spD} \Big(\hspace{-0.95cm}\underbrace{\measuredangle_{\vphi}(\prompt, \vx)}_{\text{cosine similarity of $\vphi(\prompt), \vphi(\vx)$}}\hspace{-0.95cm}\Big)^{\!\! 2}. \tag{\textbf{1st point}}
\end{align*}
This recovers the standard approach of Nearest Neighbor retrieval with respect to cosine similarity, provided cosine similarities are non-negative.
However, we show next that selecting more than one point, \method not only considers the relevance with respect to the prompt~$\prompt$, but also the redundancy with respect to the already seen data~$\vx_1$.
\begin{align*}
  \vx_2 = \argmin_{\vx \in \spD} \sigma_{\{\vx_1,\vx\}}^2(\prompt) = \argmax_{\vx \in \spD} \transpose{\begin{bmatrix}
    \measuredangle_{\vphi}(\prompt, \vx_1) \\
    \measuredangle_{\vphi}(\prompt, \vx)
  \end{bmatrix}} \!\! \inv{\begin{bmatrix}
    1 + \lambda' & \measuredangle_{\vphi}(\vx_1, \vx) \\
    \measuredangle_{\vphi}(\vx_1, \vx) & 1 + \lambda'
  \end{bmatrix}} \!\! \begin{bmatrix}
    \measuredangle_{\vphi}(\prompt, \vx_1) \\
    \measuredangle_{\vphi}(\prompt, \vx)
  \end{bmatrix}\!\!.
  \tag{\textbf{2nd point}}
\end{align*}
To illustrate how \method balances relevance and diversity, we compare the value of observing $\vx_1$ twice to observing a different $\vx$ with cosine similarity $\measuredangle_{\vphi}(\vx_1,\vx) = 0$.
We show in \cref{sec:proofs_balance} that \methodl prefers $\vx$ over $\vx_1$ for selecting $\vx_2$ \emph{if and only if}
\begin{align*}
  \measuredangle_{\vphi}(\prompt, \vx)^2 > \frac{\lambda'}{2 + \lambda'} \measuredangle_{\vphi}(\prompt, \vx_1)^2
\end{align*}
\begin{wraptable}{r}{0.4\textwidth}
  \vspace{-12pt}
  \centering
  \begin{tabular}{lc@{\hspace{3mm}}c}
    \toprule
    \textbf{Parameter} & \textbf{Relation} & \textbf{Div.} \\[1pt]
    \hline \\[-9pt]
    regularization $\lambda$ & $\lambda$ & $\textcolor{red}{\boldsymbol{\downarrow}}$ \\[2pt]
    step size $\eta$ & $1 / \eta$ & $\textcolor{blue}{\boldsymbol{\uparrow}}$ \\[2pt]
    noise $\rho$ (cf. \sref{sec:method_maximizes_info_gain}) & $\rho^2$ & $\textcolor{red}{\boldsymbol{\downarrow}}$ \\[2pt]
    \bottomrule
  \end{tabular}
  \caption{The effect of~$\lambda$ and its other interpretations on diversity of selected data (as the parameter is increased).\looseness=-1}
  \label{table:lambda}
  \vspace{-1cm}
\end{wraptable}
The hyperparameter $\lambda'$ controls the trade-off between relevance and diversity: if $\lambda' = 1$ then even if $\vx$ has one third the relevance of $\vx_1$, it is still preferred.
As $\lambda' \to \infty$, \methodl performs retrieval by repeatedly selecting the same point; and as $\lambda' \to 0$, \methodl aims only to select the most diverse points.
We observe the same relationship empirically on the Pile dataset~(cf.~\figref{fig:irreducible_uncertainty}{left}).
\Cref{table:lambda} summarizes the effect of the regularization parameter $\lambda$ and its interpretations.\looseness=-1

\subsection{The Uncertainty of \method Provably Vanishes}\label{sec:method_details:uncertainty_vanishes}

We now formally prove that unlike with Nearest Neighbor retrieval, the uncertainty $\sigma_n^2(\prompt)$ about the response to the prompt vanishes if \method is used to select data for fine-tuning.
As discussed in \sref{sec:convergence}, this requires that the data space contains sufficient information to determine the correct response.
In general, there might be an irreducible error remaining.
We will denote a basis of the embeddings $\{\vphi(\vx) : \vx \in \spD\}$ within the data space~$\spD$ by ${\mPhi \in \R^{m \times d}}$ with size $m$ and dimension $d$, and we denote by $\mPi_{\mPhi}$ its orthogonal projection onto the orthogonal complement of the span of~$\mPhi$.
\cite{hubotter2024transductive} show that for all $X \subseteq \spD$,\looseness=-1 \begin{align}
  \sigma_X^2(\prompt) \geq \norm{\vphi(\prompt)}_{\mPi_{\mPhi}}^2
\end{align} where ${\norm{\vv}_{\mA} = \sqrt{\transpose{\vv} \mA \vv}}$ denotes the Mahalanobis distance.
We call $\smash{\irred*{\prompt} \defeq \norm{\vphi(\prompt)}_{\mPi_{\mPhi}}^2}$ the \emph{irreducible uncertainty} about $\prompt$.
It can be seen that $\smash{\irred*{\vx^\parallel} = 0}$ for all $\smash{\vx^\parallel \in \spX}$ with $\smash{\vphi(\vx^\parallel) \in \spn{\mPhi}}$.
That is, the irreducible uncertainty is zero for points in the span of the data space.
In contrast, for points $\smash{\vx^\perp}$ with $\vphi(\vx^\perp) \in (\spn{\mPhi})^{\perp}$, the irreducible uncertainty equals the initial uncertainty: $\irred*{\vx^\perp} = \sigma_0^2(\vx^\perp)$.
The irreducible uncertainty of any prompt $\prompt$ can be computed by simple decomposition of $\vphi(\prompt)$ into parallel and orthogonal components.
Hence, if the data space is large and includes all relevant information to answer the prompt, the irreducible uncertainty is negligible.\looseness=-1

We will denote the \emph{uncertainty reduction} about the prompt~$\prompt$ achieved by fine-tuning on~$X$ by $\smash{\psi_{\prompt}(X) \defeq \sigma_0^2(\prompt) - \sigma_X^2(\prompt)}$ and note that \method selects $\vx_{n+1} = \argmax_{\vx \in \spD} \psi_{\prompt}(X_n \cup \{\vx\})$.
Stating the convergence guarantee of \method requires one straightforward assumption.\looseness=-1

\begin{assumption}\label{assumption:submodular}
  The uncertainty reduction $\psi_{\prompt}(X)$ is submodular.
\end{assumption}\vspace{-6pt}

\begin{figure}
  \incplt[\textwidth]{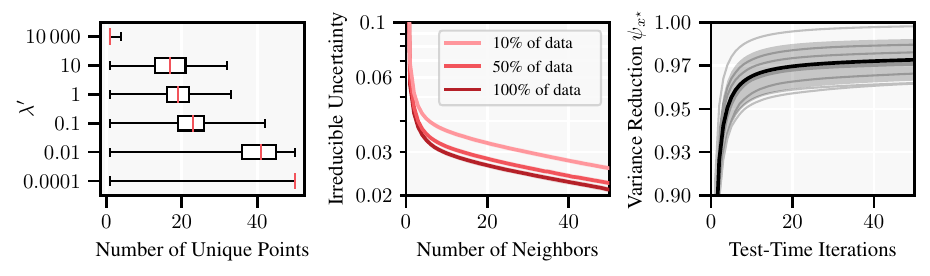}
  \vspace{-15pt}
  \caption{\textbf{Left:} The parameter $\lambda'$ controls the trade-off between relevance and diversity of the selected data. As $\lambda' \to \infty$, \method selects the same point repeatedly whereas as $\lambda' \to 0$, \method selects a diverse set of points.
  \textbf{Middle:} The irreducible uncertainty of test prompts from the Pile given neighbors selected from fractions of the Pile training dataset in the data space. The irreducible uncertainty captures how much information is available, and decays quickly.
  \textbf{Right:} We empirically observe that $\psi_{\prompt}$ is monotone submodular, i.e., its ``marginal gains'' decrease as the number of iterations increases. The shaded region denotes the standard deviation, gray lines are from $10$ randomly selected prompts.\looseness=-1}
  \label{fig:irreducible_uncertainty}
\end{figure}

Intuitively, \cref{assumption:submodular} states that the marginal uncertainty reduction achieved by adding a point to the selected data (i.e., the `marginal gain') decreases as the size of the selected data increases, which is a common assumption in prior work.\footnote{Similar assumptions have been made by \cite{bogunovic2016truncated} and \cite{kothawade2021similar}.}
Formally \cref{assumption:submodular} is satisfied if, for all $\vx \in \spD$ and $X' \subseteq X \subseteq \spD$, \begin{align}
  \Delta_{\prompt}(\vx \mid X') \geq \Delta_{\prompt}(\vx \mid X)
\end{align} where $\Delta_{\prompt}(\vx \mid X) \defeq \psi_{\prompt}(X \cup \{\vx\}) - \psi_{\prompt}(X)$ is the \emph{marginal uncertainty reduction} of $\vx$ given~$X$.

Though theoretically this assumption may be violated by some instances~\citep[Example C.8]{hubotter2024transductive}, we observe that it is satisfied in practice~(cf.~\figref{fig:irreducible_uncertainty}{right}).
Under this assumption, $\psi_{\prompt}(X_n) \geq (1-1/e) \max_{X \subseteq \spD, \abs{X} \leq n}\psi_{\prompt}(X)$ due to the seminal result on monotone submodular function maximization of \cite{nemhauser1978analysis}.
That is, the iterative scheme of \method achieves a constant factor approximation of the optimal uncertainty reduction. Moreover, recent work on transductive active learning of \cite{hubotter2024transductive} which we restate here shows that the uncertainty of \method converges to the irreducible uncertainty.
We assume w.l.o.g.\ that $\smash{\norm{\vphi(\vx)}_2^2 \leq 1}$ for all $\vx \in \spX$.\looseness=-1

\begin{theorem}[Convergence Guarantee, formalization of \cref{informal_thm:convergence}]\label{thm:convergence}
  Let \cref{assumption:submodular} hold and $X_n$ be selected by \methodl from the data space $\spD$.
  Then for all $n \geq 1$ and $\prompt \in \spX$, \begin{align*}
    \sigma_n^2(\prompt) \leq \irred*{\prompt} + \frac{d (1 + 2 d \lambda' \inv{\lambda_{\min}}) \log(1 + \frac{\hat{\lambda}_{n}}{\lambda'})}{\sqrt{n}}
  \end{align*} where $\lambda_{\min}$ is the smallest eigenvalue of $\mPhi \transpose{\mPhi}$ with $\mPhi \in \R^{m \times d}$ a basis of ${\{\vphi(\vx) : \vx \in \spD\}}$, and where $\hat{\lambda}_n \leq \BigO{n}$ is the largest eigenvalue of $\mPhi_n \transpose{\mPhi_n}$.
\end{theorem}
\begin{proof}
  \Cref{thm:convergence} follows from Theorem 3.2 of \cite{hubotter2024transductive} noting that \begin{itemize}
    \item The \method objective is a special case of VTL (Variance-based Transductive Active Learning) with ``target space'' $\spA = \{\prompt\}$.
    \item Theorem 3.2 of \cite{hubotter2024transductive} can be extended to finite-dimensional reproducing kernel Hilbert spaces~\citep[Appendix C.6.4]{hubotter2024transductive}.
    \item The ``maximum information gain of $n$ iterations'', $\gamma_n$, in the statement of \cite{hubotter2024transductive} is bounded as follows~\citep[Appendix C.3]{srinivas2009gaussian}: $\smash{\gamma_n \leq d \log(1 + \hat{\lambda}_n / \lambda')}$.
  \end{itemize}
\end{proof}

\section{Further Insights on Active Fine-Tuning}\label{sec:analysis_active_fine_tuning}

We expand the analysis of our results that we summarized in \sref{sec:results}.
We analyze aspects of the two key contributions of our work separately: In the following, we analyze the performance of \method in active fine-tuning, and in \sref{sec:analysis_test_time_fine_tuning}, we analyze the performance of test-time fine-tuning more generally.\looseness=-1

\paragraph{\insight \method's improvement over NN grows with dataset size.}
As shown in \cref{fig:dataset_size_advantages}, we find that the relative improvement of \method over Nearest Neighbor retrieval grows with dataset size.
We suspect that going from a small-size dataset to a medium-size dataset, the additional performance stems mainly from the ability of \method to adaptively select the same data for multiple gradient steps.
Going from a medium-size dataset to a large-size dataset, we suspect that the additional performance stems mainly from the ability of \method to select more diverse data points.\looseness=-1
\begin{figure}[H]
  \centering
  \incplt[0.62\textwidth]{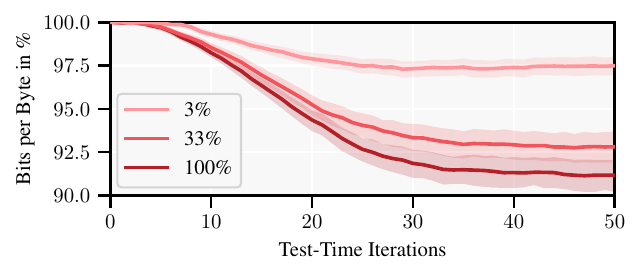}
  \vspace{-15pt}
  \caption{Bits per byte (in \% relative to the Nearest Neighbor retrieval baseline, $\downarrow$ better).
  We evaluate data selection from 3\%, 33\%, and 100\% of the Pile training dataset.
  We see a clear trend that \method's improvement over Nearest Neighbor retrieval grows with dataset size --- even from 33\% to 100\% with the highly curated Pile dataset.}
  \label{fig:dataset_size_advantages}
\end{figure}

\paragraph{\insight Points with high negative cosine similarity \emph{may} help.}
With the Roberta embedding model, we find that there are no negative cosine similarities in the data~(cf.~\cref{fig:cosine_similarity} in \sref{sec:additional_results}).
Choosing different embeddings such as influence embeddings can give negative cosine similarities~\citep[Appendix K.2]{xia2024less}.
Inspection of those points found by \cite{xia2024less} suggests that they can be equally informative as points with high positive cosine similarity.
Our derivation of \method naturally%
\setlength\parfillskip{0pt}\par\setlength\parfillskip{0pt plus 1fil}
\begin{wrapfigure}{r}{0.4\textwidth}
  \vspace{-15pt}
  \incplt[0.4\textwidth]{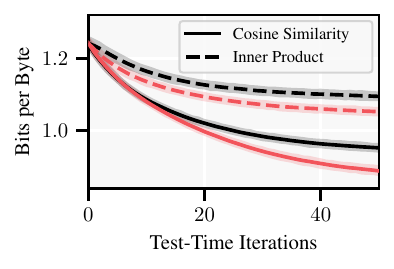}
  \vspace{-15pt}
  \caption{Data selection via \method~(\textcolor{red}{red}) and Nearest Neighbor~(black) performs best with normalized embeddings.}
  \label{fig:metrics_ablation}
  \vspace{-15pt}
\end{wrapfigure}
\vspace{-6pt}
addresses this by selecting points with large \emph{absolute} cosine similarity.
Geometrically, points with positive or negative cosine similarity are both equally ``parallel'' to the test prompt.
Our theoretical results suggest that the informativeness of a data point is closely related to how parallel its embedding is to the test prompt.
We leave further investigation to future work.\looseness=-1

\paragraph{\insight Normalizing embeddings helps.}
We evaluate the performance of Nearest Neighbor retrieval and \method with or without explicitly normalized embeddings in \cref{fig:metrics_ablation}.
We find that for both selection strategies, normalizing embeddings consistently improves performance.
Previously, \cite{hardt2023test} minimized the Euclidean distance between unnormalized embeddings, which we find to perform identically to maximizing cosine similarity.\looseness=-1

\section{Further Insights on Test-Time Fine-Tuning}\label{sec:analysis_test_time_fine_tuning}

\paragraph{\insight Scaling pre-training compute may not be all you need.\!\!}
In \cref{table:pile_benchmark} of \sref{sec:pile_benchmark}, we compare state-of-the-art LLMs to our test-time fine-tuned models.
We show that our Phi-3 with test-time fine-tuning outperforms all evaluated base models, from a wide selection of state-of-the-art LLMs, by a large margin.
Notably, we see a clear advantage of using stronger base models, i.e., better initializations.
The leading base model Gemma-2~\citep[27B,][]{team2024gemma}, which is $7\times$ larger and more recent than Phi-3, achieves $0.629$ bits per byte, whereas our test-time fine-tuned Phi-3 achieves $0.595$ bits per byte.
This indicates that scaling pre-training compute is not all you need to achieve state-of-the-art performance, and that test-time fine-tuning can be an effective method for improving the performance of a base LLM.\looseness=-1

\paragraph{\insight Test-time fine-tuning outperforms in-context learning in ``hard'' tasks.}
Interestingly, we observe that across all evaluated models, updating the base model via fine-tuning as opposed to augmenting the models' context leads to large improvements on the DeepMind Math, GitHub, ArXiv, and FreeLaw datasets.
We include the per-dataset results in \sref{sec:full_results_ttft}.
These datasets contain school-level math problems, code, scientific papers, and court opinions, which are often colloquially understood as tasks that require ``understanding'' or ``reasoning''.
In the case of DeepMind Math and ArXiv, augmenting the models' context does consistently not improve the performance of the base model at all, whereas test-time fine-tuning can lead to significant performance improvements.\looseness=-1

\begin{wrapfigure}{r}{0.4\textwidth}
  \vspace{-23pt}
  \incplt[0.4\textwidth]{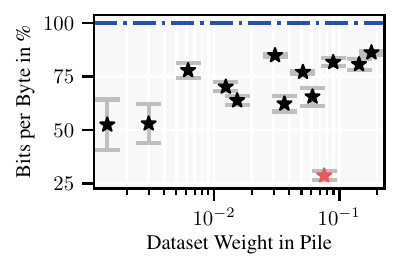}
  \vspace{-15pt}
  \caption{Improvement of $50$ test-time iterations over the base model (\textcolor{blue}{blue}; $\downarrow$~better) with \method against the percentage of bytes occupied by the dataset in the Pile.
  Error bars correspond to standard errors.
  We observe the trend that test-time fine-tuning benefits prompts at the ``boundary'' of the data distribution most.
  The ``outlier'' GitHub dataset is highlighted in \textcolor{red}{red}.}
  \label{fig:performance_against_dataset_weight}
  \vspace{-15pt}
\end{wrapfigure}

\paragraph{\insight Test-time fine-tuning yields largest gains at the boundary of the data distribution.}
In \cref{fig:performance_against_dataset_weight}, we plot the improvement of test-time fine-tuning with \method over the base model against the weight of a dataset in the Pile.
We observe the trend that test-time fine-tuning yields largest performance improvements for datasets that have a smaller weight in the Pile.
We hypothesize that this trend occurs because the weight of a dataset in the Pile corresponds roughly to the weight of similar data in the pre-training dataset of GPT-2, in which case the performance gains would be largest for prompts that are at the ``boundary'' of the data distribution.
Notable is the outlier of the large GitHub dataset where test-time fine-tuning leads to large performance gains.
We hypothesize that this is because coding is relatively dissimilar to other data in the Pile, and therefore the GitHub dataset can be seen as ``small'' relative to the rest of the data.\looseness=-1

We make the observation that if the problem domain is large (like general language modeling), almost every sub-task can be seen as at the ``boundary'' / as an ``outlier''.
We see that datasets closest to the center of mass of the data distribution do not benefit as much from test-time fine-tuning as datasets that are further away from the center of mass.
Therefore, we expect test-time fine-tuning to benefit those models most that are learning a diverse data distribution as opposed to models that are learning a very concentrated data distribution.\looseness=-1

\paragraph{\insight The order of fine-tuning data does not matter.}
In \cref{fig:order_of_examples}, we evaluate the performance of test-time fine-tuning with Nearest Neighbor retrieval when taking gradient steps in the order of selected data compared to reversed order.
We find that the order of gradient steps does not affect the final performance.
This indicates that sequentially fine-tuning on selected data is not necessary, and that batched gradient steps can be used to further speed up test-time fine-tuning.
We leave a detailed exploration of batched updates to future work.\looseness=-1
\vfill\clearpage
\begin{figure}[H]
  \incplt[0.62\textwidth]{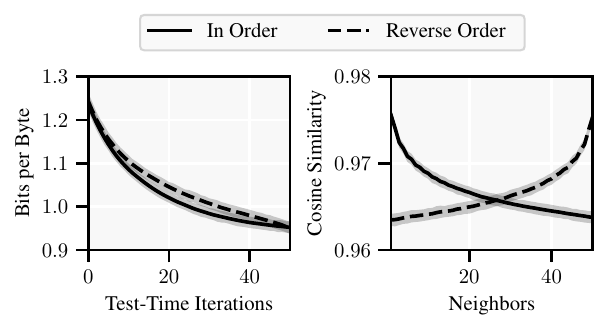}
  \vspace{-15pt}
  \caption{Taking gradient steps in order of selected data compared to reversed order. Data is selected using Nearest neighbor retrieval. We observe that the order of gradient steps does not affect the final performance.\looseness=-1}
  \label{fig:order_of_examples}
\end{figure}

\begin{wrapfigure}{r}{0.4\textwidth}
  \vspace{-25pt}
  \incplt[0.4\textwidth]{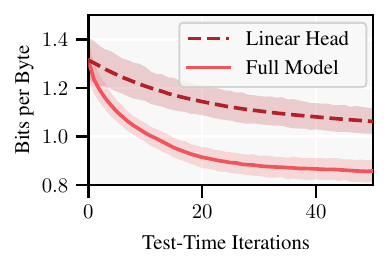}
  \vspace{-15pt}
  \caption{Bits per byte ($\downarrow$ better) against the number of test-time iterations. We compare fine-tuning only the linear head to fine-tuning the full model. We use learning rate $1\mathrm{e-}{4}$ and evaluate on 0.1\% of the full test set.}
  \label{fig:linear}
  \vspace{-15pt}
\end{wrapfigure}

\paragraph{\insight Test-time fine-tuning works also when fine-tuning only the last linear layer.}
Motivated by the linear representation hypothesis~(cf.~\cref{assumption:linear}) which informs \method's surrogate model for data selection, we evaluate whether we can fine-tune this surrogate model directly instead of fine-tuning the full model.
Concretely, we fine-tune only the last linear layer of the LLM, keeping its latent space fixed. %
The gradients for this linear surrogate model can be computed efficiently at almost no cost.
Remarkably, we find in \cref{fig:linear} that large gains of test-time fine-tuning can already be realized by fine-tuning only the last linear layer.
Given these preliminary results with GPT-2 it would be interesting to evaluate the performance gains of fine-tuning the linear head of larger base models.\looseness=-1

\paragraph{\insight Test-time fine-tuning works also with parameter-efficient fine-tuning.}
In our experiments with Phi-3, we use Low-Rank Adaptation~\citep[Lora,][]{hu2021lora} with a rank of $64$.
We find that LoRA converges slower than fine-tuning the full model, and therefore use the learning rate $5\mathrm{e}{-4}$, which is a factor $10$ larger than the learning rate used for fine-tuning the full model.
In \cref{fig:lora}, we evaluate the performance of LoRA compared to fine-tuning the full model.
On the smaller GPT-2 and GPT-2-large we use a rank of $32$.
We generally observe that fine-tuning with LoRA can recover roughly the same performance as fine-tuning the full model.
We expect that with more careful tuning of the learning rate, learning curves could be made more similar.\looseness=-1
\begin{figure}[H]
  \incplt[0.62\textwidth]{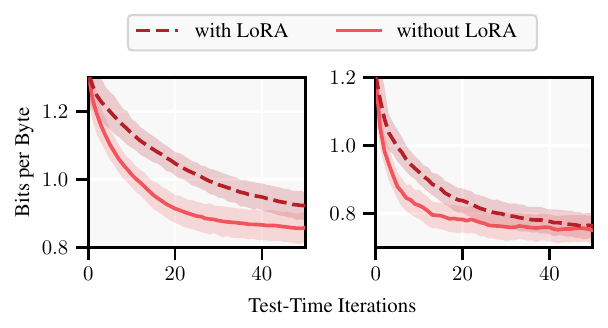}
  \vspace{-10pt}
  \caption{Bits per byte ($\downarrow$ better) against the number of test-time iterations. We compare parameter-efficient fine-tuning with LoRA and fine-tuning the full model. We use 0.1\% of the full test set.}
  \label{fig:lora}
\end{figure}

\begin{wrapfigure}{r}{0.4\textwidth}
  \vspace{-20pt}
  \incplt[0.4\textwidth]{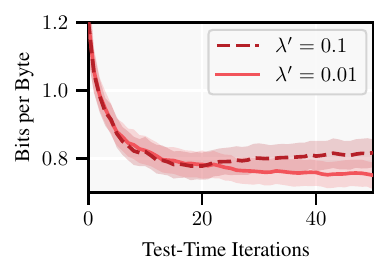}
  \vspace{-15pt}
  \caption{Bits per byte ($\downarrow$ better) with GPT-2-large and varying $\lambda'$. A larger $\lambda'$ can lead to overfitting in later iterations. We use 0.1\% of the full test set.}
  \label{fig:gpt2large_lambda}
  \vspace{-15pt}
\end{wrapfigure}

\paragraph{\insight Larger models appear to learn faster at test-time.}
We find that with a larger model (e.g., GPT-2-large vs GPT-2), a smaller $\lambda'$ tends to be more beneficial.
For example, keeping the learning rate fixed at $5\mathrm{e}{-5}$, using $\method(0.1)$ is the best choice for GPT-2, but leads to slight overfitting at later iterations for GPT-2-large as shown in \cref{fig:gpt2large_lambda}.
Recall that a smaller $\lambda'$ leads to more diverse sampling of the data space.
Thus, this observed trend indicates that larger models learn faster, and therefore benefit more from less redundant training data.
The same trend can also be observed from the behavior of NN-F from \cref{fig:spotlight_extended}:
GPT-2-large overfits much faster with NN-F than GPT-2.
This offers a potential explanation why the advantage of \method over Nearest Neighbor retrieval grows with larger models~(cf.~\sref{sec:full_results_aft}).\looseness=-1

\begin{figure}[H]
  \incplt[\textwidth]{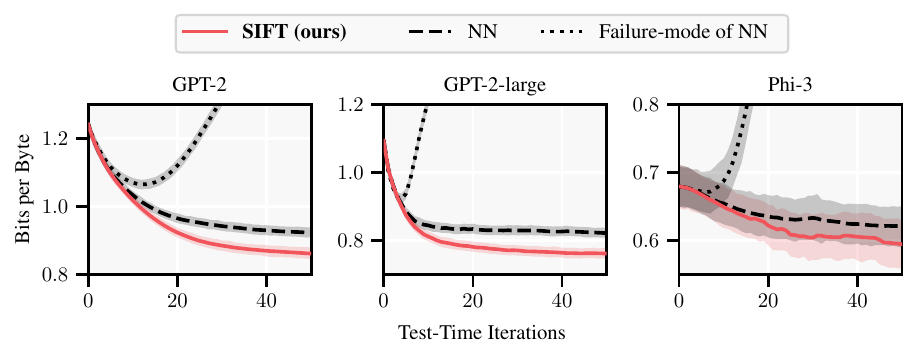}
  \vspace{-10pt}
  \caption{Bits per byte ($\downarrow$ better) against the number of test-time iterations with various base models.}
  \label{fig:spotlight_extended}
\end{figure}

\section{Extended Results}\label{sec:full_results}

This section includes additional per-dataset results to support our findings on active fine-tuning and test-time fine-tuning.

\subsection{Active Fine-Tuning}\label{sec:full_results_aft}

We compare \method against the data selection baselines Uncertainty Sampling (US), Nearest Neighbor retrieval (NN), and the failure-mode of Nearest Neighbor retrieval (with information duplication) that repeatedly retrieves the same point (NN-F).
Our results with GPT-2 as base model are summarized in the main text in \cref{table:main_results_per_dataset}.
\begin{itemize}
  \item In \cref{table:gpt2large_aft}, we include the comparison with \textbf{GPT-2-large}.
  \item In \cref{table:phi3_aft}, we include the comparison with \textbf{Phi-3}.
\end{itemize}
We find that our results on GPT-2 are consistent across all models.
In particular, test-time fine-tuning with \method improves the base model on \emph{all} datasets of the Pile, often significantly.
\method outperforms Uncertainty Sampling and Nearest Neighbor retrieval consistently.
Notably, we find that the improvement of \method over Nearest Neighbor retrieval is larger with stronger base models, indicating that informativeness of data becomes more important the stronger the base model.

\subsection{Test-Time Fine-Tuning}\label{sec:full_results_ttft}

We compare the in-context baseline against test-time fine-tuning.
\begin{itemize}
  \item In \cref{table:gpt2_ttft}, we include the comparison with \textbf{GPT-2}.
  \item In \cref{table:gpt2large_ttft}, we include the comparison with \textbf{GPT-2-large}.
  \item In \cref{table:phi3_ttft}, we include the comparison with \textbf{Phi-3}.
\end{itemize}
We find that test-time fine-tuning consistently outperforms in-context learning with GPT-2 and GPT-2-large.
With Phi-3, in-context learning and test-time fine-tuning have roughly matching performance, though test-time fine-tuning is more computationally efficient~(cf.~\cref{fig:absolute_perf_larger_models}).
Interestingly, we observe that test-time fine-tuning leads to large gains on math (``DeepMind Math'') and coding (``GitHub'') on all models, two tasks that require more complex reasoning.

\begin{table}[H]
  \centering
  \begin{tabularx}{0.69\textwidth}{@{\hspace{0mm}}l@{\hspace{2mm}}c@{\hspace{2mm}}c@{\hspace{2mm}}c@{\hspace{2mm}}c@{\hspace{2mm}}r}
    \toprule
    & \textbf{US} & \textbf{NN} & \textbf{\textsc{NN-F}} & \textbf{\method} & $\Delta$ \\[1pt]
    \hline \\[-6pt]
    NIH Grants & 96.6\reso{1.6} & 77.9\reso{4.8} & 107.6\reso{19.8} & \textbf{51.9}\reso{9.3} & \downres{26.0} \\[2pt]
    US Patents & 86.8\reso{2.3} & 78.9\reso{2.6} & 129.1\reso{7.7} & \textbf{64.7}\reso{3.8} & \downres{14.2} \\[2pt]
    Enron Emails & \textbf{73.9}\reso{12.3} & \textbf{68.6}\reso{13.6} & 102.9\reso{23.1} & \textbf{55.5}\reso{12.2} & \downres{13.1} \\[2pt]
    GitHub & 45.2\reso{2.4} & 42.8\reso{2.2} & 62.0\reso{4.5} & \textbf{31.0}\reso{2.2} & \downres{11.8} \\[2pt]
    Wikipedia & 71.0\reso{2.0} & 71.5\reso{2.0} & 141.3\reso{3.5} & \textbf{64.4}\reso{2.2} & \downres{6.6} \\[2pt]
    PubMed Abstr. & 94.5\reso{0.4} & 93.7\reso{0.6} & 202.6\reso{1.6} & \textbf{87.8}\reso{0.7} & \downres{5.9} \\[2pt]
    ArXiv & 90.6\reso{1.8} & 90.2\reso{2.0} & 175.8\reso{5.7} & \textbf{84.8}\reso{2.1} & \downres{5.4} \\[2pt]
    Hacker News & \textbf{79.4}\reso{2.6} & \textbf{79.0}\reso{2.9} & 138.7\reso{4.4} & \textbf{75.6}\reso{3.6} & \downres{3.4} \\[2pt]
    Stack Exchange & 84.1\reso{0.7} & 84.6\reso{0.8} & 165.2\reso{1.8} & \textbf{80.7}\reso{0.9} & \downres{3.4} \\[2pt]
    Common Crawl & 93.7\reso{0.6} & 89.9\reso{0.7} & 163.6\reso{2.1} & \textbf{87.1}\reso{1.0} & \downres{2.8} \\[2pt]
    PubMed Central & \textbf{87.9}\reso{2.7} & \textbf{87.6}\reso{2.7} & 157.8\reso{4.6} & \textbf{85.4}\reso{3.1} & \downres{2.2} \\[2pt]
    FreeLaw & \textbf{66.8}\reso{4.2} & \textbf{67.4}\reso{4.1} & 132.0\reso{6.4} & \textbf{68.3}\reso{4.2} & \upres{1.5} \\[2pt]
    DeepMind Math & \textbf{71.2}\reso{2.2} & \textbf{72.2}\reso{2.0} & 186.1\reso{4.1} & \textbf{74.2}\reso{2.3} & \upres{3.0} \\[2pt]
    \hline \\[-8pt]
    \emph{All} & 82.6\reso{0.6} & 80.6\reso{0.6} & 153.3\reso{1.4} & \textbf{74.9}\reso{0.7} & \downres{5.7} \\
    \bottomrule
  \end{tabularx}
  \caption{Results with \textbf{GPT-2-large}. Bits per byte (in \% relative to the base model, $\downarrow$) after $50$ test-time iterations on individual datasets of the Pile. We only include datasets with at least $10$ examples in our test set. \textbf{Bold} numbers denote the best performing selected subset. Numbers in parentheses are standard errors. $\Delta$ denotes the performance gain of \method over the strongest baseline.}
  \label{table:gpt2large_aft}
\end{table}

\begin{table}[H]
  \centering
  \begin{tabularx}{0.69\textwidth}{@{\hspace{0mm}}l@{\hspace{2mm}}c@{\hspace{2mm}}c@{\hspace{2mm}}c@{\hspace{2mm}}c@{\hspace{2mm}}r}
    \toprule
    & \textbf{US} & \textbf{NN} & \textbf{\textsc{NN-F}} & \textbf{\method} & $\Delta$ \\[1pt]
    \hline \\[-6pt]
    GitHub & 80.6 & 80.8 & 105.2 & 46.5 & \downres{34.1} \\[2pt]
    US Patents & 95.4 & 94.2 & 274.6 & 83.7 & \downres{10.5} \\[2pt]
    Enron Emails & 113.6 & 86.6 & 319.9 & 78.7 & \downres{7.9} \\[2pt]
    Wikipedia & 84.6 & 85.5 & 263.2 & 79.2 & \downres{5.4} \\[2pt]
    PubMed Abstr. & 93.5 & 93.3 & 301.8 & 89.5 & \downres{3.8} \\[2pt]
    NIH Grants & 100.4 & 100.1 & 327.6 & 98.6 & \downres{1.5} \\[2pt]
    ArXiv & 95.5 & 96.5 & 282.4 & 94.3 & \downres{1.2} \\[2pt]
    Common Crawl & 95.3 & 94.9 & 257.0 & 93.7 & \downres{1.2} \\[2pt]
    PubMed Central & 80.3 & 82.1 & 204.9 & 79.7 & \downres{0.6} \\[2pt]
    DeepMind Math & 76.4 & 75.5 & 221.4 & 75.3 & \downres{0.2} \\[2pt]
    Hacker News & 95.1 & 94.8 & 243.8 & 95.0 & \upres{0.2} \\[2pt]
    FreeLaw & 66.9 & 67.8 & 178.0 & 67.2 & \upres{0.3} \\[2pt]
    Stack Exchange & 99.7 & 98.7 & 309.9 & 99.4 & \upres{0.7} \\[2pt]
    \hline \\[-8pt]
    \emph{All} & 92.0\reso{1.1} & 91.6\reso{1.1} & 256.6\reso{7.1} & \textbf{85.7}\reso{2.0} & \downres{5.9} \\
    \bottomrule
  \end{tabularx}
  \caption{Results with \textbf{Phi-3}. Bits per byte (in \% relative to the base model, $\downarrow$) after $50$ test-time iterations on individual datasets of the Pile. \textbf{Bold} numbers denote the best performing selected subset. Numbers in parentheses are standard errors. $\Delta$ denotes the performance gain of \method over the strongest baseline.}
  \label{table:phi3_aft}
\end{table}

\begin{table}[H]
  \centering
  \begin{tabularx}{0.52\textwidth}{@{\hspace{0mm}}l@{\hspace{2mm}}c@{\hspace{2mm}}c@{\hspace{2mm}}r}
    \toprule
    & \textbf{Context} & \textbf{Fine-Tuning} & $\Delta$ \\[1pt]
    \hline \\[-6pt]
    GitHub & 74.5\reso{2.5} & \textbf{28.6}\reso{2.2} & \downres{45.9} \\[2pt]
    DeepMind Math & 100.4\reso{0.1} & \textbf{70.1}\reso{2.1} & \downres{30.3} \\[2pt]
    US Patents & 86.8\reso{2.5} & \textbf{62.2}\reso{3.6} & \downres{24.6} \\[2pt]
    Enron Emails & \textbf{73.3}\reso{9.8} & \textbf{52.4}\reso{11.8} & \downres{20.9} \\[2pt]
    FreeLaw & 85.5\reso{4.0} & \textbf{65.5}\reso{4.2} & \downres{20.0} \\[2pt]
    Stack Exchange & 96.7\reso{0.3} & \textbf{77.0}\reso{0.7} & \downres{19.7} \\[2pt]
    ArXiv & 99.2\reso{1.4} & \textbf{81.6}\reso{1.9} & \downres{17.6} \\[2pt]
    Wikipedia & 77.4\reso{2.1} & \textbf{63.7}\reso{2.1} & \downres{13.7} \\[2pt]
    PubMed Central & 92.8\reso{3.1} & \textbf{80.6}\reso{2.7} & \downres{12.2} \\[2pt]
    Hacker News & 89.0\reso{3.8} & \textbf{77.8}\reso{3.5} & \downres{11.2} \\[2pt]
    NIH Grants & \textbf{63.7}\reso{9.5} & \textbf{52.9}\reso{9.0} & \downres{10.8} \\[2pt]
    Common Crawl & 93.4\reso{0.7} & \textbf{86.1}\reso{0.9} & \downres{7.3} \\[2pt]
    PubMed Abstr. & 91.8\reso{0.6} & \textbf{84.8}\reso{0.7} & \downres{7.0} \\[2pt]
    \hline \\[-8pt]
    \emph{All} & 89.3\reso{0.5} & \textbf{73.2}\reso{0.7} & \downres{16.1} \\
    \bottomrule
  \end{tabularx}
  \caption{Comparison between the in-context baseline and test-time fine-tuning with \textbf{GPT-2}. Bits per byte (in \% relative to the base model, $\downarrow$) after $50$ test-time iterations on individual datasets of the Pile. We only include datasets with at least $10$ examples in our test set. \textbf{Bold} numbers denote the best performing selected subset. Numbers in parentheses are standard errors. $\Delta$ denotes the performance gain of test-time fine-tuning over in-context learning.}
  \label{table:gpt2_ttft}
\end{table}

\begin{table}[H]
  \centering
  \begin{tabularx}{0.52\textwidth}{@{\hspace{0mm}}l@{\hspace{2mm}}c@{\hspace{2mm}}c@{\hspace{2mm}}r}
    \toprule
    & \textbf{Context} & \textbf{Fine-Tuning} & $\Delta$ \\[1pt]
    \hline \\[-6pt]
    GitHub & 74.6\reso{2.5} & \textbf{31.0}\reso{2.2} & \downres{43.6} \\[2pt]
    DeepMind Math & 100.2\reso{0.7} & \textbf{74.2}\reso{2.3} & \downres{26.0} \\[2pt]
    US Patents & 87.4\reso{2.5} & \textbf{64.7}\reso{3.8} & \downres{22.7} \\[2pt]
    FreeLaw & 87.2\reso{3.6} & \textbf{68.3}\reso{4.2} & \downres{18.9} \\[2pt]
    Hacker News & 92.6\reso{2.7} & \textbf{75.6}\reso{3.6} & \downres{17.0} \\[2pt]
    Stack Exchange & 97.2\reso{0.4} & \textbf{80.7}\reso{0.9} & \downres{16.5} \\[2pt]
    NIH Grants & \textbf{67.7}\reso{9.4} & \textbf{51.9}\reso{9.3} & \downres{15.8} \\[2pt]
    Enron Emails & \textbf{71.9}\reso{10.2} & \textbf{55.5}\reso{12.2} & \downres{15.5} \\[2pt]
    ArXiv & 98.8\reso{1.8} & \textbf{84.8}\reso{2.1} & \downres{14.0} \\[2pt]
    Wikipedia & 76.6\reso{2.1} & \textbf{64.4}\reso{2.2} & \downres{12.2} \\[2pt]
    PubMed Central & 92.3\reso{3.3} & \textbf{85.4}\reso{3.1} & \downres{6.9} \\[2pt]
    Common Crawl & 93.5\reso{0.7} & \textbf{87.1}\reso{1.0} & \downres{6.4} \\[2pt]
    PubMed Abstr. & 91.6\reso{0.6} & \textbf{87.8}\reso{0.7} & \downres{3.8} \\[2pt]
    \hline \\[-8pt]
    \emph{All} & 89.4\reso{0.5} & \textbf{74.9}\reso{0.7} & \downres{14.5} \\
    \bottomrule
  \end{tabularx}
  \caption{Comparison between the in-context baseline and test-time fine-tuning with \textbf{GPT-2-large}. Bits per byte (in \% relative to the base model, $\downarrow$) after $50$ test-time iterations on individual datasets of the Pile. We only include datasets with at least $10$ examples in our test set. \textbf{Bold} numbers denote the best performing selected subset. Numbers in parentheses are standard errors. $\Delta$ denotes the performance gain of test-time fine-tuning over in-context learning.}
  \label{table:gpt2large_ttft}
\end{table}

\begin{table}[H]
  \centering
  \begin{tabularx}{0.52\textwidth}{@{\hspace{0mm}}l@{\hspace{2mm}}c@{\hspace{2mm}}c@{\hspace{2mm}}r}
    \toprule
    & \textbf{Context} & \textbf{Fine-Tuning} & $\Delta$ \\[1pt]
    \hline \\[-6pt]
    DeepMind Math & 100.8 & 75.3 & \downres{25.5} \\[2pt]
    GitHub & 71.3 & 46.5 & \downres{24.8} \\[2pt]
    FreeLaw & 78.2 & 67.2 & \downres{11.0} \\[2pt]
    ArXiv & 101.0 & 94.3 & \downres{6.4} \\[2pt]
    Enron Emails & 81.8 & 78.7 & \downres{3.1} \\[2pt]
    Hacker News & 97.6 & 95.0 & \downres{2.6} \\[2pt]
    Stack Exchange & 100.9 & 99.4 & \downres{1.4} \\[2pt]
    PubMed Central & 79.9 & 79.7 & \downres{0.2} \\[2pt]
    US Patents & 83.3 & 83.7 & \upres{0.4} \\[2pt]
    Wikipedia & 77.1 & 79.2 & \upres{2.1} \\[2pt]
    NIH Grants & 95.1 & 98.6 & \upres{3.5} \\[2pt]
    Common Crawl & 89.9 & 93.7 & \upres{3.8} \\[2pt]
    PubMed Abstr. & 85.7 & 89.5 & \upres{3.8} \\[2pt]
    \hline \\[-8pt]
    \emph{All} & \textbf{87.1}\reso{1.7} & \textbf{85.7}\reso{2.0} & \downres{1.4} \\
    \bottomrule
  \end{tabularx}
  \caption{Comparison between the in-context baseline and test-time fine-tuning with \textbf{Phi-3}. Bits per byte (in \% relative to the base model, $\downarrow$) after $50$ test-time iterations on individual datasets of the Pile. \textbf{Bold} numbers denote the best performing selected subset. Numbers in parentheses are standard errors. $\Delta$ denotes the performance gain of test-time fine-tuning over in-context learning.}
  \label{table:phi3_ttft}
\end{table}

\section{\method Maximizes Information Gain}\label{sec:method_maximizes_info_gain}

We discuss here briefly that \method can be interpreted as maximizing the information gain of data~$X_n$ on the response to the prompt~$\prompt$.\looseness=-1

This probabilistic interpretation takes the perspective that the sequence model predicting the next token is a \emph{probabilistic model with a prior belief} over its state $\mW$ which induces an epistemic prior belief over what might be the next token.\footnote{This \emph{epistemic} uncertainty is distinct from the irreducible \emph{aleatoric} uncertainty of natural language, such as uncertainty about the continuation of ``I love \dots''.}
Our main text describes a closed loop where this sequence model interacts with a non-parametric memory (i.e., the data space) to update its epistemic beliefs about $\mW$, obtaining posterior beliefs $\mW \mid D$ conditional on the selected data $D$.
Again, these posterior epistemic beliefs induce an epistemic uncertainty over what might be the next token.
We discuss in the following how \method can be interpreted probabilistically; as the model interacting with the non-parametric memory with the goal of reducing its posterior uncertainty about the next token.\looseness=-1

Our brief overview will proceed as follows: \begin{itemize}
  \item We establish fundamentals from information theory and Gaussian processes, which are a tractable probabilistic model~(\sref{sec:method_maximizes_info_gain:preliminaries}).
  \item We define the prior belief and probabilistic observation model and derive the posterior belief~(\sref{sec:method_maximizes_info_gain:probabilistic_model}).
  \item We show that, in this probabilistic model, \method can be interpreted as maximizing the information gain of the data about the response to the prompt~$\prompt$~(\sref{sec:method_maximizes_info_gain:probabilistic_interpretation}).
  \item We show that balancing relevance and diversity of data is a natural consequence of maximizing information gain~(\sref{sec:method_maximizes_info_gain:balancing_relevance_diversity}).
\end{itemize}

\method uses relatively simple probabilistic surrogate models that are tractable, and which remarkably lead to strong empirical performance.
\cite{hubotter2024transductive} cover the probabilistic interpretation in greater detail.\looseness=-1

\subsection{Preliminaries: Information Theory and Gaussian Processes}\label{sec:method_maximizes_info_gain:preliminaries}

\paragraph{Information Theory}

We briefly recap several important concepts from information theory.
The (differential) entropy~${\H{\vf} \defeq \E[p(\vf)]{- \log p(\vf)}}$ of a random vector~$\vf$ is one possible measure of uncertainty about~$\vf$.
Here, $- \log p(\vf)$ is also called the suprisal about an event with density $p(\vf)$.
The entropy can be interpreted as the expected suprisal about~$\vf$ upon realization.
The conditional entropy ${\H{\vf}[\vy] \defeq \E[p(\vf, \vy)]{- \log p(\vf \mid \vy)}}$ is the (expected) posterior uncertainty about~$\vf$ after observing the random vector~$\vy$.
The information gain~${\I{\vf}{\vy} = \H{\vf} - \H{\vf}[\vy]}$ measures the (expected) reduction in uncertainty about~$\vf$ due to~$\vy$.
Refer to \cite{cover1999elements} for more details.\looseness=-1

\paragraph{Gaussian Processes}

The stochastic process $f$ is a Gaussian process (GP, \cite{williams2006gaussian}), denoted ${f \sim \GP{\mu}{k}}$, with mean function $\mu$ and kernel $k$ if for any finite subset ${X = \{\vx_1, \dots, \vx_n\} \subseteq \spX}$, ${\vfsub{X} \sim \N{\vmusub{X}}{\mKsub{X}}}$ is jointly Gaussian with mean vector ${(\vmusub{X})_i = \mu(\vx_i)}$ and covariance matrix~${(\mKsub{X})_{i,j} = k(\vx_i, \vx_j)}$.
A Gaussian process can be interpreted as capturing an epistemic functional belief, i.e., a belief over functions.
Our linear surrogate model from \cref{assumption:linear} leads to a Gaussian process with the linear kernel described in the main text.
That is, our surrogate model assumption can be interpreted as the \emph{prior belief} that the ground truth function predicting the next token is a logit-linear function in a latent representation space.
This is closely linked to the hypothesis that LLMs learn linear representations of high-level concepts, which is widely known as the ``linear representation hypothesis''~\citep[e.g.,][]{park2024linear,mikolov2013linguistic,arora2016latent,elhage2022toy}.
There are two lenses through which to view such linear Gaussian processes: the \emph{weight-space} view which considers a belief about weights $\mW$, or the \emph{function-space} view which directly considers the belief about functions $f$.
Both views are equivalent, and we will focus on the function-space view in the following.\looseness=-1

For Gaussian random vectors $\vf$ and $\vy$, the entropy is ${\H{\vf} = \frac{d}{2} \log(2 \pi e) + \frac{1}{2} \log \det{\Var{\vf}}}$ and the information gain is ${\I{\vf}{\vy} = \frac{1}{2} \parentheses*{\log\det{\Var{\vf}} - \log\det{\Var{\vf \mid \vy}}}}$.\looseness=-1

\subsection{Probabilistic Observation Model}\label{sec:method_maximizes_info_gain:probabilistic_model}

We will focus in the following on the case of regression, which we introduced in \cref{sec:proofs_regression}.
We suppose that observations of $f$ follow the probabilistic model \begin{align*}
  y_{\vx} = f_{\vx} + \varepsilon_{\vx},
\end{align*} where we make the following assumptions about the prior distribution of~$f$ and the noise~$\varepsilon_{\vx}$:\looseness=-1

\begin{assumption}[Gaussian prior]\label{asm:bayesian_prior}
   We assume that ${f \sim \GP{\mu}{k}}$ with known mean function $\mu$ and kernel $k$.
\end{assumption}

\begin{assumption}[Gaussian noise]\label{asm:bayesian_noise}
   We assume that the noise $\varepsilon_{\vx}$ is mutually independent and zero-mean Gaussian with known variance $\rho^2 > 0$.
\end{assumption}

Under \cref{asm:bayesian_prior,asm:bayesian_noise}, the posterior distribution of~$f$ after observing points~$X$ with values~$\vysub{X}$ is~$\GP{\mu_n}{k_n}$ with \begin{align*}
    \mu_n(\vx) &= \mu(\vx) + \transpose{\vk_X}(\vx) \inv{(\mKsub{XX} + \rho^2 \mI)} (\vysub{X} - \vmusub{X}), \\
    k_n(\vx,\vxp) &= k(\vx,\vxp) - \transpose{\vk_X}(\vx) \inv{(\mKsub{XX} + \rho^2 \mI)} \vk_X(\vxp), \\
    \sigma_n^2(\vx) &= k_n(\vx,\vx).
\end{align*}

\subsection{The Probabilistic Interpretation of \method}\label{sec:method_maximizes_info_gain:probabilistic_interpretation}

Observe that the above definition of $\sigma_n^2$ matches the definition from \cref{eq:variance}.\footnote{Notably, it can also be shown that $\mu_n$ is the closed-form solution to the regularized loss from \cref{eq:reg_nll_loss_reg}.}
That is, under the above probabilistic model, \begin{align*}
  \sigma_n^2(\vx) = \Var{f(\vx) \mid y_{1:n}}.%
\end{align*}
As such, \methodp is minimizing the variance of the response to the prompt~$\prompt$ after observing the data~$X_n$: \begin{align}
  \vx_{n+1} &= \argmin_{\vx \in \spD} \Var{f(\prompt) \mid y_{1:n}, y(\vx)}. \nonumber
  \intertext{By simple algebraic manipulation this can be seen to be equivalent to maximizing the information gain of the data on the response to the prompt~$\prompt$:}
  \vx_{n+1} &= \argmax_{\vx \in \spD} \frac{1}{2} \Big(\underbrace{\log \Var{f(\prompt) \mid y_{1:n}}}_{\const} - \log \Var{f(\prompt) \mid y_{1:n}, y(\vx)}\Big) \nonumber \\
  &= \argmax_{\vx \in \spD} \I{f(\prompt)}{y(\vx)}[y_{1:n}]. \label{eq:probabilistic_interpretation}
\end{align}

\paragraph{Discussion}

The above offers a very intuitive probabilistic interpretation of \methodp.
In this probabilistic interpretation, the regularization parameter $\lambda'$ of \method is equal to the observation noise~$\rho^2$.
Intuitively, larger observation noise leads to slower convergence of the estimate of $f$, analogously to our discussion of larger regularization parameter and smaller step size in \cref{prop:reg_loss_min_vs_ttft}.\looseness=-1

The reason why \methodp can be interpreted \emph{both} as minimizing the variance and as minimizing the entropy of the response to the prompt~$\prompt$ is that for Gaussians, variance is proportional to the entropy of the response to the prompt~$\prompt$.
As observed by \cite{hubotter2024transductive}, if learning is amortized with respect to multiple prompts~$\{\prompt_1, \dots, \prompt_m\} = \spA$, this ceases to be the case and the two objectives lead to different data selection schemes.
It appears to be a special property of non-amortized transductive active learning that measures of uncertainty and resulting data selection schemes are interchangeable.\looseness=-1

A quick remark is in order.
\method does not only maximize the marginal information gain as shown in \cref{eq:probabilistic_interpretation}, if \cref{assumption:submodular} is satisfied, it also maximizes the joint information gain $\I{f(\prompt)}{y_{1:n}}$.
That is, in this case the ``entropy reduction'' of data~$X_n$ selected by \method achieves a constant factor approximation of the maximum possible joint information gain $\max_{X \subseteq \spD, \abs{X} \leq n} \I{f(\prompt)}{\vy(X)}$ due to the seminal result on monotone submodular function maximization of \cite{nemhauser1978analysis}.\looseness=-1

\subsection{How \method Balances Relevance and Diversity}\label{sec:method_maximizes_info_gain:balancing_relevance_diversity}

In \sref{sec:method_details:balancing_relevance_diversity}, we discussed how \method chooses data that is both relevant and diverse.
The probabilistic interpretation offers a simple explanation for how this behavior naturally emerges from selecting the most informative data.
To this end, observe that the information gain from \cref{eq:probabilistic_interpretation} can be expressed as\looseness=-1 \begin{align}
  \I{f(\prompt)}{y(\vx)}[y_{1:n}] = \underbrace{\I{f(\prompt)}{y(\vx)}}_{\text{relevance}} - \underbrace{\I{f(\prompt)}{y(\vx); y_{1:n}}}_{\text{redundancy}} \label{eq:information_gain_decomposition}
\end{align} where $\I{\vf}{\vx; \vy} \defeq \I{\vf}{\vx} - \I{\vf}{\vx}[\vy] = \I{\vf}{\vx} + \I{\vf}{\vy} - \I{\vf}{\vx, \vy}$ denotes the multivariate information gain~\citep{murphy2023probabilistic}.
The multivariate information gain is a measure of the redundancy of $\vx$ and $\vy$ in predicting $\vf$, and is therefore often called simply ``redundancy'' (which is the opposite of ``synergy'').
\Cref{eq:information_gain_decomposition} shows that the balancing of relevance and non-redundancy (i.e., diversity) arises naturally from maximizing the information gain.
Note that the tradeoff between relevance and diversity is governed by the noise parameter $\rho^2 \approx \lambda'$ of the probabilistic model.\looseness=-1

\subsection{The Perspective of Classification}

The above interpretation takes the perspective of regression.
However, the above interpretation can be extended to classification.
We will focus here on the case of binary classification for notational convenience, but the same argument can be made for multi-class classification \citep[Section 3.5]{williams2006gaussian}.\looseness=-1

In (binary) Gaussian Process Classification the logit $f \sim \GP{\mu}{k}$ is modeled as a Gaussian process, and the likelihood follows the model introduced in \cref{sec:preliminaries}: $y(\vx) \sim \mathrm{Bern}(s(f(\vx)))$ where we have Bernoulli rather than categorical feedback and use the logistic function $s(a) \defeq 1/(1+e^{-a})$ rather than the softmax by virtue of restricting to binary classification.\looseness=-1

The standard approach \citep[Section 3.4]{williams2006gaussian} is to approximate the posterior distribution of the latent function $f$ given observations $y_{1:n}$ by a Gaussian using Laplace's method.
This Gaussian can be shown to have covariance $\smash{\inv{(\inv{\mKsub{X_n}} + \mW)}}$ with $\mW \succeq \inv{\kappa} \mI_n$ where $\smash{\kappa \defeq \sup_{a \leq B} 1/\dot{s}(a)}$ and $\smash{\dot{s}(a) = s(a)(1-s(a))}$ denotes the derivative of the logistic function.\footnote{In the binary case, this is equal to the more general $\kappa$ from the main text.}
It is then straightforward to derive that \begin{align*}
  \sigma_n^2(\prompt) &= k(\prompt,\prompt) - \transpose{\vk_{X_n}}(\prompt) \inv{(\mKsub{X_n} + \inv{\mW})} \vk_{X_n}(\prompt) \\
  &\leq k(\prompt,\prompt) - \transpose{\vk_{X_n}}(\prompt) \inv{(\mKsub{X_n} + \kappa \mI_n)} \vk_{X_n}(\prompt)
\end{align*}
Thus, \method minimizes a tight upper bound to the (approximate) posterior variance of the latent function $f$ at the prompt $\prompt$.
The same relationship to maximizing information gain that was discussed above applies.\looseness=-1

\section{Efficient Computation of \method}\label{sec:efficient_computation_via_conditional_embeddings}

In the following, we show how to select data via \method at low computational cost.
Our implementation extends the Faiss library~\citep{johnson2019billion,douze2024faiss} for Nearest Neighbor retrieval.
We open-source the \href{https://github.com/jonhue/activeft}{\texttt{activeft}} (Active Fine-Tuning) library which can be used as a drop-in replacement for Nearest Neighbor retrieval. %

In our runtime analysis, we will denote by $K$ the size of the data space $\spD$, and by $N$ the number of points to be selected.
We describe two implementations of \method: \begin{enumerate}
  \item The first exact implementation has sequential computation cost $\BigO{K^2 N}$, however, computation can be effectively parallelized on a GPU.
  \item The second ``fast'' implementation assumes submodularity (i.e., \cref{assumption:submodular}) and has computation cost $\smash{\BigOTil{K + N^3}}$ where $\smash{\BigOTil{\cdot}}$ suppresses log-factors.
  This cost is only marginally above the cost of Nearest Neighbor retrieval.
\end{enumerate}
Both implementations achieve virtually identical performance gains~(cf.~\figref{fig:k_ablation_ext}{right}), which is further evidence that \cref{assumption:submodular} is satisfied in our language modeling setting.

\subsection{Exact Implementation}

The central object of the first implementation is the conditional kernel matrix of the data space given the selected points $X_n$: \begin{align*}
  \mKsub{n} \defeq \mKsub{\spD} - \mKsub{\spD, X_n} \inv{(\mKsub{X_n} + \lambda' \mI_n)} \mKsub{X_n, \spD}.
\end{align*}
The entries $k_n(\vx,\vxp)$ of this matrix can be updated efficiently via the following relation~\citep[Appendix F]{chowdhury2017kernelized} arising from properties of the Schur complement: \begin{align}
  k_{n}(\vx,\vxp) = k_{n-1}(\vx,\vxp) - \frac{k_{n-1}(\vx,\vx_n) k_{n-1}(\vx_n,\vxp)}{k_{n-1}(\vx_n,\vx_n) + \lambda'}. \label{eq:recursive_kernel_update}
\end{align}
The implementation is detailed in \cref{alg:first_implementation}.
The computation of the objective value in \cref{alg:first_implementation:selection} and the kernel matrix update in \cref{alg:first_implementation:update} can be parallelized on a GPU.
Thus, the main bottleneck of this implementation is the requirement that the kernel matrix of size $K \times K$ fits onto a GPU.
In case this is not possible, such as with large data spaces, the following two sections detail methods to reduce the computational cost.\looseness=-1

\begin{algorithm}
  \caption{\methodl}\label{alg:first_implementation}
  \begin{algorithmic}[1]
  \STATE \textbf{Input:} prompt $\prompt$, data space $\spD$, (initial) kernel matrix $k_0(\vx,\vxp) = \transpose{\vphi(\vx)} \vphi(\vxp),\,\vx,\vxp \in \spD$, number of points to select $N$
  \STATE \textbf{Output:} set of selected points $\{\vx_1, \dots, \vx_N\}$
  \FOR{$n$ from $1$ to $N$}
    \STATE\label{alg:first_implementation:selection} $\vx_n \gets \argmax_{\vx \in \spD} \frac{k_{n-1}^2(\prompt, \vx)}{k_{n-1}(\vx, \vx) + \lambda'}$ \hfill \COMMENT{Select next point}
    \FOR{each $\vx, \vxp \in \spD$}\label{alg:first_implementation:update}
      \STATE Update $k_{n}(\vx, \vxp) \leftarrow k_{n-1}(\vx, \vxp) - \frac{k_{n-1}(\vx, \vx_n) k_{n-1}(\vx_n, \vxp)}{k_{n-1}(\vx_n, \vx_n) + \lambda'}$ \hfill \COMMENT{Update kernel matrix}
    \ENDFOR
  \ENDFOR
  \end{algorithmic}
  \end{algorithm}

\subsection{Fast (Exact) Implementation}

The following ``fast'' implementation of \method rests on the assumption that the objective function optimized by \method is submodular (cf. \cref{assumption:submodular}).
Recall that this objective function can be expressed as $\vx_{n+1} = \argmax_{\vx \in \spD} \psi_{\prompt}(X_n \cup \{\vx\})$ where $\psi_{\prompt}(X) = \sigma_0^2(\prompt) - \sigma_X^2(\prompt)$ denotes the \emph{uncertainty reduction} about $\prompt$ upon fine-tuning the model on data $X$.\looseness=-1

The ``trick'' of the fast implementation is to use a max-heap (with $\BigO{1}$ lookup and $\BigO{\log K}$ insertion) to keep track of upper bounds of $\psi_{\prompt}(X_n \cup \{\vx\})$ for each $\vx \in \spD$.
The upper bounds come directly from the submodularity assumption: \begin{align*}
  \psi_{\prompt}(X_i \cup \{\vx\}) \geq \psi_{\prompt}(X_j \cup \{\vx\}) \quad \forall j \geq i.
\end{align*}
At iteration $n$, we evaluate $\psi_{\prompt}(X_{n-1} \cup \{\vx\})$ for $\vx$ in max-heap order.
As soon as we find a $\vx$ whose re-computed upper bound is smaller than a previously re-computed upper bound, we stop the evaluation.
In the worst case, one might iterate through all $K$ points in each iteration, but in practice, it can sometimes be reasonable to assume that one only needs to consider $\BigO{1}$ points per iteration.
This algorithm is known as the ``lazy greedy algorithm'' in submodular function maximization~\citep{minoux1978accelerated} where it is typically seen to result in large speed-ups.\looseness=-1

We summarize the fast implementation in \cref{alg:second_implementation}.
The kernel matrix $\mK$ tracks the conditional kernel matrix of the prompt $\prompt$ and the previously selected data $X_{n-1}$.
$\mLambda$ tracks the (regularized) inverse of the kernel matrix of the previously selected data $X_{n-1}$.
Whenever necessary, the cached kernel matrix and cached inverse are updated.
We denote by $\mPhi \in \R^{(n-1) \times d}$ the matrix of embeddings of previously selected points and by $\smash{\tilde{\mPhi} \in \R^{n \times d}}$ the same matrix extended by $\vphi(\prompt)$ as the first row.\looseness=-1

Initializing the max-heap takes time $\BigOTil{K}$ and is analogous to standard Nearest Neighbor retrieval.
Additionally, \methodfast performs a data selection loop for $N$ iterations where each operation takes $\BigO{N^2}$ time requiring persistent memory of size $\BigO{N^2}$.
Notably, only the kernel matrix of the prompt and the previously selected data is kept in memory.\looseness=-1

\begin{algorithm}
  \caption{\methodfastl}\label{alg:second_implementation}
  \begin{algorithmic}[1]
  \STATE \textbf{Input:} prompt $\prompt$, data space $\spD$, number of points to select $N$
  \STATE \textbf{Output:} set of selected points $\{\vx_1, \dots, \vx_N\}$
  \\[5pt]
  \COMMENT{Initializing max-heap (``Nearest Neighbor retrieval'')}
  \FOR{$\vx \in \spD$}
    \STATE $\alpha_{\vx} \gets \frac{(\transpose{\vphi(\prompt)} \vphi(\vx))^2}{\norm{\vphi(\vx)}_2^2 + \lambda'}$
    \STATE Insert $(\vx, \alpha_{\vx})$ into max-heap
  \ENDFOR
  \\[5pt]
  \COMMENT{Data selection}
  \STATE Initialize $\mK = \begin{bmatrix}
    \norm{\vphi(\prompt)}_2^2
  \end{bmatrix}$ and $\mLambda$ as an empty square matrix
  \FOR{$n$ from $1$ to $N$}
    \STATE Initialize lower bound $\opt{\alpha} \gets -\infty$
    \FOR{each popped $(\vx,\alpha)$ in max-heap order}
      \IF{$\alpha = \opt{\alpha}$}
        \STATE $\vx_n \gets \vx$ \hfill \COMMENT{$\vx$ maximizes the \methodl objective}
        \STATE \textbf{break}
      \ENDIF
      \STATE\label{alg:second_implementation:selection} $\alpha_{\vx}, \mLambda, \mK' \gets \hyperref[alg:second_implementation:recompute]{\textsc{Recompute}}(\vx, \mK, \mLambda)$ \hfill \COMMENT{Recompute objective value}
      \STATE $\opt{\alpha} \gets \max\{\opt{\alpha}, \alpha_{\vx}\}$
      \STATE Insert $(\vx, \alpha_{\vx})$ into max-heap
    \ENDFOR
    \STATE $\mK \gets \hyperref[alg:second_implementation:update]{\textsc{UpdateState}}(\vx_n, \mK')$ \hfill \COMMENT{Update cached kernel matrix}
  \ENDFOR
  \end{algorithmic}
\end{algorithm}

\begin{algorithm}
  \caption{\methodfastl: \textsc{Recompute}}\label{alg:second_implementation:recompute}
  \begin{algorithmic}[1]
  \STATE \textbf{Input:} prompt $\prompt$, current iteration $n$, candidate $\vx$, cached kernel matrix $\mK$, cached inverse $\mLambda$
  \STATE \textbf{Output:} objective value $\alpha_{\vx}$, updated cached inverse $\mLambda$, expanded kernel matrix $\mK$
  \\[5pt]
  \COMMENT{Expand cached kernel matrix $\mK$ (if required)}
  \IF{$\vx$ is has not been selected yet}
    \STATE
    \COMMENT{Update $\mLambda$ with the Sherman-Morrison-Woodbury formula~\citep{sherman1950adjustment}}
    \STATE Let $i$ denote the size of $\mLambda$
    \IF{$i < n-1$}
      \STATE $\mA \gets \mPhi_i \mPhi_{i+1:n-1}^\top$
      \STATE $\mB \gets \mPhi_{i+1:n-1} \mPhi_{i+1:n-1}^\top$
      \STATE $\mC \gets \inv{(\mB - \mA^\top \mLambda \mA)}$
      \STATE $\mLambda \gets \begin{bmatrix}
        \mLambda + \mLambda \mA \mC \mA^\top \mLambda & -\mLambda \mA \mC \\
        -\mC \mA^\top \mLambda & \mC
      \end{bmatrix}$
    \ENDIF
    \\[5pt]
    \COMMENT{Expand kernel matrix $\mK$}
    \STATE $\mA \gets \mI - \mPhi^\top \mLambda \mPhi$
    \STATE $\vk \gets \tilde{\mPhi} \mA \vphi(\vx)$
    \STATE $\mK \gets \begin{bmatrix}
      \mK & \vk \\
      \vk^\top & \norm{\vphi(\vx)}_{\mA}^2
    \end{bmatrix}
    $
  \ENDIF
  \\[5pt]
  \STATE $\alpha_{\vx} \gets \frac{k^2(\prompt,\vx)}{k(\vx,\vx) + \lambda'}$ \hfill \COMMENT{Compute objective value using the relation from \cref{eq:recursive_kernel_update}}
  \end{algorithmic}
\end{algorithm}

\begin{algorithm}
  \caption{\methodfastl: \textsc{UpdateState}}\label{alg:second_implementation:update}
  \begin{algorithmic}[1]
  \STATE \textbf{Input:} selected point $\vx_n$, expanded kernel matrix $\mK'$
  \STATE \textbf{Output:} new conditional kernel matrix $\mK$
  \\[5pt]
  \COMMENT{Update kernel matrix using the relation from \cref{eq:recursive_kernel_update}}
  \FOR{each $\vx, \vxp \in \{\prompt\} \cup X_n$}
    \STATE Update $k(\vx, \vxp) \leftarrow k'(\vx, \vxp) - \frac{k'(\vx, \vx_n) k'(\vx_n, \vxp)}{k'(\vx_n, \vx_n) + \lambda'}$
  \ENDFOR
  \end{algorithmic}
\end{algorithm}

\subsection{Pre-Selecting Data via Nearest Neighbor Retrieval}\label{sec:pre_selecting_data_via_nearest_neighbor_retrieval}

\begin{figure}
  \centering
  \incplt[0.4\textwidth]{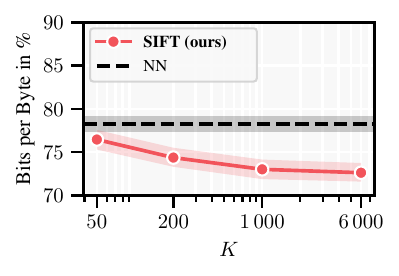}
  \vspace{-15pt}
  \caption{We run \method$\!(\lambda' = 1)$ with various values of $K$ and report the bits per byte ($\downarrow$ better) after $50$ test-time iterations. We find that performance on the Pile plateaus after $K=1`000$. Even at $K=50$, which equals the number of points selected, \method outperforms Nearest Neighbor retrieval due to being able to select the same points multiple times.}
  \label{fig:k_ablation_perf}
\end{figure}

The reason for \methodfast being so efficient is that it effectively ``discards'' all points in $\spD$ that are completely irrelevant to the prompt.
Whereas \method recomputes the objective value of every point in $\spD$ at each iteration, \methodfast only reevaluates points that are potentially relevant.
An alternative to make \method fast is therefore simply to preemptively discard irrelevant points.
In our experiments we do so by pre-selecting a subset of size $K=200$ via Nearest Neighbor retrieval within~$\spD$~(cf.~\cref{sec:experiment_details} for more details).
This step aims to eliminate all points from the data space that \method would not end up picking anyway while retaining a diverse set of relevant points.
\Cref{fig:k_ablation_perf} shows the effect of $K$ on statistical performance and \cref{fig:k_ablation} shows the effect on computational performance.\looseness=-1

\subsection{Future Work: Improving GPU Utilization of \methodfast}

\begin{figure}[]
  \centering
  \incplt[\textwidth]{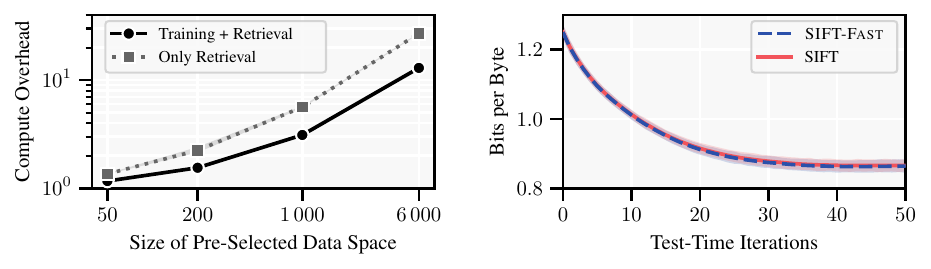}
  \vspace{-15pt}
  \caption{\textbf{Left:}~Computational overhead of \methodfast over Nearest Neighbor retrieval.
  This overhead is larger than the overhead of \method over Nearest Neighbor retrieval~(cf.~\cref{fig:k_ablation}).
  \textbf{Right:}~\methodfast achieves identical statistical performance to \method, which is further evidence that \cref{assumption:submodular} is satisfied in our language modeling setting.\looseness=-1}
  \label{fig:k_ablation_ext}
\end{figure}

In our experiments on the Pile dataset, we find that \methodfast is less efficient than \method~(cf.~\figref{fig:k_ablation_ext}{left}).
We attribute this to the fact that for any given prompt, the closest neighbors in the data space are all relatively similar to the prompt~(cf.~\cref{fig:cosine_similarity_zoom}), meaning that each iteration of \methodfast has to loop (sequentially) over the entire priority queue.
In contrast, \method performs this operation in parallel on a GPU.

\begin{wrapfigure}{r}{0.4\textwidth}
  \vspace{-20pt}
  \centering
  \incplt[0.4\textwidth]{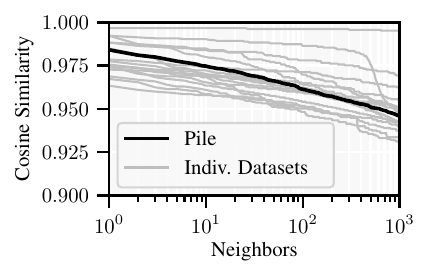}
  \vspace{-15pt}
  \caption{Average cosine similarities of test prompts to closest $1`000$ neighbors in the data space of the Pile; with the Roberta embedding model.}
  \label{fig:cosine_similarity_zoom}
  \vspace{-15pt}
\end{wrapfigure}

We believe that a promising computational approach is to combine the advantages of the \method and \methodfast implementations.
This could be achieved by keeping a large sub-selected kernel matrix on the GPU (akin to the \method implementation) and selectively using the \methodfast implementation if points on the priority queue that are not in the sub-selected kernel matrix may be selected.
This would allow for a more efficient use of the GPU memory of \methodfast, which we expect to yield comparable computational performance to the \method implementation in most cases, while still being able to handle large data spaces.

\section{Experiment Details}\label{sec:experiment_details}

We fine-tune the pre-trained model for a single gradient step each on $N=50$ selected data points.
We evaluate the performance on 1\% of the test instances of the Pile.
We use the Pile training dataset as data space for data selection, which notably does \emph{not} include data from the validation and test sets.\looseness=-1

\paragraph{Evaluation}
We use the standard implementation of the \texttt{lm-evaluation-harness} library~\citep{gao2024framework} for computing the bits per byte.
This implementation computes the log-likelihood of a document using a rolling-window approach, which ensures that the model's maximum context window is fully utilized.\looseness=-1

\paragraph{Truncation of Long Sequences}
Analogously to \cite{hardt2023test}, to generate embeddings, we naively truncate long sequences to the maximum sequence length of the embedding model, that is, we only consider the prefixes of long sequences for data selection.\looseness=-1

\paragraph{Learning Rate and Optimizer}
Following \cite{hardt2023test}, we use the Adam optimizer \citep{kingma2014adam} with $\epsilon$-value~$1\mathrm{e}{-8}$.
We use the default learning rate $5\mathrm{e}{-5}$ of the \texttt{transformers} library~\citep{wolf2020huggingface} unless noted otherwise.
\cite{hardt2023test} used a learning rate of $2\mathrm{e}{-5}$ for their experiments.
We show in \cref{fig:learning_rate_ablation} that $5\mathrm{e}{-5}$ leads to strictly better performance of the Nearest Neighbor baseline.
In our ablation study over metrics for Nearest Neighbor retrieval~(cf.~\cref{fig:metrics_ablation}), which was conducted concurrently, we still used learning rate $2\mathrm{e}{-5}$ of \cite{hardt2023test}.\looseness=-1

\paragraph{Low-Rank Adaptation (LoRA)}
We use LoRA~\citep{hu2021lora} for fine-tuning Phi-3, and also evaluate the performance of LoRA with GPT-2 and GPT-2-large~(cf.~\sref{sec:analysis_test_time_fine_tuning}).
We use LoRAs with rank~$64$, output scaling~$16$, without dropout and bias.
When fine-tuning with LoRA, we use the learning rate $5\mathrm{e}{-4}$.\looseness=-1

\paragraph{Gradient Checkpointing}
We additionally use gradient checkpointing~\citep{chen2016training} for fine-tuning Phi-3 to reduce memory footprint and allow fine-tuning on our hardware.\looseness=-1

\paragraph{Uncopyrighted Pile Dataset}
We use only those datasets of the Pile where our use is in compliance with the terms of service of the data host~\citep{gao2020pile}.
This excludes the Books3, BookCorpus2, OpenSubtitles, YTSubtitles, and OWT2 datasets.

We provide an overview of all hyperparameters of test-time fine-tuning in \cref{table:hyperparameters}.

\begin{table}[H]
  \centering
  \begin{tabular}{lccc}
    \toprule
    \textbf{Model family} & \textbf{GPT-2} & \textbf{Phi-3} & \textbf{Llama-3.2} \\[1pt]
    \hline \\[-6pt]
    $\lambda'$ & $0.01$ & $0.01$ & $0.01$ \\
    Learning rate & $5\mathrm{e}{-5}$ & $5\mathrm{e}{-4}$ & $1\mathrm{e}{-4}$ \\
    Adam's $\epsilon$-value & $1\mathrm{e}{-8}$ & $1\mathrm{e}{-8}$ & $1\mathrm{e}{-8}$ \\
    Max.\ sequence length (in tokens) & $1024$ & $4096$ & $4096$ \\
    LoRA & no & yes & yes \\
    Gradient checkpointing & no & yes & yes \\
    \bottomrule
  \end{tabular}
  \caption{Hyperparameters during test-time fine-tuning, unless noted otherwise.}
  \label{table:hyperparameters}
\end{table}

\clearpage
\subsection{Properties of the Pile Dataset}

\Cref{fig:cosine_similarity} shows the average cosine similarities of test prompts to neighbors in the data space of the Pile.
\Cref{table:dataset_weights} shows the weight of each dataset in the Pile.

\begin{figure}[H]
  \centering
  \incplt[0.6\textwidth]{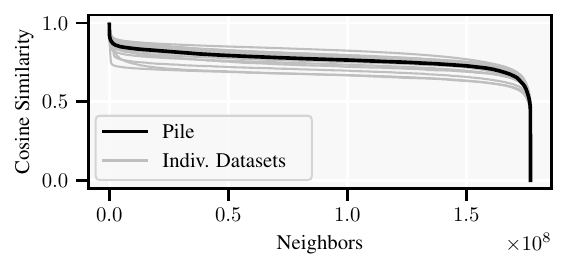}
  \vspace{-15pt}
  \caption{Average cosine similarities of test prompts to neighbors in the data space of the Pile; with the Roberta embedding model.}
  \label{fig:cosine_similarity}
\end{figure}

\begin{table}[H]
  \centering
  \begin{tabular}{lr}
    \toprule
    & \textbf{Weight} \\[1pt]
    \hline \\[-6pt]
    Common Crawl & 24.14\% \\[2pt]
    PubMed Central & 19.19\% \\[2pt]
    ArXiv & 11.94\% \\[2pt]
    GitHub & 10.12\% \\[2pt]
    FreeLaw & 8.18\% \\[2pt]
    Stack Exchange & 6.84\% \\[2pt]
    US Patents & 4.87\% \\[2pt]
    PubMed Abstracts & 4.09\% \\[2pt]
    Project Gutenberg & 2.89\% \\[2pt]
    Wikipedia & 2.04\% \\[2pt]
    DeepMind Math & 1.65\% \\[2pt]
    Ubuntu IRC & 1.17\% \\[2pt]
    EuroParl & 0.97\% \\[2pt]
    Hacker News & 0.83\% \\[2pt]
    PhilPapers & 0.51\% \\[2pt]
    NIH ExPorter Grants & 0.40\% \\[2pt]
    Enron Emails & 0.19\% \\
    \bottomrule
  \end{tabular}
  \caption{Overview of datasets in the (uncopyrighted) Pile. Weight is the percentage of bytes in the final dataset occupied by each dataset. Numbers are taken from~\cite{gao2020pile} and renormalized.}
  \label{table:dataset_weights}
\end{table}

\subsection{In-Context Baseline}

In our evaluation of in-context learning, we use the following format to insert the selected data into the context of the model: We separate all retrieved token sequences with the string \texttt{"{\textbackslash}n{\textbackslash}n"} which can be seen as a paragraph separator, and additionally add this string between the data string and the prompt.\looseness=-1

Notably, our results with in-context learning on GPT-2-large outperform the results previously reported by \cite{hardt2023test}.
We suspect that this is due to a combination of a more reasonable evaluation and using \method as opposed to Nearest Neighbor retrieval for data selection.\looseness=-1

\paragraph{Evaluation of Inference Cost of In-Context Baseline}
We estimate the inference cost of in-context learning as follows.
We evaluate the time it takes compute the rolling log-likelihood of the test instance with context included and subtract the time it takes to compute the rolling log-likelihood of the test instance without context.
This is a lower-bound of the inference cost of in-context learning, as unlike autoregressive generation, computing the log-likelihood is partially parallelized.

To compute the token throughput of the in-context baseline, we divide the total compute time by the number of tokens added to the context.

\subsection{Inference Cost with Test-Time Fine-Tuning}

\Cref{fig:training_time} evaluates the inference cost of test-time fine-tuning on all the Pile and the largest datasets.\looseness=-1

\begin{figure}[H]
  \incplt[\textwidth]{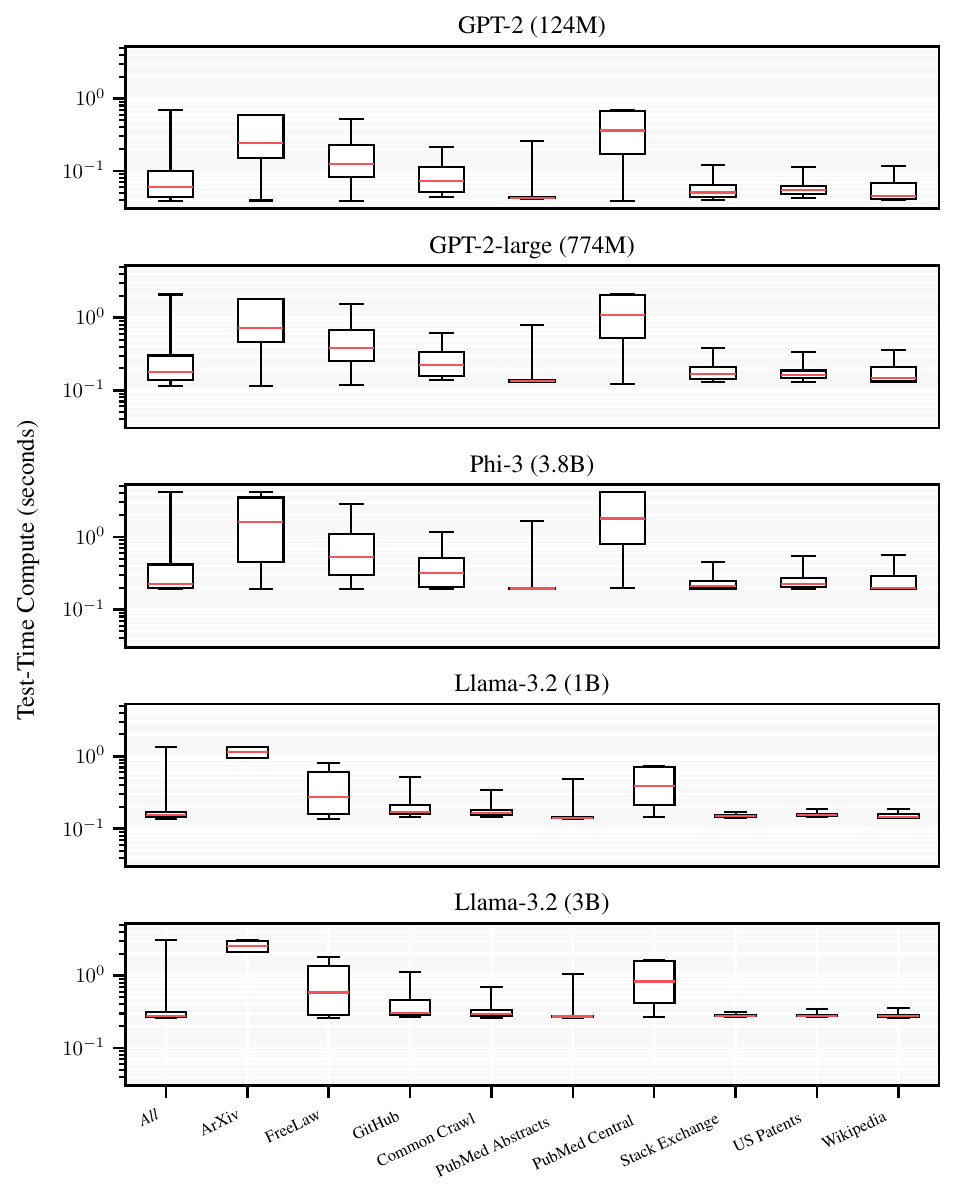}
  \vspace{-10pt}
  \caption{Cost of taking a single gradient step. Results are with an NVIDIA GH200.}
  \label{fig:training_time}
\end{figure}

\clearpage\section{Ablations}\label{sec:additional_results}

This section summarizes ablations that we conducted to investigate test-time fine-tuning and \method.
\begin{itemize}
  \item \textbf{Hyperparameter $\lambda'$}: \cref{table:lambda_results}
  \item \textbf{Learning Curves for Individual Datasets of the Pile}: \cref{fig:main_some_datasets}
  \item \textbf{Learning Rate}: \cref{fig:learning_rate_ablation}
  \item \textbf{Uncertainty Estimation}: \begin{itemize}
    \item Summary of correlations~(\cref{table:uncertainty_estimation})
    \item Visualization of $\sigma_n$~(\cref{fig:uncertainty_estimation_unscaled})
  \end{itemize}
  \item \textbf{Compute-proportional Performance Gain}: \begin{itemize}
    \item Details on \adamethod~(\cref{fig:early_stopping_details})
  \end{itemize}
\end{itemize}

\begin{landscape}
\begin{table}
  \centering
  \small
  \adjustbox{center}{%
  \begin{tabular}{S@{\hspace{2mm}}c@{\hspace{2mm}}c@{\hspace{2mm}}c@{\hspace{2mm}}c@{\hspace{2mm}}c@{\hspace{2mm}}c@{\hspace{2mm}}c@{\hspace{2mm}}c@{\hspace{2mm}}c@{\hspace{2mm}}c@{\hspace{2mm}}c@{\hspace{2mm}}c@{\hspace{2mm}}r}
    \toprule
    & $1\mathrm{e}{-12}$ & $1\mathrm{e}{-8}$ & $1\mathrm{e}{-4}$ & $0.01$ & $0.1$ & $1$ & $10$ & $100$ & $10`000$ &\raisebox{-1pt}{\vline height 7pt}& \textbf{NN} & \textbf{NN-F} & $\Delta$ \\[1pt]
    \hline \\[-6pt]
    NIH Grants & 123.9\reso{6.9} & \underline{79.0}\reso{6.4} & \underline{70.2}\reso{6.7} & \underline{\textbf{53.8}}\reso{8.9} & \underline{\textbf{52.9}}\reso{9.0} & \underline{\textbf{53.3}}\reso{9.1} & \underline{\textbf{54.2}}\reso{9.1} & \underline{\textbf{64.5}}\reso{10.9} & \underline{93.5}\reso{16.9} &\vline& 84.9\reso{2.1} & 91.6\reso{16.7} & \downres{32.0} \\[2pt]
    US Patents & 119.9\reso{3.9} & \underline{82.9}\reso{2.7} & \underline{70.2}\reso{3.1} & \underline{\textbf{62.9}}\reso{3.5} & \underline{\textbf{62.2}}\reso{3.6} & \underline{\textbf{62.7}}\reso{3.7} & \underline{\textbf{63.2}}\reso{3.7} & \underline{72.9}\reso{4.2} & 105.4\reso{6.4} &\vline& 80.3\reso{1.9} & 108.8\reso{6.6} & \downres{18.1} \\[2pt]
    GitHub & 54.6\reso{3.1} & \underline{41.4}\reso{2.2} & \underline{35.9}\reso{2.3} & \underline{\textbf{30.0}}\reso{2.2} & \underline{\textbf{28.6}}\reso{2.2} & \underline{\textbf{28.6}}\reso{2.2} & \underline{\textbf{29.2}}\reso{2.2} & \underline{36.1}\reso{2.6} & 51.3\reso{4.0} &\vline& 42.1\reso{2.0} & 53.2\reso{4.0} & \downres{13.5} \\[2pt]
    Enron Emails & \underline{87.1}\reso{16.5} & \underline{\textbf{68.6}}\reso{9.4} & \underline{\textbf{63.1}}\reso{9.1} & \underline{\textbf{53.1}}\reso{11.4} & \underline{\textbf{52.4}}\reso{11.8} & \underline{\textbf{53.8}}\reso{12.2} & \underline{\textbf{54.1}}\reso{12.2} & \underline{\textbf{59.6}}\reso{13.4} & \underline{89.4}\reso{20.4} &\vline& \textbf{64.4}\reso{10.1} & 91.6\reso{20.6} & \downres{12.0} \\[2pt]
    Common Crawl & 117.9\reso{1.3} & \underline{91.0}\reso{0.5} & \underline{90.7}\reso{0.5} & \underline{\textbf{87.5}}\reso{0.7} & \underline{\textbf{86.1}}\reso{0.9} & \underline{\textbf{87.8}}\reso{0.9} & \underline{88.3}\reso{0.9} & 99.3\reso{1.0} & 146.2\reso{1.6} &\vline& 90.4\reso{0.5} & 148.8\reso{1.5} & \downres{4.3} \\[2pt]
    ArXiv & 145.9\reso{7.0} & \underline{\textbf{83.5}}\reso{1.3} & \underline{\textbf{83.6}}\reso{1.3} & \underline{\textbf{82.5}}\reso{1.4} & \underline{\textbf{81.6}}\reso{1.9} & \underline{\textbf{81.2}}\reso{1.8} & \underline{\textbf{82.8}}\reso{1.9} & 94.6\reso{2.8} & 158.0\reso{6.1} &\vline& 85.0\reso{1.6} & 166.8\reso{6.4} & \downres{3.8} \\[2pt]
    Wikipedia & 104.2\reso{3.0} & \underline{\textbf{64.9}}\reso{2.1} & \underline{\textbf{63.9}}\reso{2.2} & \underline{\textbf{62.7}}\reso{2.1} & \underline{\textbf{63.7}}\reso{2.1} & \underline{\textbf{64.8}}\reso{2.2} & \underline{\textbf{65.6}}\reso{2.3} & 77.5\reso{2.5} & 118.1\reso{3.7} &\vline& \textbf{66.3}\reso{2.0} & 121.2\reso{3.5} & \downres{3.6} \\[2pt]
    PubMed Abstr. & 132.3\reso{1.6} & \underline{87.0}\reso{0.4} & \underline{87.0}\reso{0.4} & \underline{\textbf{84.4}}\reso{0.6} & \underline{\textbf{84.8}}\reso{0.7} & \underline{86.4}\reso{0.7} & \underline{86.7}\reso{0.7} & 102.0\reso{0.9} & 158.9\reso{1.4} &\vline& 87.2\reso{0.4} & 162.6\reso{1.3} & \downres{2.8} \\[2pt]
    PubMed Central & 131.9\reso{4.9} & \underline{\textbf{80.5}}\reso{2.5} & \underline{\textbf{80.0}}\reso{2.7} & \underline{\textbf{79.5}}\reso{2.6} & \underline{\textbf{80.6}}\reso{2.7} & \underline{\textbf{82.0}}\reso{2.7} & \underline{\textbf{83.8}}\reso{2.9} & 98.6\reso{3.7} & 151.6\reso{5.5} &\vline& \textbf{81.7}\reso{2.6} & 155.6\reso{5.1} & \downres{2.2} \\[2pt]
    Stack Exchange & 118.0\reso{1.7} & \underline{\textbf{77.6}}\reso{0.7} & \underline{\textbf{77.6}}\reso{0.7} & \underline{\textbf{76.7}}\reso{0.7} & \underline{\textbf{77.0}}\reso{0.7} & \underline{\textbf{77.8}}\reso{0.7} & \underline{78.1}\reso{0.7} & 85.9\reso{0.9} & 136.9\reso{1.6} &\vline& 78.2\reso{0.7} & 141.9\reso{1.5} & \downres{1.5} \\[2pt]
    Hacker News & 113.9\reso{7.2} & \underline{\textbf{78.8}}\reso{2.7} & \underline{\textbf{78.9}}\reso{2.7} & \underline{\textbf{78.4}}\reso{2.8} & \underline{\textbf{77.8}}\reso{3.5} & \underline{\textbf{78.1}}\reso{3.6} & \underline{\textbf{78.4}}\reso{3.6} & 86.2\reso{3.3} & 131.3\reso{6.2} &\vline& \textbf{79.2}\reso{2.8} & 133.1\reso{6.3} & \downres{1.4} \\[2pt]
    DeepMind Math & 104.7\reso{6.2} & \underline{\textbf{69.3}}\reso{2.1} & \underline{\textbf{69.1}}\reso{2.1} & \underline{\textbf{69.7}}\reso{2.1} & \underline{\textbf{70.1}}\reso{2.1} & \underline{\textbf{69.0}}\reso{2.0} & \underline{\textbf{70.1}}\reso{2.1} & \underline{\textbf{71.9}}\reso{2.2} & 103.5\reso{5.6} &\vline& \textbf{69.6}\reso{2.1} & 121.8\reso{3.1} & \downres{0.6} \\[2pt]
    FreeLaw & 102.5\reso{6.3} & \underline{\textbf{64.0}}\reso{3.9} & \underline{\textbf{63.5}}\reso{4.0} & \underline{\textbf{64.0}}\reso{4.1} & \underline{\textbf{65.5}}\reso{4.2} & \underline{\textbf{65.7}}\reso{4.1} & \underline{\textbf{67.0}}\reso{4.2} & 80.3\reso{5.0} & 114.1\reso{7.1} &\raisebox{-1pt}{\vline height 8pt}& \textbf{64.1}\reso{4.0} & 122.4\reso{7.1} & \downres{0.6} \\\\[-8pt]
    \hline \\[-8pt]
    \emph{All} & 112.9\reso{0.9} & \underline{78.5}\reso{0.6} & \underline{76.7}\reso{0.6} & \underline{\textbf{73.5}}\reso{0.6} & \underline{\textbf{73.2}}\reso{0.7} & \underline{\textbf{74.3}}\reso{0.7} & \underline{74.9}\reso{0.7} & 85.6\reso{0.8} & 129.8\reso{1.2} &\raisebox{-1pt}{\vline height 7pt}& 78.3\reso{0.5} & 133.3\reso{1.2} & \downres{5.4} \\
    \bottomrule
  \end{tabular}}
  \caption{Percentage of bits per byte after $50$ test-time iterations for varying $\lambda'$, relative to the bits per byte of the base model. We only include datasets with at least $10$ examples in our test set. \textbf{Bold} numbers denote the best performing selected subset. \underline{Underlined} numbers denote better or on-par performance with Nearest Neighbor retrieval. $\Delta$ denotes the performance gain of \method with the strongest $\lambda'$ \emph{per dataset} over Nearest Neighbor retrieval.
  Numbers in partentheses are standard errors.
  We remark that $\lambda'$ is on a logarithmic scale.
  For any choice of $\lambda' \in [1\mathrm{e-}8, 10]$, \method \emph{always} performs at least on-par with Nearest Neighbor retrieval.}
  \label{table:lambda_results}
\end{table}
\end{landscape}

\begin{figure}[H]
  \incplt[\textwidth]{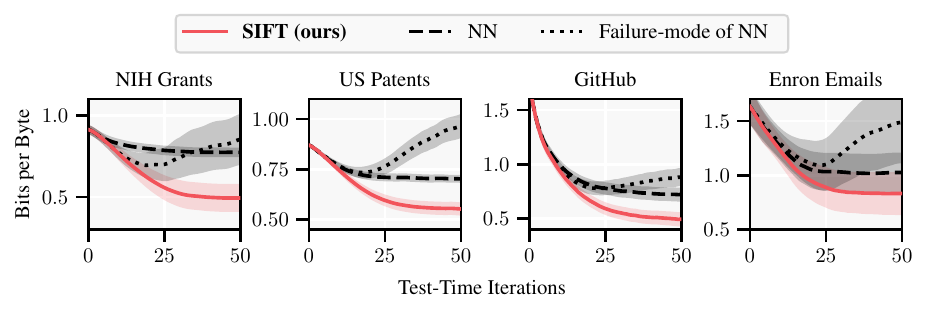}
  \vspace{-10pt}
  \caption{Performance in some of the datasets of the Pile, with GPT-2 as base model. Error bars correspond to standard errors.}
  \label{fig:main_some_datasets}
\end{figure}

\begin{figure}[H]
  \incplt[\textwidth]{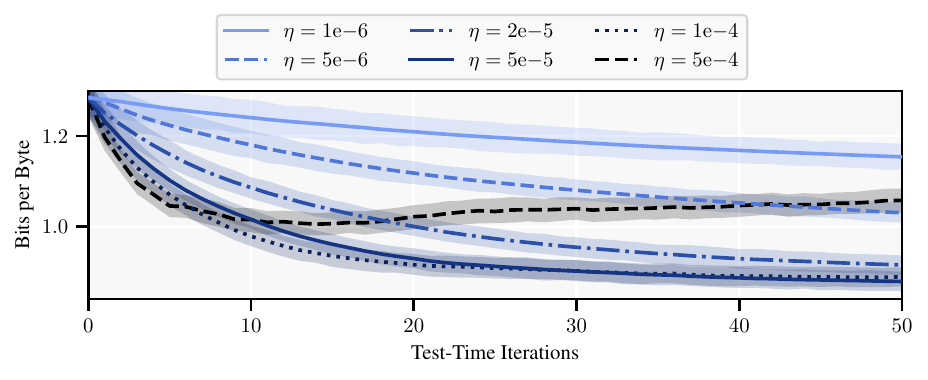}
  \caption{Ablation of the learning rate with data selected by Nearest Neighbor retrieval.
  We find that the default learning rate $5\mathrm{e}{-5}$ of the \texttt{transformers} library~\citep{wolf2020huggingface} works best, and conduct our other experiments with this learning rate unless noted otherwise.
  \cite{hardt2023test} had previously used $2\mathrm{e}{-5}$ which we find to be suboptimal.}
  \label{fig:learning_rate_ablation}
\end{figure}

\begin{table}
  \centering
  \begin{tabular}{llc@{\hspace{3mm}}c}
    \toprule
    && \textbf{Spearman} & \textbf{Pearson} \\[1pt]
    \hline \\[-6pt]
    \parbox[t]{8mm}{\multirow{2}{*}{$\sigma_n$}} & all steps & 0.485 & 0.421 \\[2pt]
    & final step & 0.496 & 0.443 \\[2pt]
    \hline \\[-7pt]
    \parbox[t]{8mm}{\multirow{2}{*}{$\hat{\sigma}_n$}} & all steps & 0.722 & 0.581 \\[2pt]
    & final step & 0.682 & 0.482 \\[2pt]
    \hline \\[-7pt]
    \parbox[t]{8mm}{\multirow{2}{*}{$\log \sigma_n$}} & all steps & 0.485 & 0.468 \\[2pt]
    & final step & 0.496 & 0.466 \\[2pt]
    \hline \\[-7pt]
    \parbox[t]{8mm}{\multirow{2}{*}{$\log \hat{\sigma}_n$}} & all steps & 0.722 & 0.618 \\[2pt]
    & final step & 0.682 & 0.526 \\[2pt]
    \bottomrule
  \end{tabular}
  \caption{We find a strong / moderate correlation between the uncertainty estimates $\hat{\sigma}_n$ / $\sigma_n$ and bits per byte.
  We further consider the correlation at all test-time iterations (from $0$ to $50$) as well as only at the final iteration.
  We report both the Spearman and Pearson correlation coefficients, measuring monotonic and linear relationships, respectively.
  Before determining the Pearson correlation, we exclude the 0.25\% of the data points with the lowest and highest uncertainty estimates to avoid the influence of outliers.
  The p-value of all correlations is below $1\mathrm{e}{-5}$ due to the large sample size.}
  \label{table:uncertainty_estimation}
\end{table}

\begin{figure}[H]
  \vspace{-5pt}
  \centering
  \incplt[0.62\textwidth]{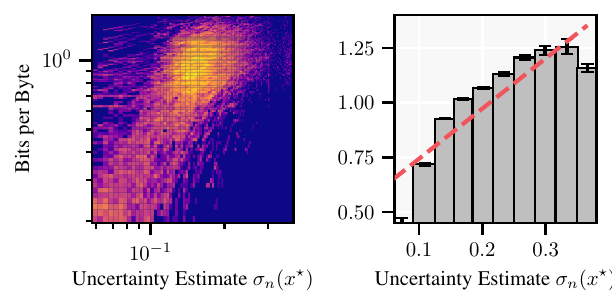}
  \vspace{-15pt}
  \caption{We visualize the predictive ability of the uncertainty estimates $\sigma_n$ analogously to \cref{fig:uncertainty_estimation}.}
  \label{fig:uncertainty_estimation_unscaled}
\end{figure}

\begin{figure}[H]
  \centering
  \incplt[0.8\textwidth]{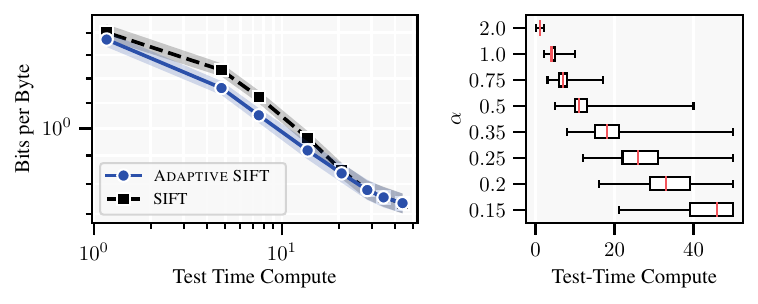}
  \caption{We evaluate \adamethod with the same choices of $\alpha$ as in \figref{fig:uncertainty_estimation}{right}.
  \textbf{Left:}~Bits per byte of \adamethod ($\downarrow$ better) against test-time compute.
  Every marker corresponds to the performance of \adamethod with a given $\alpha$, where the associated test-time compute is the average number of test-time iterations on prompts.
  We compare \adamethod to \method, where we spend the same test-time compute on every prompt.
  We see a slight advantage of \adamethod over \method, due to adaptively stopping depending on the prompt.
  Our current experiment exhibits a bias as test-time compute approaches $50$, since we force-stop the compute at $50$ iterations.
  This biases \adamethod to perform similarly to \method.
  We hypothesize that the initial advantage of \adamethod over \method may grow with more test-time compute if compute is not force-stopped at $50$ iterations.
  \textbf{Right:}~Frequency of stopping at a given iteration for given values of~$\alpha$.}
  \label{fig:early_stopping_details}
\end{figure}

\section{Proofs}\label{sec:proofs}

This section provides the formal proofs of the results presented in the main text.
\begin{itemize}
  \item \sref{sec:proofs:insufficiency_nn} proves the insufficiency of Nearest Neighbor retrieval~(\cref{prop:insufficiency_nn_informal}).
  \item \sref{sec:proofs_reg_loss_min_vs_ttft} shows the close relationship of regularized loss minimization and test-time fine-tuning~(\cref{prop:reg_loss_min_vs_ttft}).
  \item \sref{sec:proofs_balance} details how \method balances relevance and diversity~(\sref{sec:method_details:balancing_relevance_diversity}).
  \item \sref{sec:proofs_regression} states confidence sets for fine-tuning regression models that are analogous to the confidence sets for classification from the main text.
  \item \sref{sec:proofs_confidence_sets} derives the confidence sets from the main text~(\cref{thm:confidence_sets}).
\end{itemize}

\subsection{Notation}

Throughout this work, $\log$ denotes the natural logarithm.
Unless noted otherwise $\{\dots\,\!\}$ denotes a multiset.
We define the feature map $\mPhi_n \defeq (\vphi(\vx_1), \dots, \vphi(\vx_n)) \in \R^{n \times d}$, which gives rise to the kernel matrix $\smash{\mKsub{n} \defeq \mKsub{X_n} = \mPhi_n \transpose{\mPhi_n} \in \R^{n \times n}}$ and the covariance operator $\smash{\mSigma_n \defeq \transpose{\mPhi_n} \mPhi_n \in \R^{d \times d}}$.

\subsection{Insufficiency of Nearest Neighbor Retrieval (\cref{prop:insufficiency_nn_informal})}\label{sec:proofs:insufficiency_nn}

We refer to \sref{sec:method_details:uncertainty_vanishes} for the formal definition of the irreducible uncertainty~$\irred[\spD]{\prompt}$.

We remark that if embeddings are of unit length, the cosine similarity scoring function is equivalent to the (negative) Euclidean distance scoring function: \begin{align*}
  \norm{\prompt - \vx}_2^2 &= \transpose{(\prompt - \vx)}(\prompt - \vx) = \norm{\prompt}_2^2 + \norm{\vx}_2^2 - 2 \transpose{{\prompt}} \vx = 2 - 2 \cos(\prompt, \vx).
\end{align*}
We henceforth consider the Euclidean distance scoring function.

\begin{proposition}[Insufficiency of Nearest Neighbor Retrieval]\label{prop:insufficiency_nn}
  Suppose w.l.o.g.\ that~$\vphi(\vx) = \vx$.
  Consider the data space~$\spD = \bigcup_{i=1}^d \spD_i$ where $\spD_i = \{\ve_i \mid j \in \Nat\}$ with $\ve_i$ the $i$-th basis vector of~$\R^d$.\footnote{We remark that $\{\dots\,\!\}$ denotes a multiset.}
  Let $\prompt = \frac{1}{\sqrt{4 + (d-1)}} (2, 1, 1, \dots, 1) \in \R^d$.

  Then, for all $n \geq 1$: \begin{enumerate}
    \item If $X_n$ are the $n$ nearest neighbors of~$\prompt$ in~$\spD$, $\sigma_n^2(\prompt) \geq \irred*[\spD_1]{\prompt} \gg 0$.

    \item If $X_n$ is selected by \method, $\sigma_n^2(\prompt) \overset{n \to \infty}{\longrightarrow} \irred*[\spD]{\prompt} = 0$.
  \end{enumerate}
\end{proposition}

\begin{proof}
  \leavevmode\begin{enumerate}
    \item Clearly, $\norm{\prompt - \ve_1}_2^2 < \norm{\prompt - \ve_i}_2^2$ for all $i > 1$.
    Hence, $X_n = \{\ve_1 \mid i \in [n]\} \subset \spD_1$.
    This is as if the data space was restricted to $\spD_1$, and hence $\sigma_n^2(\prompt) \geq \irred*[\spD_1]{\prompt}$.\looseness=-1

    \item This follows readily from \cref{thm:convergence} and noting that $\spn{\spD} = \R^d$, implying $\irred*[\spD]{\prompt} = 0$.
  \end{enumerate}
\end{proof}

\paragraph{Discussion}

The setting examined in \cref{prop:insufficiency_nn} is an extreme case (where data exists with exact duplication), yet we deem that it illustrates a realistic scenario.
Particularly nowadays that similar information is accessible from many sources in different forms, it is crucial to explicitly select diverse data for fine-tuning.
We show here theoretically and in \cref{sec:examples:balancing_relevance_and_diversity} qualitatively that \method does not have this limitation.\looseness=-1

\subsection{The close relationship of Regularized Loss Minimization and Test-Time Fine-Tuning (\cref{prop:reg_loss_min_vs_ttft})}\label{sec:proofs_reg_loss_min_vs_ttft}

\begin{proof}

  We note that the regularized negative log-likelihood loss $\loss^\lambda$ from \cref{eq:reg_nll_loss}, \begin{align*}
    \loss^\lambda(\mW; D) &= \underbrace{- \sum_{(\vx,y) \in D} \log s_y(\mW \vphi(\vx))}_{\loss(\mW; D)} + \frac{\lambda}{2} \normF{\mW - \mWpre}^2,
  \end{align*} is strictly convex in $\mW$ and has a unique minimizer $\mWsub{\lambda}$ which satisfies \begin{align*}
    \grad \loss^\lambda(\mWsub{\lambda}; D) = \grad \loss(\mWsub{\lambda}; D) + \lambda(\mWsub{\lambda} - \mWpre) = \vzero.
  \end{align*}
  It follows that $\mWsub{\lambda} = \mWpre - \frac{1}{\lambda} \grad \loss(\mWsub{\lambda}; D)$.

  Finally, recall that $\mWhat_\eta = \mWpre - \eta \grad \loss(\mWpre; D)$.
  We obtain \begin{align*}
    \normF*{\mWsub{1/\eta} - \mWhat_{\!\eta}} &= \normF*{\eta \grad \loss(\mWpre; D) - \eta \grad \loss(\mWsub{1/\eta}; D)} \\
    &= \eta \normF*{\grad \loss(\mWpre; D) - \grad \loss(\mWsub{1/\eta}; D)}.
  \end{align*}
\end{proof}

\subsection{How \method Balances Relevance and Diversity}\label{sec:proofs_balance}

\paragraph{1st point}

For non-unit length embeddings, the first selected point can be expressed as follows: \begin{align*}
  \vx_1 = \argmin_{\vx \in \spD} \sigma_{\{\vx\}}^2(\prompt) = \argmax_{\vx \in \spD} \frac{(\transpose{\vphi(\prompt)} \vphi(\vx))^2}{\norm{\vphi(\vx)}_2^2 + \lambda'} = \argmax_{\vx \in \spD}
  \begin{cases}
    \measuredangle_{\vphi}(\prompt, \vx)^2 & \text{as $\lambda' \to 0$} \\
    (\transpose{\vphi(\prompt)} \vphi(\vx))^2 & \text{as $\lambda' \to \infty$}.
  \end{cases}
\end{align*}

\paragraph{2nd point}

Next, we consider the second selected point.
We derive the results in terms of the dot product kernel $k(\vx,\vxp) = \transpose{\vphi(\vx)} \vxp$ which is such that $k(\vx,\vxp) = \measuredangle_{\vphi}(\vx, \vxp)$ for unit length embeddings.
Let $\vx$ be such that $k(\vx_1, \vx) = 0$.
We have \begin{align*}
  \psi_{\prompt}(\{\vx_1,\vx_1\}) &= \transpose{\begin{bmatrix}
    k(\prompt, \vx_1) \\
    k(\prompt, \vx_1)
  \end{bmatrix}} \inv{\begin{bmatrix}
    1 + \lambda' & 1 \\
    1 & 1 + \lambda'
  \end{bmatrix}} \begin{bmatrix}
    k(\prompt, \vx_1) \\
    k(\prompt, \vx_1)
  \end{bmatrix} \\
  &= \frac{1}{(1+\lambda')^2 - 1} \transpose{\begin{bmatrix}
    k(\prompt, \vx_1) \\
    k(\prompt, \vx_1)
  \end{bmatrix}} \begin{bmatrix}
    1 + \lambda' & -1 \\
    -1 & 1 + \lambda'
  \end{bmatrix} \begin{bmatrix}
    k(\prompt, \vx_1) \\
    k(\prompt, \vx_1)
  \end{bmatrix} \\
  &= \frac{2 \lambda' k(\prompt, \vx_1)^2}{(1+\lambda')^2 - 1} \\
  &= \frac{2 k(\prompt, \vx_1)^2}{2 + \lambda'}.
\end{align*}
For $\vx$, we have \begin{align*}
  \psi_{\prompt}(\{\vx_1,\vx\}) &= \transpose{\begin{bmatrix}
    k(\prompt, \vx_1) \\
    k(\prompt, \vx)
  \end{bmatrix}} \inv{\begin{bmatrix}
    1 + \lambda' & 0 \\
    0 & 1 + \lambda'
  \end{bmatrix}} \begin{bmatrix}
    k(\prompt, \vx_1) \\
    k(\prompt, \vx)
  \end{bmatrix} \\
  &= \frac{1}{(1+\lambda')^2} \transpose{\begin{bmatrix}
    k(\prompt, \vx_1) \\
    k(\prompt, \vx)
  \end{bmatrix}} \begin{bmatrix}
    1 + \lambda' & 0 \\
    0 & 1 + \lambda'
  \end{bmatrix} \begin{bmatrix}
    k(\prompt, \vx_1) \\
    k(\prompt, \vx)
  \end{bmatrix} \\
  &= \frac{k(\prompt, \vx_1)^2 + k(\prompt, \vx)^2}{1 + \lambda'}.
\end{align*}
We see that $\vx$ is preferred over $\prompt$ if and only if \begin{align*}
  \frac{k(\prompt, \vx_1)^2 + k(\prompt, \vx)^2}{1 + \lambda'} > \frac{2 k(\prompt, \vx_1)^2}{2 + \lambda'} \quad\iff\quad
  k(\prompt, \vx)^2 > \underbrace{\frac{\lambda'}{2 + \lambda'}}_{c(\lambda')} k(\prompt, \vx_1)^2.
\end{align*}
As $\lambda' \to \infty$, $c(\lambda') \to 1$; whereas as $\lambda' \to 0$, $c(\lambda') \to 0$.

We interpret the expressions extensively in \cref{sec:method}.

\subsection{Confidence Sets for Regression}\label{sec:proofs_regression}

Before moving on to deriving confidence sets for the setting with categorical feedback, we state analogous results for the regression setting under the following standard assumptions.
For ease of notation, we consider the scalar case.\looseness=-1

\begin{assumption}[Linear function in a known latent space]\label{asm:linear_reg}
  We assume $f^\star(\vx) = \vphi(\vx)^\top \opt{\vw}$ with ${\vw^\star \in \R^{d}}$ and where ${\vphi(\cdot) \in \R^d}$ denotes known embeddings.
  We assume that $\opt{\vw}$ has bounded norm, i.e., $\norm{\opt{\vw} - \vwpre}_2 \leq B$ for some finite $B \in \R$.
\end{assumption}

\begin{assumption}[Sub-Gaussian Noise]\label{asm:sub_gaussian_noise}
  We assume that the data follows \begin{align*}
    y_n = \opt{f}(\vx_n) + \varepsilon_n
  \end{align*} where each $\varepsilon_{n}$ from the noise sequence $\{\varepsilon_n\}_{n=1}^\infty$ is conditionally zero-mean $\rho$-sub-Gaussian with known constant ${\rho > 0}$.
  Formally, \begin{align*}
    \forall n \geq 1, \lambda \in \R : \quad \E{e^{\lambda \epsilon_{n}}}[D_{n-1}] \leq \exp\parentheses*{\frac{\lambda^2 \rho^2}{2}}
  \end{align*} where $D_{n-1}$ corresponds to the $\sigma$-algebra generated by the random variables $\{\vx_i,\epsilon_i\}_{i=1}^{n-1}$ and $\vx_n$.
\end{assumption}

We consider the standard squared loss $\loss(\vw; D) \defeq \frac{1}{2} \sum_{(\vx,y) \in D} (f(\vx; \vw) - y)^2$ where we write $f(\vx; \vw) \defeq \vphi(\vw)^\top \vw$.
The regularized loss with minimizer $\vw_n$ is then
\begin{align}
  \loss^\lambda(\vw; D_n) \defeq \loss(\vw; D_n) + \frac{\lambda}{2} \norm{f - \vwpre}_2^2 \label{eq:reg_nll_loss_reg}
\end{align} where $\lambda > 0$ is the regularization parameter.
In the following, we write $f_n(\vx) \defeq f(\vx; \vw_n)$ and $\fpre(\vx) \defeq f(\vx; \vwpre)$.
The closed-form solution to the optimization problem from \cref{eq:reg_nll_loss_reg} is well-known~\citep[see, e.g.,][Section 6.2.2]{williams2006gaussian} to be \begin{align*}
  f_n(\vx) = \fpre(\vx) + \transpose{\vk_{X_n}}(\vx) \inv{(\mKsub{X_n} + \lambda\mI_n)} (\vy_{n} - \vfpre_{n})
\end{align*} where $\vfpre_n$ is the vector of predictions of $\fpre$ at $X_n$ and $\vy_n$ is the vector of observations in $D_n$.

The below result is an almost immediate consequence of the results of \cite{abbasi2013online} and \cite{chowdhury2017kernelized}.

\begin{theorem}[Confidence Sets for Regression]\label{thm:confidence_sets_regression}
  Pick ${\delta \in (0,1)}$ and let \cref{asm:linear_reg,asm:sub_gaussian_noise} hold.
  Let \begin{align*}
    \beta_{n}(\delta) \defeq B + \rho \sqrt{2(\gamma_{n} + 1 + \log(1 / \delta))}
  \end{align*} where $\gamma_n \defeq \max_{\vx_1, \dots, \vx_n} \frac{1}{2} \log \det{(\mI_n + \inv{\lambda} \mK_{X_n})}$.
  Then \begin{align*}
    \Pr{\forall n \geq 1, \vx \in \spX : |\opt{f}(\vx) - f_{n}(\vx)| \leq \beta_{n}(\delta) \sigma_{n}(\vx)} \geq 1 - \delta.
  \end{align*}
\end{theorem}

\begin{proof}
  Let us define the \emph{residual} of the ground truth and pre-trained model as ${\opt{\tilde{f}}(\vx) \defeq \opt{f}(\vx) - \fpre(\vx)}$ with corresponding weight vector $\tilde{\vw}$.
  Analogously, let $\tilde{y}_n = \opt{\tilde{f}}(\vx_n) + \varepsilon_n$ be the observed error.
  We have that $\opt{\tilde{\vw}} \defeq \opt{\vw} - \vwpre \in \R^d$ with norm $\norm{\opt{\vw} - \vwpre}_k$.
  The unbiased estimate of the remaining error is \begin{align*}
    \tilde{f}_n = \transpose{\vk_{X_n}}(\vx) \inv{(\mKsub{X_n} + \lambda\mI_n)} \tilde{\vy}_{n}.
  \end{align*}
  By Theorem 2 of \cite{chowdhury2017kernelized}, for all $\vx \in \spX$ and $n \geq 1$, jointly with probability at least $1-\delta$, ${|\opt{\tilde{f}}(\vx) - \tilde{f}_{n}(\vx)| \leq \beta_n(\delta) \sigma_n(\vx)}$.
  It remains now only to observe that \begin{align*}
    |\opt{\tilde{f}}(\vx) - \tilde{f}_{n}(\vx)| = |\opt{f}(\vx) - f_{n}(\vx)|.
  \end{align*}
\end{proof}

\subsection{Confidence Sets for Classification (\cref{thm:confidence_sets})}\label{sec:proofs_confidence_sets}

We begin by re-stating Corollary 1 of \cite{amani2021ucb}.
Analogous results can be obtained from Theorem 1 of \cite{zhang2024online}.
Substantial work has studied the special case of binary feedback, $K = 2$~\cite{faury2020improved,pasztor2024bandits}.

Let $\mA(\vx; \mW) \in \R^{K \times K}$ be the matrix satisfying ${(\mA(\vx; \mW))_{i,j} \defeq s_i(\vx; \mW)(\Ind{i = j} - s_j(\vx; \mW))}$.
Equivalently, ${\mA(\vx; \mW) = \diag{\vs(\vx; \mW)} - \vs(\vx; \mW) \transpose{\vs(\vx; \mW)}}$.
Based on this matrix, we define ${L \defeq \sup_{\vx \in \spX, \mW \in \spW_B} \lambda_{\max}(\mA(\vx; \mW))}$ and ${\kappa \defeq \sup_{\vx \in \spX, \mW \in \spW_B} 1/\lambda_{\min}(\mA(\vx; \mW))}$.

\begin{lemma}[Corollary 1 of \cite{amani2021ucb}]\label{lem:confidence_sets}
  Assume $\mW^\star \in \spW_B$ and $\mWpre = \mzero$.
  Let $\delta \in (0,1)$ and set \begin{align}
    \tilde{\beta}_n(\delta) \defeq \sqrt{\lambda} \parentheses*{B + \frac{1}{2 \sqrt{K}}} + \frac{2 K^{3/2} d}{\sqrt{\lambda}} \log\parentheses*{\frac{2}{\delta}\sqrt{1 + \frac{n}{d \lambda}}}.
  \end{align}
  Then, \begin{align*}
    \Pr{\forall n \geq 1, \vx \in \spX : \norm*{\vs_n(\vx) - \vs^\star(\vx)}_2 \leq 2 L \tilde{\beta}_n(\delta) \sqrt{\kappa(1 + 2 B)} \norm{\vphi(\vx)}_{\inv{\mV_n}}} \geq 1 - \delta,
  \end{align*} where $\mVsub{n} \defeq \mSigma_n + \kappa\lambda \mI_d$.
\end{lemma}

Our result follows from two auxiliary lemmas.

\begin{lemma}\label{lem:tv_dist_bound}
  For any $\vs, \vsp \in \R^K$, $\dTV{\vs, \vsp} \leq \frac{\sqrt{K}}{2} \norm{\vs - \vsp}_2$.
\end{lemma}
\begin{proof}
  We have \begin{align*}
    \dTV{\vs, \vsp} &= \frac{1}{2} \norm{\vs - \vsp}_1 = \frac{1}{2} \sum_{i=1}^K |s_i - s'_i| \leq \frac{1}{2} \sqrt{K} \sqrt{\sum_{i=1}^K (s_i - s'_i)^2} = \frac{\sqrt{K}}{2} \norm{\vs - \vsp}_2
  \end{align*} where the inequality follows from Cauchy-Schwarz.
\end{proof}

The following lemma is a standard result in the literature~\citep{srinivas2009gaussian,chowdhury2017kernelized,pasztor2024bandits}, which we include here for completeness.\looseness=-1

\begin{lemma}\label{lem:kernel_helper_lemma}
  Let $\sigma_n$ be as defined in \cref{eq:variance}.
  Then, $\sqrt{\kappa \lambda} \norm{\vphi(\vx)}_{\inv{\mVsub{n}}} = \sigma_n(\vx)$ for any $\vx \in \spX$.
\end{lemma}
\begin{proof}
  Note that $(\mSigma_n + \kappa\lambda \mI_d) \transpose{\mPhi_n} = \transpose{\mPhi_n} (\mKsub{n} + \kappa\lambda \mI_n)$ which implies \begin{align}
    \inv{(\mSigma_n + \kappa\lambda \mI_d)} \transpose{\mPhi_n} = \transpose{\mPhi_n} \inv{(\mKsub{n} + \kappa\lambda \mI_n)}. \label{eq:kernel_helper_lemma_helper1}
  \end{align}
  Further, by definition of $\vk_n$, $\vk_n(\vx) = \mPhi_n \vphi(\vx)$ which permits writing \begin{align*}
    (\mSigma_n + \kappa\lambda \mI_d) \vphi(\vx) = \transpose{\mPhi_n} \vk_n(\vx) + \kappa\lambda \vphi(\vx)
  \end{align*} and implies \begin{align}
    \vphi(\vx) & = \inv{(\mSigma_n + \kappa\lambda \mI_d)} \transpose{\mPhi_n} \vk_n(\vx) + \kappa\lambda \inv{(\mSigma_n + \kappa\lambda \mI_d)} \vphi(\vx) \nonumber \\
    &\overset{(\ref{eq:kernel_helper_lemma_helper1})}{=} \transpose{\mPhi_n} \inv{(\mKsub{n} + \kappa\lambda \mI_n)} \vk_n(\vx) + \kappa\lambda \inv{(\mSigma_n + \kappa\lambda \mI_d)} \vphi(\vx) \label{eq:kernel_helper_lemma_helper2}
  \end{align}
  We have \begin{align*}
    k(\vx, \vx) &= \transpose{\vphi(\vx)} \vphi(\vx) \\
    &\overset{(\ref{eq:kernel_helper_lemma_helper2})}{=} \transpose{\parentheses*{\transpose{\mPhi_n} \inv{(\mKsub{n} + \kappa\lambda \mI_n)} \vk_n(\vx) + \kappa\lambda \inv{(\mSigma_n + \kappa\lambda \mI_d)} \vphi(\vx)}} \vphi(\vx) \\
    &= \transpose{\vk_n(\vx)} \inv{(\mKsub{n} + \kappa\lambda \mI_n)} \vk_n(\vx) + \kappa\lambda \transpose{\vphi(\vx)} \inv{(\mSigma_n + \kappa\lambda \mI_d)} \vphi(\vx) \\
    &= \transpose{\vk_n(\vx)} \inv{(\mKsub{n} + \kappa\lambda \mI_n)} \vk_n(\vx) + \kappa\lambda \transpose{\vphi(\vx)} \inv{\mVsub{n}} \vphi(\vx).
  \end{align*}
  Reordering this equation, we obtain \begin{align*}
    \kappa\lambda \norm{\vphi(\vx)}_{\inv{\mVsub{n}}}^2 = \kappa\lambda \transpose{\vphi(\vx)} \inv{\mVsub{n}} \vphi(\vx) = k(\vx, \vx) - \transpose{\vk_n(\vx)} \inv{(\mKsub{n} + \kappa\lambda \mI_n)} \vk_n(\vx) = \sigma_n^2(\vx),
  \end{align*} concluding the proof.
\end{proof}

We now proceed to prove a version of \cref{thm:confidence_sets} with $\mWpre = \mzero$.

\begin{theorem}\label{thm:confidence_sets_unbiased}
  Assume $\mW^\star \in \spW_B$ and $\mWpre = \mzero$.
  Let $\delta \in (0,1)$ and $\beta_n(\delta)$ as in \cref{eq:beta}.
  Then \begin{align*}
    \Pr{\forall n \geq 1, \vx \in \spX : \dTV*{\vs_n(\vx), \vs^\star(\vx)} \leq \beta_n(\delta) \cdot \sigma_n(\vx)} \geq 1 - \delta.
  \end{align*}
\end{theorem}
\begin{proof}
  We have \begin{align*}
    \dTV*{\vs_n(\vx), \vs^\star(\vx)} &\leq \frac{\sqrt{K}}{2} \norm{\vs_n(\vx) - \vs^\star(\vx)}_2 \tag{\cref{lem:tv_dist_bound}} \\
    \overset{\text{w.h.p.}}&{\leq} L \tilde{\beta}_n(\delta) \sqrt{K \kappa (1 + 2 B)} \norm{\vphi(\vx)}_{\inv{\mVsub{n}}} \tag{\cref{lem:confidence_sets}} \\
    &= L \tilde{\beta}_n(\delta) \sqrt{\frac{K (1 + 2 B)}{\lambda}} \sigma_n(\vx). \tag{\cref{lem:kernel_helper_lemma}}
  \end{align*}
  It remains to note that \begin{align*}
    L \tilde{\beta}_n(\delta) \sqrt{\frac{K (1 + 2 B)}{\lambda}} &= L \sqrt{K (1 + 2 B)} \parentheses*{B + \frac{1}{2 \sqrt{K}}} + \frac{2 L K^2 d \sqrt{1 + 2 B}}{\lambda} \log\parentheses*{\frac{2}{\delta}\sqrt{1 + \frac{n}{d \lambda}}} \\
    &\leq 2 \sqrt{K (1 + 2 B)} \brackets*{B + \frac{L K^{3/2} d}{\lambda} \log\parentheses*{\frac{2}{\delta}\sqrt{1 + \frac{n}{d \lambda}}}} = \beta_n(\delta).
  \end{align*}
\end{proof}

With this we are ready to prove \cref{thm:confidence_sets}.

\begin{proof}[Proof of \cref{thm:confidence_sets}]
  We will proceed analogously to the proof of \cref{thm:confidence_sets_regression}.
  That is, our objective will be to bound the deviation of our biased model, which we refer to as $\mWsub{n} = \argmin_{\mW \in \spW_B} \loss^\lambda(\mW; D_n)$, to $\mW^\star$.
  Let \begin{align*}
    \tilde{\loss}(\mW'; D) \defeq -\sum_{(\vx, y) \in D} \log s_y((\mW' + \mWpre) \vphi(\vx)) \quad\text{and}\quad \tilde{\loss}^\lambda(\mW'; D) \defeq \tilde{\loss}(\mW'; D) + \frac{\lambda}{2} \normF{\mW'}^2
  \end{align*} with minimizer $\mWsub{n}' \defeq \argmin_{\mW' : \normF{\mW'} \leq B} \tilde{\loss}^\lambda(\mW'; D_n)$.
  We further define the residual weights $\smash{{\tilde{\mW}^\star} \defeq \mW^\star - \mWpre}$.

  Next, we make the following observation:
  In their proof of \cref{lem:confidence_sets}, \cite{amani2021ucb} bound \begin{align}
    \norm*{\vs(\vf(\vx; \mWsub{n}')) - \vs(\vf(\vx; \tilde{\mW}^\star))}_2 \leq \const \cdot \norm*{\mathrm{vec}(\tilde{\mW}^\star) - \mathrm{vec}(\mWsub{n}')}_{\tilde{\mG}(\tilde{\mW}^\star, \mWsub{n}')} \label{eq:confidence_sets_helper1}
  \end{align} where $\const$ is independent of $\mW^\star, \mWpre, \mWsub{n}'$ and the matrix $\tilde{\mG}(\tilde{\mW}^\star, \mWsub{n}')$ is invariant to a change of variables, i.e., $\tilde{\mG}(\tilde{\mW}^\star, \mWsub{n}') = \mG(\mW^\star, \mWsub{n}' + \mWpre)$ with $\tilde{\mG}$ defined with respect to the loss $\tilde{\loss}^\lambda$ and $\mG$ defined with respect to the loss $\loss^\lambda$.
  \Cref{thm:confidence_sets_unbiased} applies to $\vs(\vf(\vx; \mWsub{n}'))$ and $\vs(\vf(\vx; \mW^\star - \mWpre))$ since the regularization of $\tilde{\loss}^\lambda$ is unbiased and the residual weights satisfy $\smash{\normF*{\tilde{\mW}^\star} = \normF{\mW^\star - \mWpre} \leq B}$ by assumption.

  Since ${\tilde{\mW}^\star - \mWsub{n}' = \mW^\star - (\mWsub{n}' + \mWpre)}$, the bounds of \cref{eq:confidence_sets_helper1} as well as \cref{thm:confidence_sets_unbiased} then also apply to $\vs(\vf(\vx; \mWsub{n}' + \mWpre)), \vs(\vf(\vx; \mW^\star))$.
  Observing that ${\mWsub{n} = \mWsub{n}' + \mWpre}$ as a direct consequence of the change of variables completes the proof.
\end{proof}

\section{Qualitative Examples}\label{sec:examples}

\subsection{Balancing Relevance and Diversity}\label{sec:examples:balancing_relevance_and_diversity}

The following details the data space and prompt used in the qualitative example of \cref{fig:qualitative_example}.
We evaluate \method with $\lambda' = 1\mathrm{e}{-4}$ and normalized embeddings, using the same embedding model as in our main experiments.

\begin{table}[h!]
  \centering
  \begin{tabularx}{\textwidth}{lX}
    \toprule
    & \textbf{Prompt} \\
    \midrule \\[-9pt]
    & What is the age of Michael Jordan and how many kids does he have? \\[10pt]
    & \textbf{Data space} \\
    \midrule \\[-9pt]
    1 & Michael Jordan was born on February 17, 1963, in Brooklyn, New York. \\[2pt]
    2 & The age of Michael Jordan is 61 years. \\[2pt]
    \textcolor{red}{3} & \textcolor{red}{Michael Jordan has five children.} \\[2pt]
    \textcolor{red}{4} & \textcolor{red}{Michael Jordan has 5 kids.} \\[2pt]
    \bottomrule
  \end{tabularx}
  \caption{Query and information about Michael Jordan within data space}
  \label{table:synthetic_example:michael_jordan}
\end{table}

\subsection{Examples from the Pile}

The following provides examples of the data selected by \method for some queries from the Pile dataset.

\begin{examplebox}{DeepMind Math}{}
  {\bfseries Query}\\[3pt]
  Find the second derivative of -222966*l*s**2 + 152*l*s - 8111*l + s**2 + 2 wrt s.\\-445932*l + 2\\What is the third derivative of 175*s**5 - 5*s**4 - 6106*s**3 + 53*s**2 + 169*s - 1753?\\10500*s**2 - 120*s - 36636\\What is the third derivative of 23679631*b**5 - 2*b**3 + 8*b**2 + 2*b - 6771326 wrt b?\\1420777860*b**2 - 12\\Find the second derivative of 3263785*m**4 + 141*m + 11251.\\39165420*m**2\\What is the second derivative of -47089*k*z**3 - 30997*k*z + 59*z**2 + 295*z wrt z?\\\dots
  \medskip\hrule\medskip
  {\bfseries 1st example}\\[3pt]
  What is the second derivative of 333510825*p**3 - 292254*p + 96 wrt p?\\2001064950*p\\What is the third derivative of -2862429*f**5 - 5*f**2 + 439*f - 557?\\-171745740*f**2\\What is the derivative of 32081*i**4 + 10*i**3 - 2*i - 9371139?\\128324*i**3 + 30*i**2 - 2\\Find the third derivative of -439900344*z**5 - 675051939*z**2 wrt z.\\-26394020640*z**2\\\dots
  \medskip\hrule\medskip
  {\bfseries 2nd example}\\[3pt]
  What is the third derivative of 2322809*k**3 + 38*k**2 + 105*k + 236 wrt k?\\13936854\\What is the third derivative of 1242810*p**4 - 5*p**3 + 8382*p**2 + 491*p wrt p?\\29827440*p - 30\\Differentiate -23915071*o**4 + 25970708.\\-95660284*o**3\\Find the first derivative of -73333026*k - 218757639 wrt k.\\-73333026\\What is the second derivative of -9350*n**4 + 2047*n**2 - n - 42762066?\\-112200*n**2 + 4094\\\dots
\end{examplebox}

\begin{examplebox}{Enron Emails}{}
  {\bfseries Query}\\[3pt]
  Patti,\\\\What do I do with this now?  How do I get the \$50?  Can I wait and get a \\series of months reimbursed later or do I have to go through this every month?\\\\Fletch Sturm
  \medskip\hrule\medskip
  {\bfseries 1st example}\\[3pt]
  Lucy, \\\\Here is a rentroll for this week.  \\\\What is the outstanding balance on \#1.  It looks like 190 + 110(this week)= \\300.  I don't think we should make him pay late fees if can't communicate \\clearly.\\\\\#2 still owe deposit?\\\\\#9  What day will she pay and is she going to pay monthly or biweekly.\\\\Have a good weekend.  I will talk to you next week.\\\\In about two weeks we should know for sure if these buyers are going to buy \\the property.  I will keep you informed.\\\\Phillip
  \medskip\hrule\medskip
  {\bfseries 2nd example}\\[3pt]
  Kim,\\\\I am getting parking deducted twice from my pay check. Who do I contact to \\straighten that out?\\\\Thanx\\\\Chris
\end{examplebox}

\begin{examplebox}{FreeLaw}{}
  {\bfseries Query}\\[3pt]
  In the United States Court of Federal Claims\\                                  OFFICE OF SPECIAL MASTERS\\                                           No. 15-349V\\                                      Filed: August 20, 2015\\                                           Unpublished\\\\****************************\\ARIKA BROWNE,                         *\\                                      *\\                  Petitioner,         *      Ruling on Entitlement; Concession;\\                                      *      Influenza; Shoulder Injury (“SIRVA”)\\                                      *      Special Processing Unit (“SPU”)\\SECRETARY OF HEALTH                   *\\AND HUMAN SERVICES,                   *\\                                      *\\                  Respondent.         *\\                                      *\\****************************\\Andrew Downing, Van Cott \& Talamante, PLLC, Phoenix, AZ, for petitioner.\\Claudia Barnes Gangi, U.S. Department of Justice, Washington, DC, for respondent.\\\\                                    RULING ON ENTITLEMENT 1\\\\Vowell, Chief Special Master:\\\\       On April 7, 2015, Arika Browne filed a petition for compensation under the\\National Vaccine Injury Compensation Program, 42 U.S.C. §300aa-10, et seq., 2 [the\\“Vaccine Act” or “Program”]. Petitioner alleges that she suffered a left shoulder injury as\\a result of the administration of an influenza vaccine. Petition at 1. The case was\\assigned to the Special Processing Unit of the Office of Special Masters.\\\\       On August 20, 2015, respondent filed her Rule 4(c) report in which she concedes\\\dots
  \medskip\hrule\medskip
  {\bfseries 1st example}\\[3pt]
  In the United States Court of Federal Claims\\                                  OFFICE OF SPECIAL MASTERS\\                                           No. 15-349V\\                                      Filed: October 5, 2015\\                                           Unpublished\\\\****************************\\ARIKA BROWNE,                         *\\                                      *\\                  Petitioner,         *      Damages Decision Based on Proffer;\\                                      *      Influenza; Shoulder Injury (“SIRVA”)\\                                      *      Special Processing Unit (“SPU”)\\SECRETARY OF HEALTH                   *\\AND HUMAN SERVICES,                   *\\                                      *\\                  Respondent.         *\\                                      *\\****************************\\Andrew Downing, Van Cott \& Talamante, PLLC, Phoenix, AZ, for petitioner.\\Claudia Barnes Gangi, U.S. Department of Justice, Washington, DC for respondent.\\\\                               DECISION AWARDING DAMAGES 1\\\\Dorsey, Chief Special Master:\\\\       On April 7, 2015, Arika Browne filed a petition for compensation under the\\National Vaccine Injury Compensation Program, 42 U.S.C. §300aa-10, et seq., 2 [the\\“Vaccine Act” or “Program”]. Petitioner alleges that she suffered a left shoulder injury as\\a result of the administration of an influenza vaccine. Petition at 1. The case was\\assigned to the Special Processing Unit of the Office of Special Masters.\\\\       On August 20, 2015, a ruling on entitlement was issued, finding petitioner entitled\\\dots
  \medskip\hrule\medskip
  {\bfseries 2nd example}\\[3pt]
  In the United States Court of Federal Claims\\                                 OFFICE OF SPECIAL MASTERS\\                                          No. 15-936V\\                                   Filed: November 23, 2015\\                                          Unpublished\\\\****************************\\JENNIFER SIEKIERSKI,                   *\\                                       *\\                   Petitioner,         *      Ruling on Entitlement; Concession;\\                                       *      Influenza;\\                                       *      Shoulder Injury (“SIRVA”);\\SECRETARY OF HEALTH                    *      Special Processing Unit (“SPU”)\\AND HUMAN SERVICES,                    *\\                                       *\\                   Respondent.         *\\                                       *\\****************************\\Katheryn Lee Bruns, Faraci Lange, LLP, Rochester, NY, for petitioner.\\Julia Wernett McInerny, U.S. Department of Justice, Washington, DC, for respondent.\\\\                                    RULING ON ENTITLEMENT 1\\\\Dorsey, Chief Special Master:\\\\       On August 26, 2015, Petitioner filed a petition for compensation under the\\National Vaccine Injury Compensation Program, 42 U.S.C. §300aa-10, et seq., 2 [the\\“Vaccine Act” or “Program”]. Petitioner alleges that she experienced a shoulder injury\\related to vaccine administration (“SIRVA”) as a result of her receipt of an influenza\\vaccine on November 4, 2014. Petition at 1. The case was assigned to the Special\\Processing Unit of the Office of Special Masters.\\\\       On November 23, 2015, respondent filed her Rule 4(c) report in which she\\\dots
\end{examplebox}

\begin{examplebox}{GitHub}{}
  {\bfseries Query}\\[3pt]
  $<$?php\\\\/*\\ * This file is part of PHPExifTool.\\ *\\ * (c) 2012 Romain Neutron <imprec@gmail.com>\\ *\\ * For the full copyright and license information, please view the LICENSE\\ * file that was distributed with this source code.\\ */\\\\namespace PHPExiftool\\Driver\\Tag\\QuickTime;\\\\use JMS\\Serializer\\Annotation\\ExclusionPolicy;\\use PHPExiftool\\Driver\\AbstractTag;\\\\/**\\ * @ExclusionPolicy("all")\\ */\\class UserDataDji extends AbstractTag\\\{\\\\    protected \$Id = 'xa9dji';\\\\    protected \$Name = 'UserData\_dji';\\\\    protected \$FullName = 'QuickTime::UserData';\\\\    protected \$GroupName = 'QuickTime';\\\\    protected \$g0 = 'QuickTime';\\\\    protected \$g1 = 'QuickTime';\\\\    protected \$g2 = 'Video';\\\\    protected \$Type = 'undef';\\\\    protected \$Writable = false;\\\\    protected \$Description = 'User Data dji';\\\\    protected \$flag\_Binary = true;\\\}\\
  \medskip\hrule\medskip
  {\bfseries 1st example}\\[3pt]
  $<$?php\\\\/*\\ * This file is part of PHPExifTool.\\ *\\ * (c) 2012 Romain Neutron <imprec@gmail.com>\\ *\\ * For the full copyright and license information, please view the LICENSE\\ * file that was distributed with this source code.\\ */\\\\namespace PHPExiftool\\Driver\\Tag\\QuickTime;\\\\use JMS\\Serializer\\Annotation\\ExclusionPolicy;\\use PHPExiftool\\Driver\\AbstractTag;\\\\/**\\ * @ExclusionPolicy("all")\\ */\\class UserDataUid extends AbstractTag\\\{\\\\    protected \$Id = 'xa9uid';\\\\    protected \$Name = 'UserData\_uid';\\\\    protected \$FullName = 'QuickTime::UserData';\\\\    protected \$GroupName = 'QuickTime';\\\\    protected \$g0 = 'QuickTime';\\\\    protected \$g1 = 'QuickTime';\\\\    protected \$g2 = 'Video';\\\\    protected \$Type = 'undef';\\\\    protected \$Writable = false;\\\\    protected \$Description = 'User Data uid';\\\\    protected \$flag\_Binary = true;\\\}\\
  \medskip\hrule\medskip
  {\bfseries 2nd example}\\[3pt]
  $<$?php\\\\/*\\ * This file is part of PHPExifTool.\\ *\\ * (c) 2012 Romain Neutron <imprec@gmail.com>\\ *\\ * For the full copyright and license information, please view the LICENSE\\ * file that was distributed with this source code.\\ */\\\\namespace PHPExiftool\\Driver\\Tag\\QuickTime;\\\\use JMS\\Serializer\\Annotation\\ExclusionPolicy;\\use PHPExiftool\\Driver\\AbstractTag;\\\\/**\\ * @ExclusionPolicy("all")\\ */\\class MovieData extends AbstractTag\\\{\\\\    protected \$Id = 'mdat';\\\\    protected \$Name = 'MovieData';\\\\    protected \$FullName = 'QuickTime::Main';\\\\    protected \$GroupName = 'QuickTime';\\\\    protected \$g0 = 'QuickTime';\\\\    protected \$g1 = 'QuickTime';\\\\    protected \$g2 = 'Video';\\\\    protected \$Type = '?';\\\\    protected \$Writable = false;\\\\    protected \$Description = 'Movie Data';\\\\    protected \$flag\_Binary = true;\\\}\\
\end{examplebox}

\end{document}